\documentclass{article}

\usepackage{graphicx}
\usepackage{float}
\usepackage{caption}
\usepackage{subcaption}
\let\captionbox=\undefined
\DeclareCaptionType{copyrightbox}
\usepackage{wrapfig}

\usepackage{natbib}
\usepackage{multibib} %
\newcites{sup}{Supplementary References}

\usepackage{algorithm}
\usepackage[noend]{algorithmic} %

\usepackage{hyperref}

\usepackage{array,multirow}

\usepackage{amsmath, amssymb, amsthm, bm, color}
\usepackage{notation}

\ifdefined \pgftikzinstalled
    \usepackage[accepted]{icml2013}
    \usepackage{tikz}
    \usetikzlibrary{external}
    \usepackage{pgfplots}
    \usepackage{pgfplotstable}
    \pgfplotsset{compat=newest}
\else
    \usepackage[accepted]{icml2013}
    \usepackage{tikzexternal}
\fi

\tikzsetexternalprefix{externalized/}

\icmltitlerunning{Block-Coordinate Frank-Wolfe Optimization for Structural SVMs}

\begin{document}

\twocolumn[
\icmltitle{Block-Coordinate Frank-Wolfe Optimization for Structural SVMs}

\icmlauthor{Simon Lacoste-Julien$^*\!$}{INRIA - SIERRA project-team, {\'E}cole Normale Sup{\'e}rieure, Paris, France\!\!}%
\icmladdress{}
\vspace*{-0.6em}

\icmlauthor{Martin Jaggi$^*\!$}{CMAP, {\'E}cole Polytechnique, Palaiseau, France}%
\icmladdress{}
\vspace*{-0.6em}

\icmlauthor{Mark Schmidt}{INRIA - SIERRA project-team, {\'E}cole Normale Sup{\'e}rieure, Paris, France}%
\icmladdress{}
\vspace*{-0.6em}

\icmlauthor{Patrick Pletscher}{Machine Learning Laboratory, ETH Zurich, Switzerland}%
\icmladdress{{\small $^*$ Both authors contributed equally.}}

\icmlkeywords{structured output prediction, Frank-Wolfe, block-coordinate, structured SVM, machine learning, ICML}

\vspace*{1em}
]

\begin{abstract}
We propose a randomized block-coordinate variant of the classic Frank-Wolfe algorithm for convex optimization with block-separable constraints. Despite its lower iteration cost, we show that it achieves a similar convergence rate in duality gap as the full Frank-Wolfe algorithm.
We also show that, when applied to the dual structural support vector machine (SVM) objective, this yields an online algorithm that has the same low iteration complexity as primal stochastic subgradient methods.
However, unlike stochastic subgradient methods, the block-coordinate Frank-Wolfe algorithm allows us to compute the \emph{optimal} step-size and yields a computable duality gap guarantee.
Our experiments indicate that this simple algorithm outperforms competing structural SVM solvers.
\end{abstract}
\section{Introduction}

Binary SVMs are amongst the most popular classification methods, and this
has motivated substantial interest in optimization solvers that are tailored
to their specific problem structure.  However, despite their wider applicability,
there has been much less work on solving the optimization problem associated with
\emph{structural} SVMs, which are the generalization of SVMs to structured outputs
like graphs and other combinatorial objects~\citep{Taskar2003, Tsochantaridis2005}.
This seems to be due to the difficulty of dealing with the exponential number of constraints
in the primal problem, or the exponential number of variables in the dual problem.
Indeed, because they achieve an $\tilde{O}(1/\varepsilon)$ convergence rate while only requiring a single call to the so-called \emph{maximization oracle}
on each iteration, basic stochastic subgradient methods are still widely used for training
structural SVMs~\citep{Ratliff:2007subgradient,ShalevShwartz:2010cg}. However, these methods are often frustrating to use for practitioners,
because their performance is very sensitive to the sequence of step sizes, and because
it is difficult to decide when to terminate the iterations.

To solve the dual structural SVM problem, in this paper we consider the
Frank-Wolfe~\yrcite{Frank:1956vp} algorithm,
which has seen a recent surge
of interest in machine learning and signal processing~\citep{Mangasarian:1995wa,Clarkson:2010hv,Jaggi:2011ux,Jaggi:2013wg,bach12herding},
including in the context of binary SVMs~\citep{GartnerJaggi:2009,Ouyang:2010vc}.
A key advantage of this algorithm is that the iterates
are \emph{sparse}, and we show that this allows us to efficiently apply it
to the dual structural SVM objective even though there are an exponential number of variables.
A second key advantage of this algorithm is that the iterations only require optimizing
linear functions over the constrained domain, and we show that \emph{this is equivalent to
the maximization oracle} used by subgradient and cutting-plane methods~\citep{Joachims:2009ex, Teo:2010bundle}. Thus, the Frank-Wolfe algorithm
has the same wide applicability as subgradient methods,
and can be applied to problems such as
low-treewidth graphical models~\citep{Taskar2003}, graph matchings~\citep{Caetano:09graphMatching},
and associative Markov networks~\citep{taskar04thesis}. In contrast, other approaches must
use more expensive (and potentially intractable) oracles such as computing marginals over labels~\citep{Collins2008,Zhang:2011:ATM3net} or doing a Bregman projection onto the space of structures~\citep{Taskar06extrag}.
Interestingly, for structural SVMs we also show
that existing batch subgradient and cutting-plane methods are \emph{special cases} of Frank-Wolfe algorithms, and this leads to
stronger and simpler $O(1/\varepsilon)$ convergence rate guarantees for these existing algorithms. %

As in other batch structural SVM solvers like cutting-plane methods~\citep{Joachims:2009ex, Teo:2010bundle} and the excessive gap technique~\citep{Zhang:2011:ATM3net} (see Table~\ref{tab:rates} at the end for an overview), each Frank-Wolfe iteration unfortunately requires calling the appropriate oracle once for \emph{all} training examples, unlike the single oracle call needed by stochastic subgradient methods. This can be prohibitive for data sets with a large number of training examples.
To reduce this cost, we propose a novel randomized block-coordinate version of the Frank-Wolfe algorithm for problems with block-separable constraints. We show that this algorithm still achieves the $O(1/\varepsilon)$ convergence rate of the full Frank-Wolfe algorithm, and in the context of structural SVMs, it only requires a single call to the maximization oracle. Although the stochastic subgradient and the novel block-coordinate Frank-Wolfe algorithms have a similar iteration cost and theoretical convergence rate for solving the structural SVM problem, the new algorithm has several important advantages for practitioners:\vspace{-3mm}
\begin{itemize}
\item The \emph{optimal} step-size can be efficiently computed in closed-form, hence no step-size needs to be selected.\vspace{-2mm}
\item The algorithm yields a \emph{duality gap} guarantee, and (at the cost of computing the primal objective) we can compute the duality gap as a proper stopping criterion.\vspace{-2mm}
\item The convergence rate holds even when using \emph{approximate} maximization oracles.\vspace{-2mm}
\end{itemize}
Further, our experimental results show that the optimal step-size leads to a significant advantage during the first few passes through the data, and a systematic (but smaller) advantage in later passes.

\vspace{-1mm}
\section{Structural Support Vector Machines} \label{sec:setup}

We first briefly review the standard convex optimization setup for structural SVMs~\citep{Taskar2003,Tsochantaridis2005}. In structured prediction, the goal is to predict a structured object $\outputvarv \in \outputdomain(\inputvarv)$ (such as a sequence of tags) for a given input $\inputvarv \in \inputdomain$. In the standard approach, a structured feature map $\featuremapv: \inputdomain \times \outputdomain \rightarrow \R^d$ encodes the relevant information for input/output pairs, and a linear classifier with parameter $\weightv$ is defined by $h_{\weightv}(\inputvarv) = \argmax_{\outputvarv\in\outputdomain(\inputvarv)}\langle\weightv, \featuremapv(\inputvarv,\outputvarv) \rangle$.
Given a labeled training set $\data = \{(\inputvarv_i,\outputvarv_i)\}_{i=1}^n$, $\weightv$ is estimated by solving\vspace{-1mm}
\begin{align}
    \label{eq:svmstruct_nslack_primal}
    \min_{\weightv,\, \bm{\xi}} \quad &   \frac{\regularizerweight}{2}\norm{\weightv}^2 +
    \frac1n \sum_{i=1}^n \xi_i\\[-4mm]
    \text{s.t.} \quad &
    \langle \weightv, \featuremapdiffv_i(\outputvarv) \rangle
    \geq \errorterm(\outputvarv_i,\outputvarv) - \xi_i \quad  \forall i, \, \forall \outputvarv \in \overbrace{
    \outputdomain(\inputvarv_i)}^{=: \outputdomain_i}, 
      \notag
\end{align}
where $\featuremapdiffv_i(\outputvarv):= \featuremapv(\inputvarv_i,\outputvarv_i) - \featuremapv(\inputvarv_i,\outputvarv)$, and $\errorterm_i(\outputvarv) := \errorterm(\outputvarv_i,\outputvarv)$ denotes the task-dependent structured error of predicting output $\outputvarv$ instead of the observed output $\outputvarv_i$ (typically a Hamming distance between the two labels). The slack variable~$\xi_i$ measures the surrogate loss for the $i$-th datapoint
 and $\regularizerweight$ is the regularization parameter.
The convex problem~\eqref{eq:svmstruct_nslack_primal} is what \citet[Optimization Problem 2]{Joachims:2009ex} call the $n$-slack structural SVM with margin-rescaling. A variant with \emph{slack-rescaling} was proposed by \citet{Tsochantaridis2005}, which is equivalent to our setting if we replace all vectors $\featuremapdiffv_i(\outputvarv)$ by $\errorterm_i(\outputvarv) \featuremapdiffv_i(\outputvarv)$.

\paragraph{Loss-Augmented Decoding.}
Unfortunately, the above problem can have an exponential number of constraints due to the combinatorial nature of $\outputdomain$. We can replace
the $\sum_i|\outputdomain_i|$ linear constraints with $n$ \emph{piecewise-linear} ones by defining the structured hinge-loss:\vspace{-1mm}
\begin{equation}\label{eq:subproblem_loss_augm}
\text{\parbox[t]{3em}{`max \\oracle'}} \quad %
  \tilde{H}_i(\weightv) := \max_{\outputvarv\in\outputdomain_i} \
    \underbrace{%
    \errorterm_i(\outputvarv)
    - \langle \weightv,
    \featuremapdiffv_i(\outputvarv)
    \rangle
    }_{=:\, H_i(\outputvarv;\weightv)}.\vspace{-1mm}
\end{equation}
The constraints in~\eqref{eq:svmstruct_nslack_primal} can thus be replaced with the non-linear ones $\xi_i \geq \tilde{H}_i(\weightv)$. The computation of the structured hinge-loss for each $i$ amounts to finding the most `violating' output $\outputvarv$ for a given input~$\inputvarv_i$, a task which can be carried out efficiently in many structured prediction settings (see the introduction). %
This problem is called the \emph{loss-augmented decoding} subproblem. In this paper, we only assume access to an efficient solver for this subproblem, and we call such a solver a \emph{maximization oracle}.
The equivalent non-smooth unconstrained formulation of \eqref{eq:svmstruct_nslack_primal} is:
\vspace{-2mm}
\begin{equation}
    \label{eq:svmstruct_nslack_primal_nonsmooth}
    \min_{\weightv} \quad \frac{\regularizerweight}{2}\norm{\weightv}^2 +
    \frac1n \sum_{i=1}^n \tilde{H}_i(\weightv) .\vspace{-2mm}
\end{equation}
Having a maximization oracle allows us to apply subgradient methods to this problem~\citep{Ratliff:2007subgradient}, as a subgradient of $\tilde{H}_i(\weightv)$ with respect to $\weightv$ is $-\featuremapdiffv_i(\outputvarv_i^*)$, where~$\outputvarv_i^*$ is any maximizer of the loss-augmented decoding subproblem \eqref{eq:subproblem_loss_augm}.

\paragraph{The Dual.}%
The Lagrange dual of the above $n$-slack-formulation (\ref{eq:svmstruct_nslack_primal}) has $m := \sum_i |\outputdomain_i|$ variables or potential `support vectors'.
Writing $\dualvar_i(\outputvarv)$ for the dual variable associated with the training example $i$ and potential output $\outputvarv \in \outputdomain_i$, the dual problem is given by %
\begin{align}
    \label{eq:svmstruct_nslack_dual} %
    \min_{\substack{ \dualvarv\in\R^{m} \\  \dualvarv \geq 0}} \quad  f(\dualvarv) \;:=&  \;\;
    \frac{\lambda}{2}
    \big\| A\dualvarv \big\|^2
    - \bv^T\dualvarv
    \\[-2mm]
    \text{s.t.} \quad &  \;
      \textstyle\sum_{\outputvarv \in \outputdomain_i}  \dualvar_i(\outputvarv) = 1 ~~~\forall i\in[n] \ , \notag 
\end{align}
where the matrix $A\in\R^{d\times m}$ consists of the $m$ columns $A := \SetOf{\frac1{\lambda n} \featuremapdiffv_i(\outputvarv) \in\R^d}{i\in[n],\outputvarv \in \outputdomain_i}$, and the vector $\bv \in \R^m$ is given by 
$\bv:= \left(\frac1n \errorterm_i(\outputvarv) \right)$$_{i\in[n],\outputvarv\in\outputdomain_i}$. %
Given a dual variable vector $\dualvarv$, we can use the Karush-Kuhn-Tucker optimality conditions  to obtain the corresponding primal variables~$
\weightv = A\dualvarv  = \sum_{i,\,\outputvarv \in \outputdomain_i} \dualvar_i(\outputvarv)  \frac{\featuremapdiffv_i(\outputvarv)}{\lambda n}
$, see Appendix~\ref{sec:app_duals}.
The gradient of $f$ then takes the simple form $\nabla f(\dualvarv) = \lambda A^TA\dualvarv - \bv = \lambda A^T\weightv - \bv$; its \mbox{$(i,\outputvarv)$-th} component is $-\frac{1}{n} H_i(\outputvarv; \weightv)$, cf.~\eqref{eq:subproblem_loss_augm}. 
Finally, note that the domain $\domain \subset \R^m$ of~\eqref{eq:svmstruct_nslack_dual} is the product of $n$ probability simplices, $\domain := \simplex_{|\outputdomain_1|}\times\mathellipsis\times\simplex_{|\outputdomain_n|}$.
\section{The Frank-Wolfe Algorithm}%
\label{sec:FW}

\setlength{\textfloatsep}{10pt}%

\begin{algorithm}[b!]
  \caption{Frank-Wolfe on a Compact Domain}
  \label{alg:FW}
\begin{algorithmic}
  \STATE Let $\dualvarv^{(0)} \in \domain$
  \FOR{$k=0\dots K$}
  \STATE Compute $\sv := \displaystyle\argmin_{\sv'\in \domain} \left\langle \sv', \nabla f(\dualvarv^{(k)}) \right\rangle$
  \STATE Let $\stepsize := \frac2{k+2}$,\  {\small or optimize $\stepsize$ by line-search}
  \STATE Update $\dualvarv^{(k+1)}:= (1-\stepsize)\dualvarv^{(k)}+\stepsize\sv$
  \ENDFOR
\end{algorithmic}
\end{algorithm}

We consider the convex optimization problem $\min_{\dualvarv \in \domain} \, f(\dualvarv)$, where the convex feasible set $\domain$ is \emph{compact} and the convex objective $f$ is \emph{continuously differentiable}.
The Frank-Wolfe algorithm~\yrcite{Frank:1956vp} (shown in Algorithm~\ref{alg:FW}) is an iterative optimization algorithm for such problems that only requires optimizing \emph{linear} functions over $\domain$, and thus has wider applicability than projected gradient algorithms, which require optimizing a quadratic function over~$\domain$.
At every iteration, a feasible search corner $\sv$ is first found by minimizing over $\domain$ the \emph{linearization} of~$f$ at the current iterate $\dualvarv$ (see picture in inset).
\begin{wrapfigure}{r}{4.7cm}\vspace{-1em}
\includegraphics[width=0.95\linewidth]{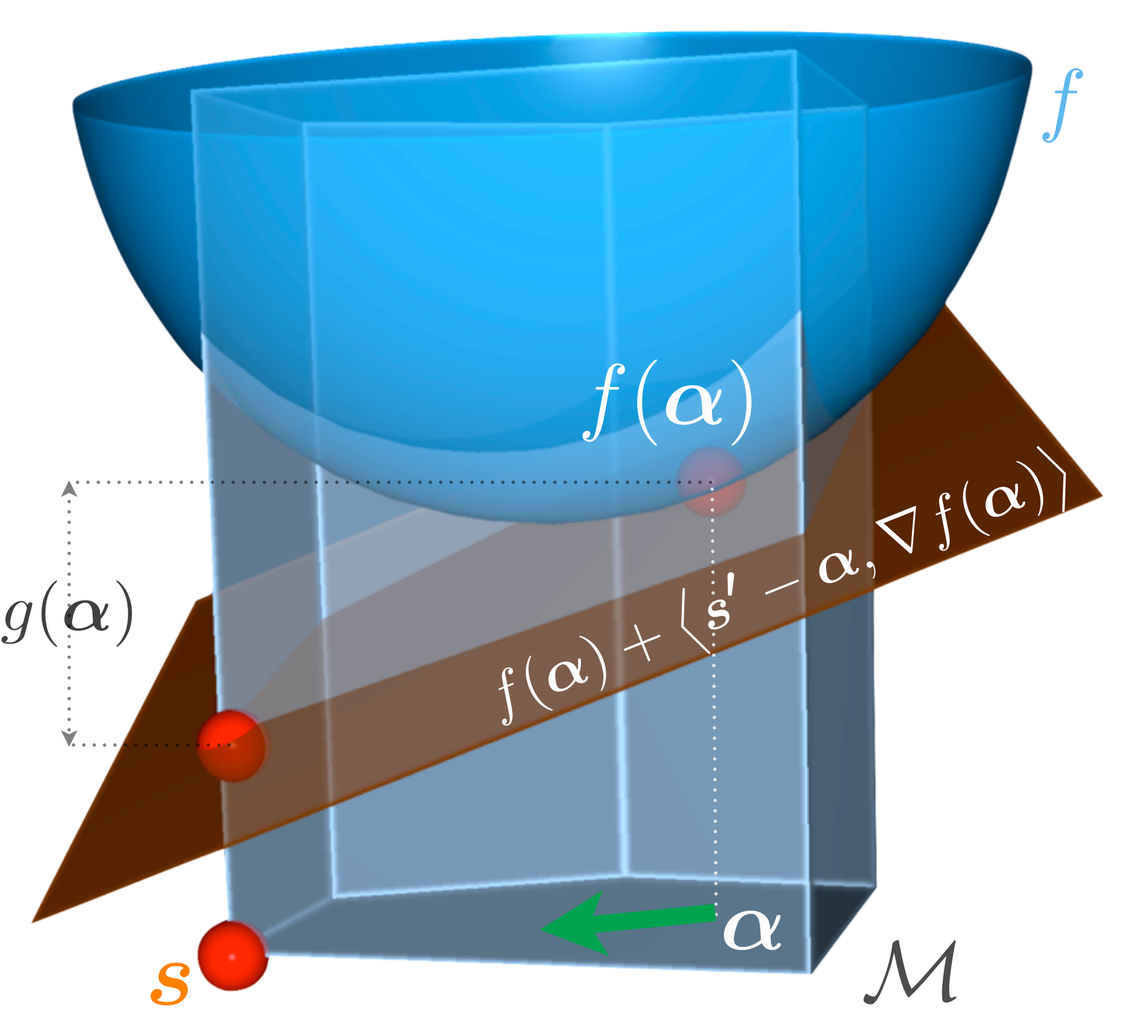}\vspace{-1em}
\end{wrapfigure}
The next iterate is then obtained as a convex combination of $\sv$ and the previous iterate, with step-size~$\stepsize$.
These simple updates yield two interesting properties. First, every iterate $\dualvarv^{(k)}$ can be written as a convex combination of the starting point~$\dualvarv^{(0)}$ and the search corners~$\sv$ found previously. The parameter $\dualvarv^{(k)}$ thus has a sparse representation, which makes the algorithm suitable even for cases where the dimensionality of $\dualvarv$ is exponential. Second, since~$f$ is convex, the minimum of the linearization of $f$ over~$\domain$ immediately gives a lower bound on the value of the yet unknown optimal solution $f(\dualvarv^*)$.
Every step of the algorithm thus computes for free the following `linearization duality gap' defined for any feasible point $\dualvarv \in \domain$ (which is in fact a special case of the Fenchel duality gap as explained in Appendix~\ref{sec:Fenchel_gap_equivalence}):\vspace{-1mm}
\begin{equation}\label{eq:duality_gap}
  g(\dualvarv) := \max_{\sv' \in \domain} \,\langle \dualvarv - \sv', \nabla f(\dualvarv) \rangle
  = \langle \dualvarv - \sv, \nabla f(\dualvarv) \rangle .\vspace{-1mm}
\end{equation}
As $g(\dualvarv)\ge f(\dualvarv)-f(\dualvarv^*)$ by the above argument, $\sv$
thus readily gives at each iteration the current duality gap as a
\emph{certificate} for the current approximation quality \citep{Jaggi:2011ux,Jaggi:2013wg}, allowing us to monitor the convergence, and more importantly to choose the theoretically sound stopping criterion $g(\dualvarv^{(k)}) \le \varepsilon$.%

In terms of convergence, it is known that after $O(1/\varepsilon)$ iterations, Algorithm~\ref{alg:FW} obtains an $\varepsilon$-approximate solution \citep{Frank:1956vp,Dunn:1978di} as well as a guaranteed $\varepsilon$-small duality gap \citep{Clarkson:2010hv,Jaggi:2013wg}, along with a certificate to~\eqref{eq:duality_gap}.
For the convergence results to hold, the internal linear subproblem does not need to be solved exactly, but only to some  error.
We review and generalize the convergence proof in Appendix \ref{sec:app_convergence_proof}.
The constant hidden in the~$O(1/\varepsilon)$ notation is the \emph{curvature constant} $\Cf$, an affine invariant quantity 
measuring the maximum deviation of $f$ from its linear approximation over $\domain$ (it yields a weaker form of Lipschitz assumption on the gradient, %
see e.g. Appendix~\ref{sec:app_curvature} for a formal definition).

\section{Frank-Wolfe for Structural SVMs}
\label{sec:FW_SVM}

\begin{algorithm}[b!]
    \caption{Batch Primal-Dual Frank-Wolfe Algorithm for the Structural SVM} %
    \label{alg:FW_SVM}
\begin{algorithmic}
        \STATE Let $\weightv^{(0)}:= \0,~~\ell^{(0)}:=0$
       \FOR{$k=0\dots K$}
          \FOR{$i=1\dots n$}
                \STATE Solve $\outputvarv_i^* := \displaystyle\argmax_{\outputvarv\in\outputdomain_i} \ H_i(\outputvarv;\weightv^{(k)})$ cf.~\eqref{eq:subproblem_loss_augm} %
           \ENDFOR %
           \STATE Let $\weightv_{\sv} := {\displaystyle\sum_{i=1}^n} \frac1{\lambda n} \featuremapdiffv_i(\outputvarv_i^*)$
                 {\small ~and $\ell_{\sv} := \frac1n \displaystyle\sum_{i=1}^n \errorterm_i(\outputvarv_i^*)$}
           \STATE {\small Let $\stepsize := \frac{\lambda (\weightv^{(k)}-\weightv_{\sv})^T\weightv^{(k)} - \ell^{(k)} + \ell_{\sv} }{ \lambda \|\weightv^{(k)}-\weightv_{\sv}\|^2 }$ ~and clip  to $[0,1]$}
           \STATE Update $\weightv^{(k+1)}:= (1-\stepsize)\weightv^{(k)}+\stepsize\, \weightv_{\sv}$
           \STATE {\small~~~~~~~and~ $\ell^{(k+1)}:= (1-\stepsize)\ell^{(k)}+\stepsize\, \ell_{\sv}$}
        \ENDFOR
\end{algorithmic}
\end{algorithm}

Note that classical algorithms like the projected gradient method cannot be tractably applied to the dual of the structural SVM problem~\eqref{eq:svmstruct_nslack_dual}, due to the large number of dual variables.
In this section, we explain how the Frank-Wolfe method (Algorithm~\ref{alg:FW}) can be efficiently applied to this dual problem, and discuss its relationship to other algorithms.
The main insight here is to notice that the linear subproblem employed by Frank-Wolfe is actually directly equivalent to the loss-augmented decoding subproblem \eqref{eq:subproblem_loss_augm} for each datapoint, which can be solved efficiently
(see Appendix~\ref{ssec:FW_equivalent_decoding} for details).
Recall that the optimization domain for the dual variables $\dualvarv$ is the product of~$n$~simplices, $\domain = \simplex_{|\outputdomain_1|}\times\mathellipsis\times\simplex_{|\outputdomain_n|}$. Since each simplex consists of a potentially exponential number $|\outputdomain_i|$ of dual variables, we cannot maintain a dense vector~$\dualvarv$ during the algorithm. %
However, as
mentioned in Section~\ref{sec:FW}, each iterate $\dualvarv^{(k)}$ of the
Frank-Wolfe algorithm is a sparse convex combination of the previously visited
corners $\sv$ and the starting point $\dualvarv^{(0)}$, and so we only need to
maintain the list of previously seen solutions to the loss-augmented decoding
subproblems to keep track of the non-zero coordinates of $\dualvarv$, avoiding
the problem of its exponential size. Alternately, if we do not use kernels, we can avoid the quadratic explosion of the number of operations needed in the dual
by \emph{not} explicitly maintaining $\dualvarv^{(k)}$, but instead
maintaining the corresponding \emph{primal} variable $\weightv^{(k)}$.

\paragraph{A Primal-Dual Frank-Wolfe Algorithm for the Structural SVM Dual.}%
Applying %
Algorithm~\ref{alg:FW} with line search to the dual of the structural SVM~\eqref{eq:svmstruct_nslack_dual}, but only maintaining the corresponding \emph{primal} primal iterates $\weightv^{(k)}:= A\dualvarv^{(k)}$, we obtain Algorithm~\ref{alg:FW_SVM}. Note that the Frank-Wolfe search corner $\sv = (\unit^{\outputvarv_1^*},\dots,\unit^{\outputvarv_n^*})$, which is obtained by solving the loss-augmented subproblems, yields the update $\weightv_{\sv} = A\sv$. We use the natural starting point $\dualvarv^{(0)} := (\unit^{\outputvarv_1},\dots,\unit^{\outputvarv_n})$ which yields $\weightv^{(0)} = \0$ as $\featuremapdiffv_i(\outputvarv_i) = \0$ $\forall i$.

\vspace{-1mm}
\paragraph{The Duality Gap.}
The duality gap \eqref{eq:duality_gap} for our structural SVM dual formulation \eqref{eq:svmstruct_nslack_dual} is given by\vspace{-1mm}
\[
\begin{array}{rl}
  g(\dualvarv) :=& \displaystyle\max_{\sv' \in \domain} \,\langle \dualvarv - \sv', \nabla f(\dualvarv) \rangle \\
    =& (\dualvarv - \sv)^T(\lambda A^TA\dualvarv - \bv) \\
    =& \lambda (\weightv-A\sv)^T\weightv - \bv^T\dualvarv + \bv^T\sv
   \ ,\vspace{-1mm}
\end{array}
\]
where~$\sv$ is an \emph{exact} minimizer of the linearized problem given at
the point $\dualvarv$. This (Fenchel) duality gap turns out to be the same as the Lagrangian duality gap here (see Appendix~\ref{sec:gap_comparison}), and gives a direct handle on the suboptimality of $\weightv^{(k)}$ for the primal problem~\eqref{eq:svmstruct_nslack_primal_nonsmooth}. Using $\weightv_{\sv} :=
A\sv$ and $\ell_{\sv} := \bv^T\sv$, we observe that the gap is efficient to compute given the \emph{primal}
variables $\weightv := A\dualvarv$ and $\ell := \bv^T\dualvarv$, which are maintained during the run of Algorithm~\ref{alg:FW_SVM}. Therefore, we can use the duality gap $g(\dualvarv^{(k)}) \le \varepsilon$ as a proper stopping criterion.

\paragraph{Implementing the Line-Search.}
Because the objective of the structural SVM dual~\eqref{eq:svmstruct_nslack_dual} is simply a quadratic function in $\dualvarv$, the optimal step-size for any given candidate search point $\sv\in\domain$ can be obtained \emph{analytically}. Namely, $\stepsize_{LS} := \argmin_{\stepsize\in[0,1]} \, f\left(\dualvarv+\stepsize\big(\sv - \dualvarv\big)\right)$ is obtained by setting the derivative of this univariate quadratic function in~$\stepsize$ to zero, which here (before restricting to $[0,1]$) gives 
$
\stepsize_{opt} := \frac{\langle \dualvarv - \sv, \nabla f(\dualvarv) \rangle}{\lambda\norm{A(\dualvarv-\sv)}^2}
= \frac{g(\dualvarv)}{\lambda\norm{\weightv-\weightv_{\sv}}^2}
$ %
 (used in Algorithms~\ref{alg:FW_SVM} and~\ref{alg:FW_product_SVM}).

\vspace{-1mm}
\paragraph{Convergence Proof and Running Time.}
In the following, we write $R$ for the maximal length of a difference feature
vector, i.e. $R \! := \! \max_{i\in[n],\outputvarv \in \outputdomain_i} \! \norm{\featuremapdiffv_i(\outputvarv)}_2$, and we write the maximum error as $\Lmax:= \max_{i,\outputvarv } \errorterm_i(\outputvarv)$.
By bounding the curvature constant $\Cf$ for the dual SVM objective~\eqref{eq:svmstruct_nslack_dual}, we can now directly apply the known convergence results for the standard Frank-Wolfe algorithm to obtain the following \emph{primal-dual} rate (proof in Appendix~\ref{ssec:app_SVM_proofs}): 
\begin{theorem}%
\label{thm:convergence_FW_SVM}
Algorithm~\ref{alg:FW_SVM} obtains an $\varepsilon$-approximate solution to the structural SVM dual problem~\eqref{eq:svmstruct_nslack_dual} and duality gap $g(\dualvarv^{(k)}) \le\varepsilon$ after at most $O\left(\frac{R^2}{\lambda\varepsilon}\right)$ iterations, where each iteration costs $n$ oracle calls.
\end{theorem}%
Since we have proved that the duality gap is smaller than $\varepsilon$, this implies that the original SVM primal objective~\eqref{eq:svmstruct_nslack_primal_nonsmooth} is actually solved to accuracy $\varepsilon$ as well.

\paragraph{Relationship with the Batch Subgradient Method in the Primal.}
Surprisingly, the batch Frank-Wolfe method (Algorithm~\ref{alg:FW_SVM}) %
is equivalent to the batch subgradient method in the primal, though Frank-Wolfe allows a more clever choice of step-size, since line-search can be used in the dual.
To see the equivalence, notice that a subgradient of~\eqref{eq:svmstruct_nslack_primal_nonsmooth} is given by
$
    \bm{d}_{sub} = \lambda \weightv - \frac1n \sum_i
    \featuremapdiffv_i(\outputvarv_i^*) = \lambda (\weightv - \weightv_{\sv}),
$
where $\outputvarv_i^*$ and $\weightv_{\sv}$ are as defined in Algorithm~\ref{alg:FW_SVM}. Hence, for a step-size of $\beta$, the subgradient method update becomes $
    \weightv^{(k+1)}:= \weightv^{(k)}-\beta \bm{d}_{sub} = \weightv^{(k)} - \beta \lambda (\weightv^{(k)} - \weightv_{\sv}) = (1 - \beta \lambda) \weightv^{(k)} + \beta \lambda \weightv_{\sv}
.$
Comparing this with Algorithm~\ref{alg:FW_SVM}, we see that each Frank-Wolfe
step on the dual problem~\eqref{eq:svmstruct_nslack_dual} with step-size
$\stepsize$ is equivalent to a batch subgradient step in the primal
with a step-size of $\beta = \stepsize / \lambda$, and thus our convergence
results also apply to it. This seems to generalize the equivalence between
Frank-Wolfe and the subgradient method for a quadratic objective
with identity Hessian as observed by~\citet[Section
4.1]{bach12herding}.%
\paragraph{Relationship with Cutting Plane Algorithms.}
In each iteration, the cutting plane algorithm of~\citet{Joachims:2009ex} and the Frank-Wolfe method (Algorithm~\ref{alg:FW_SVM}) 
solve the loss-augmented decoding problem for each datapoint, selecting the same new `active' coordinates to add to the dual problem.
The only difference is that instead of just moving towards the corner~$\sv$, as in classical Frank-Wolfe, the cutting plane algorithm re-optimizes over all the previously added `active' dual variables (this task is a quadratic program).
This shows that the method is exactly equivalent to the `fully corrective' variant of Frank-Wolfe, which in each iteration re-optimizes over all previously visited corners~\cite{Clarkson:2010hv,ShalevShwartz:2010wq}.
Note that the convergence results for the `fully corrective' variant directly follow from the ones for Frank-Wolfe (by inclusion), thus our convergence results apply to the cutting plane algorithm of \citet{Joachims:2009ex}, significantly simplifying its analysis.

\section{Faster Block-Coordinate Frank-Wolfe}
\label{sec:FW_product_SVM}

\begin{algorithm}[t!]
  \caption{Block-Coordinate Frank-Wolfe Algorithm on Product Domain}
  \label{alg:FW_product}
\begin{algorithmic}
  \STATE Let $\dualvarv^{(0)} \in \domain = \domain^{(1)}\times\mathellipsis\times \domain^{(n)}$ %
  \FOR{$k=0\dots K$}
  \STATE Pick $i$ at random in $\{1,\ldots,n\}$
  \STATE Find $\sv_{(i)} := \displaystyle\argmin_{\sv_{(i)}'\in \domain^{(i)}} \left\langle \sv_{(i)}', \nabla_{\!(i)} f(\dualvarv^{(k)}) \right\rangle$\vspace{-1pt}
  \STATE Let $\stepsize := \frac{2n}{k+2n}$,\  {\small or optimize $\stepsize$ by line-search}
  \STATE Update $\dualvarv^{(k+1)}_{(i)}:=\dualvarv^{(k)}_{(i)}+\stepsize\big(\sv_{(i)} - \dualvarv^{(k)}_{(i)}\big)$
  \ENDFOR
\end{algorithmic}
\end{algorithm}

A major disadvantage of the standard Frank-Wolfe algorithm when applied to the
structural SVM problem is that each iteration requires a full pass through the data, resulting in $n$ calls to the maximization oracle.
In this section, we present the main new contribution of the paper: a \emph{block-coordinate} generalization of the Frank-Wolfe algorithm that maintains all appealing properties of Frank-Wolfe, but yields much cheaper iterations, requiring only one call to the maximization oracle in the context of structural SVMs. The new method is given in
Algorithm~\ref{alg:FW_product}, and applies to any constrained convex optimization problem of the form\vspace{-1mm}
\begin{equation}\label{eq:opt_gen_convex_product}
   \min_{\dualvarv \in \domain^{(1)}\times\mathellipsis\times \domain^{(n)}} \, f(\dualvarv)\ ,\vspace{-1mm}
\end{equation}
where the domain has the structure of a Cartesian product
$\domain = \domain^{(1)}\times\mathellipsis\times \domain^{(n)} \subseteq \R^m$
over~$n\ge1$ blocks. The main idea of the method is to perform cheaper update steps that only affect a single variable block~$\domain^{(i)}$, and not all of them simultaneously. This is motivated by coordinate descent methods, which have a very successful history when applied to large scale optimization.
Here we assume that each factor $\domain^{(i)} \subseteq\R^{m_i}$ is convex
and \emph{compact}, with $m = \sum_{i=1}^n m_i$. We will write
$\dualvarv_{(i)}\in\R^{m_i}$ for the $i$-th block of coordinates of a vector $\dualvarv\in\R^m$.
In each step, Algorithm~\ref{alg:FW_product} picks one of the~$n$ blocks uniformly at random, and leaves all other blocks unchanged. If there is only one block ($n=1$), then Algorithm \ref{alg:FW_product} becomes the standard Frank-Wolfe Algorithm \ref{alg:FW}.
The algorithm can be interpreted as a simplification of Nesterov's `huge-scale' uniform coordinate descent method \citep[Section 4]{Nesterov:2012fa}.
Here, instead of computing a projection operator on a block (which is intractable for structural SVMs), we only need to solve one linear subproblem in each iteration, which for structural SVMs is equivalent to a call to the maximization oracle.
\begin{algorithm}[t!]
    \caption{Block-Coordinate Primal-Dual Frank-Wolfe Algorithm for the Structural SVM}%
    \label{alg:FW_product_SVM}
\begin{algorithmic}
        \STATE Let $\weightv^{(0)}:= {\weightv_i}^{(0)}:= \bar{\weightv}^{(0)} := \0,~~\ell^{(0)}:={\ell_i}^{(0)}:=0$
       \FOR{$k=0\dots K$}
                \STATE Pick $i$ at random in $\{1,\ldots,n\}$
                \STATE Solve $\outputvarv_i^* := \displaystyle\argmax_{\outputvarv\in\outputdomain_i} \ H_i(\outputvarv;\weightv^{(k)})$ cf.~\eqref{eq:subproblem_loss_augm}
                \STATE Let $\weightv_{\sv} := \frac1{\lambda n} \featuremapdiffv_i(\outputvarv_i^*)$
                {\small~~~and~~ $\ell_{\sv} := \frac1n \errorterm_i(\outputvarv_i^*)$}\vspace{1pt}
               \STATE {\small Let $\stepsize := \frac{ \lambda (\weightv_i^{(k)}-\weightv_{\sv})^T\weightv^{(k)} - \ell_i^{(k)} + \ell_{\sv} }{ \lambda \|\weightv_i^{(k)}-\weightv_{\sv}\|^2}$~and clip to $[0,1]$}
                \STATE Update ${\weightv_i}^{(k+1)}:= (1-\stepsize){\weightv_i}^{(k)}+\stepsize \,\weightv_{\sv}$
                \STATE {\small~~~~~~~and~ ${\ell_i}^{(k+1)}:= (1-\stepsize){\ell_i}^{(k)}+\stepsize\, \ell_{\sv}$}
                \STATE Update $\weightv^{(k+1)}\;:= \weightv^{(k)} + {\weightv_i}^{(k+1)} - {\weightv_i}^{(k)}$
                \STATE {\small~~~~~~~and~~ $\ell^{(k+1)}:= ~\ell^{(k)}+{\ell_i}^{(k+1)} \ \  - {\ell_i}^{(k)}$}
                \STATE {\small(Optionally: Update $\bar{\weightv}^{(k+1)}:=\frac{k}{k+2}  \bar{\weightv}^{(k)}+ \frac{2}{k+2}\weightv^{(k+1)}$)\!\!\!\!\!}
        \ENDFOR
\end{algorithmic}
\end{algorithm}

\paragraph{Convergence Results.}
The following main theorem shows that after $O(1/\varepsilon)$ many iterations, Algorithm~\ref{alg:FW_product} obtains an $\varepsilon$-approximate solution to \eqref{eq:opt_gen_convex_product}, and guaranteed $\varepsilon$-small duality gap (proof in~Appendix~\ref{sec:app_convergence_proof}).
Here the constant $\CfTotal := \sum_{i=1}^n \Cf^{(i)}$ is the sum of the (partial) curvature constants of $f$ with respect to the individual domain block $\domain^{(i)}$. We discuss this Lipschitz assumption on the gradient in more details in Appendix~\ref{sec:app_curvature}, where we compute the constant precisely for the structural SVM and obtain $ \CfTotal = \Cf/n$, where $\Cf$ is the classical Frank-Wolfe curvature. %

\begin{theorem}\label{thm:convergence_FW_product}
For each $k\ge 0$, the iterate $\dualvarv^{(k)}$ of %
Algorithm~\ref{alg:FW_product} (either using the predefined step-sizes, or using line-search) satisfies
$
\E\!\big[f(\dualvarv^{(k)})\big] - f(\dualvarv^*) \le \frac{2n}{k+2n}\big(\CfTotal+ h_0\big) \, ,
$
where $\dualvarv^*\in \domain$ is a solution to problem~(\ref{eq:opt_gen_convex_product}), $h_0 := f(\dualvarv^{(0)}) - f(\dualvarv^*)$ is the initial error at the starting point of the algorithm, and the expectation is over the random choice of the block~$i$ in the steps of the algorithm.

Furthermore, if Algorithm~\ref{alg:FW_product} is run for $K\ge 0$ iterations, then it has an iterate~$\dualvarv^{(\hat k)}$, $0\le \hat k\le K$, with duality gap bounded by
$
\E\!\big[g(\dualvarv^{(\hat k)})\big] \le \frac{6 n}{K+1} \big(\CfTotal+ h_0) \, .
$
\end{theorem}
\vspace{-3mm}

\paragraph{Application to the Structural SVM.}
Algorithm~\ref{alg:FW_product_SVM} applies the block-coordinate Frank-Wolfe algorithm with line-search to the structural SVM dual problem~\eqref{eq:svmstruct_nslack_dual}, maintaining only the primal variables~$\weightv$.
We see that Algorithm \ref{alg:FW_product_SVM} is equivalent to %
Algorithm~\ref{alg:FW_product}, by observing that the corresponding primal updates become $\weightv_{\sv} = A\sv_{[i]}$ and $\ell_{\sv} = \bv^T\sv_{[i]}$.
Here $\sv_{[i]}$ is the zero-padding of $\sv_{(i)} := \unit^{\outputvarv_i^*} \in \domain^{(i)}$ so that $\sv_{[i]} \in \domain$.
Note that Algorithm~\ref{alg:FW_product_SVM} has a primal parameter vector $\weightv_i$ ($=A\dualvarv_{[i]}$) for each datapoint $i$,
but that this does not significantly increase the storage cost of the algorithm since each $\weightv_i$ has
a sparsity pattern that is the union of the corresponding $\featuremapdiffv_i(\outputvarv_i^*)$ vectors. If the feature vectors are not sparse, it might be more efficient to work directly in the dual instead (see the kernelized version below). 
The line-search is analogous to the batch Frank-Wolfe case discussed above, and formalized in Appendix~\ref{ssec:app_block_implement}. 

By applying Theorem~\ref{thm:convergence_FW_product} to the SVM case where $\CfTotal = \Cf/n = 4 R^2 / \lambda n$ (in the worst case), we get that the number of iterations needed for our new block-wise Algorithm~\ref{alg:FW_product_SVM} to obtain a specific accuracy $\varepsilon$ is the same as for the batch version in Algorithm~\ref{alg:FW_SVM} (under the assumption that the initial error $h_0$ is smaller than $ 4 R^2 / \lambda n$), even though each iteration takes $n$ times fewer oracle calls. %
If $h_0 > 4 R^2 / \lambda n$, we can use the fact that Algorithm~\ref{alg:FW_product_SVM} is using line-search to get a weaker dependence on $h_0$ in the rate (Theorem~\ref{thm:fast_warmup}). We summarize the overall rate as follows (proof in Appendix~\ref{ssec:app_SVM_proofs}):
\vspace{-1mm}
\begin{theorem}%
\label{thm:convergence_FW_product_SVM}
If $\Lmax \le \frac{4R^2}{\lambda n}$ (so $h_0 \leq \frac{4R^2}{\lambda n}$), then Algorithm~\ref{alg:FW_product_SVM} obtains an $\varepsilon$-approximate solution to the structural SVM dual problem~\eqref{eq:svmstruct_nslack_dual} and expected duality gap $\E [g(\dualvarv^{(k)})] \le\varepsilon$ after at most $O\left(\frac{R^2}{\lambda\varepsilon}\right)$ iterations, where each iteration costs a single oracle call.

If $\Lmax > \frac{4R^2}{\lambda n}$, then it requires at most an additional (constant in $\varepsilon$) number of $O\left(n \log\left( \frac{\lambda n \Lmax}{R^2}\right) \right)$ %
steps to get the same error and duality gap guarantees. %
\end{theorem}\vspace{-1mm}
In terms of $\varepsilon$, the $O(1/\varepsilon)$ convergence rate above is similar to existing stochastic subgradient and cutting-plane methods.
However, unlike stochastic subgradient methods, the block-coordinate Frank-Wolfe method allows us to compute the optimal step-size at each iteration (while for an additional pass through the data %
we can evaluate the duality gap~\eqref{eq:duality_gap} to allow us to decide when to terminate the algorithm in practice).
Further, unlike cutting-plane methods which require $n$ oracle calls per iteration, this rate is achieved `online', using only a single oracle call per iteration.

\vspace{-2.5mm}
\paragraph{Approximate Subproblems and Decoding.}
Interestingly, we can show that all the convergence results presented in this paper also hold if only approximate minimizers of the linear subproblems are used instead of exact minimizers.
If we are using an approximate oracle giving candidate directions $\sv_{(i)}$ in Algorithm~\ref{alg:FW_product} (or $\sv$ in Algorithm \ref{alg:FW}) with a \emph{multiplicative} accuracy $\mapprox\in(0,1]$ (with respect to the the duality gap~\eqref{eq:duality_gap} on the current block), then the above convergence bounds from Theorem~\ref{thm:convergence_FW_product} still apply.
The only change is that the convergence is slowed by a factor of $1/\mapprox^2$. We prove this generalization in the Theorems of Appendix~\ref{sec:app_convergence_proof}. 
For structural SVMs, this significantly improves the applicability to large-scale problems, where \emph{exact} decoding is often too costly but \emph{approximate} loss-augmented decoding may be possible.

\vspace{-2mm}
\paragraph{Kernelized Algorithms.}%
Both Algorithms~\ref{alg:FW_SVM} and \ref{alg:FW_product_SVM} can directly be used with kernels by maintaining the sparse dual variables $\dualvarv^{(k)}$ instead of the primal variables $\weightv^{(k)}$. In this case, the classifier is only given implicitly as a sparse combination of the corresponding kernel functions, i.e. $\weightv = A\dualvarv$.
Using our Algorithm~\ref{alg:FW_product_SVM}, we obtain the currently best known bound on the \emph{number of support vectors}, %
i.e. a guaranteed $\varepsilon$-approximation with only $O(\frac{R^2}{\lambda\varepsilon})$ support vectors. In comparison, the standard cutting plane method \citep{Joachims:2009ex} adds $n$ support vectors $\featuremapdiffv_i(\outputvarv_i^*)$ at each iteration.
More details on the kernelized variant of Algorithm~\ref{alg:FW_product_SVM} are discussed in Appendix~\ref{ssec:app_kernelized}.

\section{Experiments}
\label{sec:experiments}

\begin{figure*}[htb]
    \centering
    \begin{subfigure}[t]{0.32\linewidth}
        \centering
        \def\xlabel{effective passes}
        \def\xmin{1}
        \def\xmax{149}
        \def\ymin{0.01}
        \def\ymax{1}
        \def\xmode{normal}
        \def\showlegend{0}
        \def\showeg{1}
        \def\showfrankwolfebatch{1}
        \def\showcuttingplanes{1}
        \def\legendpos{north east}
        \def\experimentprefix{include/data/dataset=ocr2_lambda=0.010000}
        \small
\begin{tikzpicture}[scale=0.63]

\begin{axis}[
xlabel=\xlabel,
ylabel=primal suboptimality for problem \eqref{eq:svmstruct_nslack_primal},
xmin=\xmin,
xmax=\xmax,
ymin=\ymin,
ymax=\ymax,
enlargelimits=false, area style,
ymode=log,
xmode=\xmode,
line legend,
legend pos=\legendpos,
]


\addplot[fill=blue,draw=none,forget plot,opacity=0.2] table[x index=0,y
index=1, header=true, col sep=comma]
{\experimentprefix/product-LS_confidence.txt};

\addplot [
color=blue,
solid,
style=thick,
mark=triangle,
mark repeat=20,
]
table[x index=0,y index=2, header=true, col sep=comma]
{\experimentprefix/product-LS.txt};
\ifnum \showlegend=1
{
\addlegendentry{BCFW}
}
\fi

\ifdefined \showtavg
{
\addplot [
color=cyan,
solid,
style=thick,
mark=triangle*,
mark repeat=20,
mark options=solid
]
table[x index=0,y index=2, header=true, col sep=comma]
{\experimentprefix/product-LS-opt.txt};
\ifnum \showlegend=1
{
\addlegendentry{BCFW-tavg}
}
\fi
}
\fi


\addplot[fill=cyan,draw=none,forget plot,opacity=0.3] table[x index=0,y
index=1, header=true, col sep=comma]
{\experimentprefix/product-LS-wavg_confidence.txt};

\addplot [
color=cyan,
densely dotted,
style=thick,
mark=square,
mark repeat=20,
mark options=solid
]
table[x index=0,y index=2, header=true, col sep=comma]
{\experimentprefix/product-LS-wavg.txt};
\ifnum \showlegend=1
{
\addlegendentry{BCFW-wavg}
}
\fi


\addplot[fill=green,draw=none,forget plot,opacity=0.2] table[x index=0,y
index=1, header=true, col sep=comma]
{\experimentprefix/pegasos_confidence.txt};

\addplot [
color=green,
solid,
style=thick,
mark=o,
mark repeat=20,
]
table[x index=0,y index=2, header=true, col sep=comma]
{\experimentprefix/pegasos.txt};
\ifnum \showlegend=1
{
\addlegendentry{SSG}
}
\fi

\ifdefined \showtavg
{
\addplot [
color=purple,
solid,
style=thick,
mark=diamond,
mark repeat=20,
]
table[x index=0,y index=2, header=true, col sep=comma]
{\experimentprefix/optimalSG.txt};
\ifnum \showlegend=1
{
\addlegendentry{SSG-tavg}
}
\fi
}
\fi

\addplot[fill=purple,draw=none,forget plot,opacity=0.2] table[x index=0,y
index=1, header=true, col sep=comma]
{\experimentprefix/pegasos-wavg_confidence.txt};

\addplot [
color=purple,
densely dotted,
style=thick,
mark=*,
mark repeat=20,
]
table[x index=0,y index=2, header=true, col sep=comma]
{\experimentprefix/pegasos-wavg.txt};
\ifnum \showlegend=1
{
\addlegendentry{SSG-wavg}
}
\fi

\ifdefined \showeg
{

\addplot[fill=gray,draw=none,forget plot,opacity=0.3] table[x index=0,y
index=1, header=true, col sep=comma]
{\experimentprefix/OEG_confidence.txt};

\addplot [
color=gray,
solid,
style=thick,
mark=pentagon*,
mark repeat=20,
]
table[x index=0,y index=2, header=true, col sep=comma]
{\experimentprefix/OEG.txt};
\ifnum \showlegend=1
{
\addlegendentry{online-EG}
}
\fi
}
\fi

\ifdefined \showfrankwolfebatch
{

\addplot [
color=orange,
densely dashed,
style=thick,
mark=pentagon,
mark repeat=20,
mark options=solid
]
table[x index=0,y index=1, header=false, col sep=comma]
{\experimentprefix/frankWolfe-LS_lowerenvelope.txt};
\ifnum \showlegend=1
{
\addlegendentry{FW}
}
\fi
}
\fi

\ifdefined \showcuttingplanes
{

\addplot [
color=black,
densely dashed,
style=thick,
mark=square,
mark repeat=20,
mark options=solid
]
table[x index=0,y index=1, header=false, col sep=comma]
{\experimentprefix/svmStruct_lowerenvelope.txt};
\ifnum \showlegend=1
{
\addlegendentry{cutting plane}
}
\fi
}
\fi

\end{axis}

\end{tikzpicture}
\normalsize\vspace{-2mm}
        \caption{OCR dataset, $\lambda=0.01$.}
    \end{subfigure}
    \begin{subfigure}[t]{0.32\linewidth}
        \centering
        \def\xlabel{effective passes}
        \def\xmin{1}
        \def\xmax{149}
        \def\ymin{0.02}
        \def\ymax{10}
        \def\xmode{normal}
        \def\showlegend{0}
        \def\showeg{1}
        \def\showfrankwolfebatch{1}
        \def\showcuttingplanes{1}
        \def\legendpos{north east}
        \def\experimentprefix{include/data/dataset=ocr2_lambda=0.001000}
        \small
\begin{tikzpicture}[scale=0.63]

\begin{axis}[
xlabel=\xlabel,
ylabel=primal suboptimality for problem \eqref{eq:svmstruct_nslack_primal},
xmin=\xmin,
xmax=\xmax,
ymin=\ymin,
ymax=\ymax,
enlargelimits=false, area style,
ymode=log,
xmode=\xmode,
line legend,
legend pos=\legendpos,
]


\addplot[fill=blue,draw=none,forget plot,opacity=0.2] table[x index=0,y
index=1, header=true, col sep=comma]
{\experimentprefix/product-LS_confidence.txt};

\addplot [
color=blue,
solid,
style=thick,
mark=triangle,
mark repeat=20,
]
table[x index=0,y index=2, header=true, col sep=comma]
{\experimentprefix/product-LS.txt};
\ifnum \showlegend=1
{
\addlegendentry{BCFW}
}
\fi

\ifdefined \showtavg
{
\addplot [
color=cyan,
solid,
style=thick,
mark=triangle*,
mark repeat=20,
mark options=solid
]
table[x index=0,y index=2, header=true, col sep=comma]
{\experimentprefix/product-LS-opt.txt};
\ifnum \showlegend=1
{
\addlegendentry{BCFW-tavg}
}
\fi
}
\fi


\addplot[fill=cyan,draw=none,forget plot,opacity=0.3] table[x index=0,y
index=1, header=true, col sep=comma]
{\experimentprefix/product-LS-wavg_confidence.txt};

\addplot [
color=cyan,
densely dotted,
style=thick,
mark=square,
mark repeat=20,
mark options=solid
]
table[x index=0,y index=2, header=true, col sep=comma]
{\experimentprefix/product-LS-wavg.txt};
\ifnum \showlegend=1
{
\addlegendentry{BCFW-wavg}
}
\fi


\addplot[fill=green,draw=none,forget plot,opacity=0.2] table[x index=0,y
index=1, header=true, col sep=comma]
{\experimentprefix/pegasos_confidence.txt};

\addplot [
color=green,
solid,
style=thick,
mark=o,
mark repeat=20,
]
table[x index=0,y index=2, header=true, col sep=comma]
{\experimentprefix/pegasos.txt};
\ifnum \showlegend=1
{
\addlegendentry{SSG}
}
\fi

\ifdefined \showtavg
{
\addplot [
color=purple,
solid,
style=thick,
mark=diamond,
mark repeat=20,
]
table[x index=0,y index=2, header=true, col sep=comma]
{\experimentprefix/optimalSG.txt};
\ifnum \showlegend=1
{
\addlegendentry{SSG-tavg}
}
\fi
}
\fi

\addplot[fill=purple,draw=none,forget plot,opacity=0.2] table[x index=0,y
index=1, header=true, col sep=comma]
{\experimentprefix/pegasos-wavg_confidence.txt};

\addplot [
color=purple,
densely dotted,
style=thick,
mark=*,
mark repeat=20,
]
table[x index=0,y index=2, header=true, col sep=comma]
{\experimentprefix/pegasos-wavg.txt};
\ifnum \showlegend=1
{
\addlegendentry{SSG-wavg}
}
\fi

\ifdefined \showeg
{

\addplot[fill=gray,draw=none,forget plot,opacity=0.3] table[x index=0,y
index=1, header=true, col sep=comma]
{\experimentprefix/OEG_confidence.txt};

\addplot [
color=gray,
solid,
style=thick,
mark=pentagon*,
mark repeat=20,
]
table[x index=0,y index=2, header=true, col sep=comma]
{\experimentprefix/OEG.txt};
\ifnum \showlegend=1
{
\addlegendentry{online-EG}
}
\fi
}
\fi

\ifdefined \showfrankwolfebatch
{

\addplot [
color=orange,
densely dashed,
style=thick,
mark=pentagon,
mark repeat=20,
mark options=solid
]
table[x index=0,y index=1, header=false, col sep=comma]
{\experimentprefix/frankWolfe-LS_lowerenvelope.txt};
\ifnum \showlegend=1
{
\addlegendentry{FW}
}
\fi
}
\fi

\ifdefined \showcuttingplanes
{

\addplot [
color=black,
densely dashed,
style=thick,
mark=square,
mark repeat=20,
mark options=solid
]
table[x index=0,y index=1, header=false, col sep=comma]
{\experimentprefix/svmStruct_lowerenvelope.txt};
\ifnum \showlegend=1
{
\addlegendentry{cutting plane}
}
\fi
}
\fi

\end{axis}

\end{tikzpicture}
\normalsize\vspace{-2mm}
        \caption{OCR dataset, $\lambda=0.001$.}
    \end{subfigure}
    \begin{subfigure}[t]{0.32\linewidth}
        \centering
        \def\xlabel{effective passes}
        \def\xmin{1}
        \def\xmax{149}
        \def\ymin{0.06}
        \def\ymax{10}
        \def\xmode{normal}
        \def\showlegend{1}
        \def\showeg{1}
        \def\showfrankwolfebatch{1}
        \def\showcuttingplanes{1}
        \def\legendpos{north east}
        \def\experimentprefix{include/data/dataset=ocr2_lambda=0.000160}
        \small
\begin{tikzpicture}[scale=0.63]

\begin{axis}[
xlabel=\xlabel,
ylabel=primal suboptimality for problem \eqref{eq:svmstruct_nslack_primal},
xmin=\xmin,
xmax=\xmax,
ymin=\ymin,
ymax=\ymax,
enlargelimits=false, area style,
ymode=log,
xmode=\xmode,
line legend,
legend pos=\legendpos,
]


\addplot[fill=blue,draw=none,forget plot,opacity=0.2] table[x index=0,y
index=1, header=true, col sep=comma]
{\experimentprefix/product-LS_confidence.txt};

\addplot [
color=blue,
solid,
style=thick,
mark=triangle,
mark repeat=20,
]
table[x index=0,y index=2, header=true, col sep=comma]
{\experimentprefix/product-LS.txt};
\ifnum \showlegend=1
{
\addlegendentry{BCFW}
}
\fi

\ifdefined \showtavg
{
\addplot [
color=cyan,
solid,
style=thick,
mark=triangle*,
mark repeat=20,
mark options=solid
]
table[x index=0,y index=2, header=true, col sep=comma]
{\experimentprefix/product-LS-opt.txt};
\ifnum \showlegend=1
{
\addlegendentry{BCFW-tavg}
}
\fi
}
\fi


\addplot[fill=cyan,draw=none,forget plot,opacity=0.3] table[x index=0,y
index=1, header=true, col sep=comma]
{\experimentprefix/product-LS-wavg_confidence.txt};

\addplot [
color=cyan,
densely dotted,
style=thick,
mark=square,
mark repeat=20,
mark options=solid
]
table[x index=0,y index=2, header=true, col sep=comma]
{\experimentprefix/product-LS-wavg.txt};
\ifnum \showlegend=1
{
\addlegendentry{BCFW-wavg}
}
\fi


\addplot[fill=green,draw=none,forget plot,opacity=0.2] table[x index=0,y
index=1, header=true, col sep=comma]
{\experimentprefix/pegasos_confidence.txt};

\addplot [
color=green,
solid,
style=thick,
mark=o,
mark repeat=20,
]
table[x index=0,y index=2, header=true, col sep=comma]
{\experimentprefix/pegasos.txt};
\ifnum \showlegend=1
{
\addlegendentry{SSG}
}
\fi

\ifdefined \showtavg
{
\addplot [
color=purple,
solid,
style=thick,
mark=diamond,
mark repeat=20,
]
table[x index=0,y index=2, header=true, col sep=comma]
{\experimentprefix/optimalSG.txt};
\ifnum \showlegend=1
{
\addlegendentry{SSG-tavg}
}
\fi
}
\fi

\addplot[fill=purple,draw=none,forget plot,opacity=0.2] table[x index=0,y
index=1, header=true, col sep=comma]
{\experimentprefix/pegasos-wavg_confidence.txt};

\addplot [
color=purple,
densely dotted,
style=thick,
mark=*,
mark repeat=20,
]
table[x index=0,y index=2, header=true, col sep=comma]
{\experimentprefix/pegasos-wavg.txt};
\ifnum \showlegend=1
{
\addlegendentry{SSG-wavg}
}
\fi

\ifdefined \showeg
{

\addplot[fill=gray,draw=none,forget plot,opacity=0.3] table[x index=0,y
index=1, header=true, col sep=comma]
{\experimentprefix/OEG_confidence.txt};

\addplot [
color=gray,
solid,
style=thick,
mark=pentagon*,
mark repeat=20,
]
table[x index=0,y index=2, header=true, col sep=comma]
{\experimentprefix/OEG.txt};
\ifnum \showlegend=1
{
\addlegendentry{online-EG}
}
\fi
}
\fi

\ifdefined \showfrankwolfebatch
{

\addplot [
color=orange,
densely dashed,
style=thick,
mark=pentagon,
mark repeat=20,
mark options=solid
]
table[x index=0,y index=1, header=false, col sep=comma]
{\experimentprefix/frankWolfe-LS_lowerenvelope.txt};
\ifnum \showlegend=1
{
\addlegendentry{FW}
}
\fi
}
\fi

\ifdefined \showcuttingplanes
{

\addplot [
color=black,
densely dashed,
style=thick,
mark=square,
mark repeat=20,
mark options=solid
]
table[x index=0,y index=1, header=false, col sep=comma]
{\experimentprefix/svmStruct_lowerenvelope.txt};
\ifnum \showlegend=1
{
\addlegendentry{cutting plane}
}
\fi
}
\fi

\end{axis}

\end{tikzpicture}
\normalsize\vspace{-2mm}
        \caption{OCR dataset, $\lambda=1/n$.}
    \end{subfigure}
    \begin{subfigure}[t]{0.32\linewidth}
        \centering
        \def\xlabel{effective passes}
        \def\xmin{1e-1}
        \def\xmax{48}
        \def\ymin{0.002}
        \def\ymax{5}
        \def\xmode{log}
        \def\showlegend{1}
        \def\legendpos{south west}
        \def\showfrankwolfebatch{1}
        \def\showcuttingplanes{1}
        \def\experimentprefix{include/data/dataset=conll_lambda=0.000112}
        \small
\begin{tikzpicture}[scale=0.63]

\begin{axis}[
xlabel=\xlabel,
ylabel=primal suboptimality for problem \eqref{eq:svmstruct_nslack_primal},
xmin=\xmin,
xmax=\xmax,
ymin=\ymin,
ymax=\ymax,
enlargelimits=false, area style,
ymode=log,
xmode=\xmode,
line legend,
legend pos=\legendpos,
]


\addplot[fill=blue,draw=none,forget plot,opacity=0.2] table[x index=0,y
index=1, header=true, col sep=comma]
{\experimentprefix/product-LS_confidence.txt};

\addplot [
color=blue,
solid,
style=thick,
mark=triangle,
mark repeat=20,
]
table[x index=0,y index=2, header=true, col sep=comma]
{\experimentprefix/product-LS.txt};
\ifnum \showlegend=1
{
\addlegendentry{BCFW}
}
\fi

\ifdefined \showtavg
{
\addplot [
color=cyan,
solid,
style=thick,
mark=triangle*,
mark repeat=20,
mark options=solid
]
table[x index=0,y index=2, header=true, col sep=comma]
{\experimentprefix/product-LS-opt.txt};
\ifnum \showlegend=1
{
\addlegendentry{BCFW-tavg}
}
\fi
}
\fi


\addplot[fill=cyan,draw=none,forget plot,opacity=0.3] table[x index=0,y
index=1, header=true, col sep=comma]
{\experimentprefix/product-LS-wavg_confidence.txt};

\addplot [
color=cyan,
densely dotted,
style=thick,
mark=square,
mark repeat=20,
mark options=solid
]
table[x index=0,y index=2, header=true, col sep=comma]
{\experimentprefix/product-LS-wavg.txt};
\ifnum \showlegend=1
{
\addlegendentry{BCFW-wavg}
}
\fi


\addplot[fill=green,draw=none,forget plot,opacity=0.2] table[x index=0,y
index=1, header=true, col sep=comma]
{\experimentprefix/pegasos_confidence.txt};

\addplot [
color=green,
solid,
style=thick,
mark=o,
mark repeat=20,
]
table[x index=0,y index=2, header=true, col sep=comma]
{\experimentprefix/pegasos.txt};
\ifnum \showlegend=1
{
\addlegendentry{SSG}
}
\fi

\ifdefined \showtavg
{
\addplot [
color=purple,
solid,
style=thick,
mark=diamond,
mark repeat=20,
]
table[x index=0,y index=2, header=true, col sep=comma]
{\experimentprefix/optimalSG.txt};
\ifnum \showlegend=1
{
\addlegendentry{SSG-tavg}
}
\fi
}
\fi

\addplot[fill=purple,draw=none,forget plot,opacity=0.2] table[x index=0,y
index=1, header=true, col sep=comma]
{\experimentprefix/pegasos-wavg_confidence.txt};

\addplot [
color=purple,
densely dotted,
style=thick,
mark=*,
mark repeat=20,
]
table[x index=0,y index=2, header=true, col sep=comma]
{\experimentprefix/pegasos-wavg.txt};
\ifnum \showlegend=1
{
\addlegendentry{SSG-wavg}
}
\fi

\ifdefined \showeg
{

\addplot[fill=gray,draw=none,forget plot,opacity=0.3] table[x index=0,y
index=1, header=true, col sep=comma]
{\experimentprefix/OEG_confidence.txt};

\addplot [
color=gray,
solid,
style=thick,
mark=pentagon*,
mark repeat=20,
]
table[x index=0,y index=2, header=true, col sep=comma]
{\experimentprefix/OEG.txt};
\ifnum \showlegend=1
{
\addlegendentry{online-EG}
}
\fi
}
\fi

\ifdefined \showfrankwolfebatch
{

\addplot [
color=orange,
densely dashed,
style=thick,
mark=pentagon,
mark repeat=20,
mark options=solid
]
table[x index=0,y index=1, header=false, col sep=comma]
{\experimentprefix/frankWolfe-LS_lowerenvelope.txt};
\ifnum \showlegend=1
{
\addlegendentry{FW}
}
\fi
}
\fi

\ifdefined \showcuttingplanes
{

\addplot [
color=black,
densely dashed,
style=thick,
mark=square,
mark repeat=20,
mark options=solid
]
table[x index=0,y index=1, header=false, col sep=comma]
{\experimentprefix/svmStruct_lowerenvelope.txt};
\ifnum \showlegend=1
{
\addlegendentry{cutting plane}
}
\fi
}
\fi

\end{axis}

\end{tikzpicture}
\normalsize\vspace{-2mm}
        \caption{CoNLL dataset, $\lambda=1/n$.}
        \label{fig:results_conll_stochastic}
    \end{subfigure}
    \begin{subfigure}[t]{0.32\linewidth}
        \centering
        \def\xlabel{effective passes}
        \def\showlegend{0}
        \def\xmin{1e-1}
        \def\xmax{48}
        \def\ymin{0.04}
        \def\ymax{0.10}
        \def\xmode{log}
        \def\showconfidence{1}
        \def\showfrankwolfebatch{1}
        \def\showcuttingplanes{1}
        \def\legendpos{north west}
        \def\experimentprefix{include/data/dataset=conll_lambda=0.000112}
        \small
\begin{tikzpicture}[scale=0.63]

\pgfplotsset{y tick label style={ 
         scaled ticks=false, 
         /pgf/number format/fixed zerofill, 
         /pgf/number format/fixed, 
         /pgf/number format/precision=3, 
     }
}

\begin{axis}[
xlabel=\xlabel,
ylabel=test error,
xmin=\xmin,
xmax=\xmax,
ymin=\ymin,
ymax=\ymax,
enlargelimits=false, area style,
ymode=normal,
xmode=\xmode,
line legend,
legend pos=\legendpos,
]


\ifdefined \showconfidence
{
\addplot[fill=blue,draw=none,forget plot,opacity=0.2] table[x index=0,y
index=2, header=true, col sep=comma]
{\experimentprefix/product-LS_confidence.txt};
}
\fi

\addplot [
color=blue,
solid,
style=thick,
mark=triangle,
mark repeat=20,
]
table[x index=0,y index=3, header=true, col sep=comma]
{\experimentprefix/product-LS.txt};
\ifnum \showlegend=1
{
\addlegendentry{BCFW}
}
\fi

\ifdefined \showtavg
{

\addplot [
color=cyan,
solid,
style=thick,
mark=triangle*,
mark repeat=20,
mark options=solid
]
table[x index=0,y index=3, header=true, col sep=comma]
{\experimentprefix/product-LS-opt.txt};
\ifnum \showlegend=1
{
\addlegendentry{BCFW-tavg}
}
\fi
}
\fi

\addplot[fill=cyan,draw=none,forget plot,opacity=0.3] table[x index=0,y
index=2, header=true, col sep=comma]
{\experimentprefix/product-LS-wavg_confidence.txt};

\addplot [
color=cyan,
densely dotted,
style=thick,
mark=square,
mark repeat=20,
mark options=solid
]
table[x index=0,y index=3, header=true, col sep=comma]
{\experimentprefix/product-LS-wavg.txt};
\ifnum \showlegend=1
{
\addlegendentry{BCFW-wavg}
}
\fi


\ifdefined \showconfidence
{
\addplot[fill=green,draw=none,forget plot,opacity=0.3] table[x index=0,y
index=2, header=true, col sep=comma]
{\experimentprefix/pegasos_confidence.txt};
}
\fi

\addplot [
color=green,
solid,
style=thick,
mark=o,
mark repeat=20,
]
table[x index=0,y index=3, header=true, col sep=comma]
{\experimentprefix/pegasos.txt};
\ifnum \showlegend=1
{
\addlegendentry{SSG}
}
\fi

\ifdefined \showtavg
{
\addplot [
color=purple,
densely dotted,
style=thick,
mark=diamond,
mark repeat=20,
]
table[x index=0,y index=3, header=true, col sep=comma]
{\experimentprefix/optimalSG.txt};
\ifnum \showlegend=1
{
\addlegendentry{SSG-tavg}
}
\fi
}
\fi


\addplot[fill=purple,draw=none,forget plot,opacity=0.2] table[x index=0,y
index=2, header=true, col sep=comma]
{\experimentprefix/pegasos-wavg_confidence.txt};

\addplot [
color=purple,
densely dotted,
style=thick,
mark=*,
mark repeat=20,
]
table[x index=0,y index=3, header=true, col sep=comma]
{\experimentprefix/pegasos-wavg.txt};
\ifnum \showlegend=1
{
\addlegendentry{SSG-wavg}
}
\fi

\ifdefined \showeg
{

\ifdefined \showconfidence
{
\addplot[fill=gray,draw=none,forget plot,opacity=0.3] table[x index=0,y
index=2, header=true, col sep=comma]
{\experimentprefix/OEG_confidence.txt};
}
\fi

\addplot [
color=gray,
solid,
style=thick,
mark=pentagon*,
mark repeat=20,
]
table[x index=0,y index=3, header=true, col sep=comma]
{\experimentprefix/OEG.txt};
\ifnum \showlegend=1
{
\addlegendentry{online-EG}
}
\fi
}
\fi

\ifdefined \showfrankwolfebatch
{

\addplot [
color=orange,
densely dashed,
style=thick,
mark=pentagon,
mark repeat=20,
mark options=solid
]
table[x index=0,y index=2, header=false, col sep=comma]
{\experimentprefix/frankWolfe-LS_lowerenvelope.txt};
\ifnum \showlegend=1
{
\addlegendentry{FW}
}
\fi
}
\fi

\ifdefined \showcuttingplanes
{

\addplot [
color=black,
densely dashed,
style=thick,
mark=square,
mark repeat=20,
mark options=solid
]
table[x index=0,y index=2, header=false, col sep=comma]
{\experimentprefix/svmStruct_lowerenvelope.txt};
\ifnum \showlegend=1
{
\addlegendentry{cutting plane}
}
\fi
}
\fi

\end{axis}

\end{tikzpicture}
\normalsize\vspace{-2mm}
        \caption{Test error for $\lambda=1/n$ on CoNLL.}
        \label{fig:results_conll_stochastic_testerror}
    \end{subfigure}
    \begin{subfigure}[t]{0.32\linewidth}
        \centering
        \def\xlabel{effective passes}
        \def\xmin{1e-2}
        \def\xmax{30}
        \def\ymin{0.0001}
        \def\ymax{100}
        \def\xmode{log}
        \def\showlegend{1}
        \def\showfrankwolfebatch{1}
        \def\showcuttingplanes{1}
        \def\legendpos{south west}
        \def\experimentprefix{include/data/dataset=matching2_lambda=0.001000}
        \small
\begin{tikzpicture}[scale=0.63]

\begin{axis}[
xlabel=\xlabel,
ylabel=primal suboptimality for problem \eqref{eq:svmstruct_nslack_primal},
xmin=\xmin,
xmax=\xmax,
ymin=\ymin,
ymax=\ymax,
enlargelimits=false, area style,
ymode=log,
xmode=\xmode,
line legend,
legend pos=\legendpos,
]


\addplot[fill=blue,draw=none,forget plot,opacity=0.2] table[x index=0,y
index=1, header=true, col sep=comma]
{\experimentprefix/product-LS_confidence.txt};

\addplot [
color=blue,
solid,
style=thick,
mark=triangle,
mark repeat=20,
]
table[x index=0,y index=2, header=true, col sep=comma]
{\experimentprefix/product-LS.txt};
\ifnum \showlegend=1
{
\addlegendentry{BCFW}
}
\fi

\ifdefined \showtavg
{
\addplot [
color=cyan,
solid,
style=thick,
mark=triangle*,
mark repeat=20,
mark options=solid
]
table[x index=0,y index=2, header=true, col sep=comma]
{\experimentprefix/product-LS-opt.txt};
\ifnum \showlegend=1
{
\addlegendentry{BCFW-tavg}
}
\fi
}
\fi


\addplot[fill=cyan,draw=none,forget plot,opacity=0.3] table[x index=0,y
index=1, header=true, col sep=comma]
{\experimentprefix/product-LS-wavg_confidence.txt};

\addplot [
color=cyan,
densely dotted,
style=thick,
mark=square,
mark repeat=20,
mark options=solid
]
table[x index=0,y index=2, header=true, col sep=comma]
{\experimentprefix/product-LS-wavg.txt};
\ifnum \showlegend=1
{
\addlegendentry{BCFW-wavg}
}
\fi


\addplot[fill=green,draw=none,forget plot,opacity=0.2] table[x index=0,y
index=1, header=true, col sep=comma]
{\experimentprefix/pegasos_confidence.txt};

\addplot [
color=green,
solid,
style=thick,
mark=o,
mark repeat=20,
]
table[x index=0,y index=2, header=true, col sep=comma]
{\experimentprefix/pegasos.txt};
\ifnum \showlegend=1
{
\addlegendentry{SSG}
}
\fi

\ifdefined \showtavg
{
\addplot [
color=purple,
solid,
style=thick,
mark=diamond,
mark repeat=20,
]
table[x index=0,y index=2, header=true, col sep=comma]
{\experimentprefix/optimalSG.txt};
\ifnum \showlegend=1
{
\addlegendentry{SSG-tavg}
}
\fi
}
\fi

\addplot[fill=purple,draw=none,forget plot,opacity=0.2] table[x index=0,y
index=1, header=true, col sep=comma]
{\experimentprefix/pegasos-wavg_confidence.txt};

\addplot [
color=purple,
densely dotted,
style=thick,
mark=*,
mark repeat=20,
]
table[x index=0,y index=2, header=true, col sep=comma]
{\experimentprefix/pegasos-wavg.txt};
\ifnum \showlegend=1
{
\addlegendentry{SSG-wavg}
}
\fi

\ifdefined \showeg
{

\addplot[fill=gray,draw=none,forget plot,opacity=0.3] table[x index=0,y
index=1, header=true, col sep=comma]
{\experimentprefix/OEG_confidence.txt};

\addplot [
color=gray,
solid,
style=thick,
mark=pentagon*,
mark repeat=20,
]
table[x index=0,y index=2, header=true, col sep=comma]
{\experimentprefix/OEG.txt};
\ifnum \showlegend=1
{
\addlegendentry{online-EG}
}
\fi
}
\fi

\ifdefined \showfrankwolfebatch
{

\addplot [
color=orange,
densely dashed,
style=thick,
mark=pentagon,
mark repeat=20,
mark options=solid
]
table[x index=0,y index=1, header=false, col sep=comma]
{\experimentprefix/frankWolfe-LS_lowerenvelope.txt};
\ifnum \showlegend=1
{
\addlegendentry{FW}
}
\fi
}
\fi

\ifdefined \showcuttingplanes
{

\addplot [
color=black,
densely dashed,
style=thick,
mark=square,
mark repeat=20,
mark options=solid
]
table[x index=0,y index=1, header=false, col sep=comma]
{\experimentprefix/svmStruct_lowerenvelope.txt};
\ifnum \showlegend=1
{
\addlegendentry{cutting plane}
}
\fi
}
\fi

\end{axis}

\end{tikzpicture}
\normalsize\vspace{-2mm}
        \caption{Matching dataset, $\lambda=0.001$.}
    \end{subfigure}
    \caption{
    The shaded areas for the stochastic methods (\emph{BCFW}, \emph{SSG} and \emph{online-EG}) indicate the
    worst and best objective achieved in 10 randomized runs.
    The top row compares the suboptimality achieved by different solvers for different regularization parameters $\lambda$. For large $\lambda$ (a), the stochastic algorithms (\emph{BCFW} and \emph{SSG})
    perform considerably better than the batch solvers (\emph{cutting plane} and
    \emph{FW}).
    For a small $\lambda$ (c), even the batch solvers achieve a
    lower objective earlier on than \emph{SSG}. %
    Our proposed
    \emph{BCFW} algorithm achieves a low objective
    in both settings.
    (d) shows the
    convergence for CoNLL with the first passes in more details.
    Here \emph{BCFW} already results
    in a low objective even after seeing only few datapoints. 
	The advantage is less clear for the test error in (e) though, where \emph{SSG-wavg} does surprisingly well. 
    Finally, (f) compares the methods for the matching prediction task.
    }
    \label{fig:results}
    \vspace{-3mm}
\end{figure*}
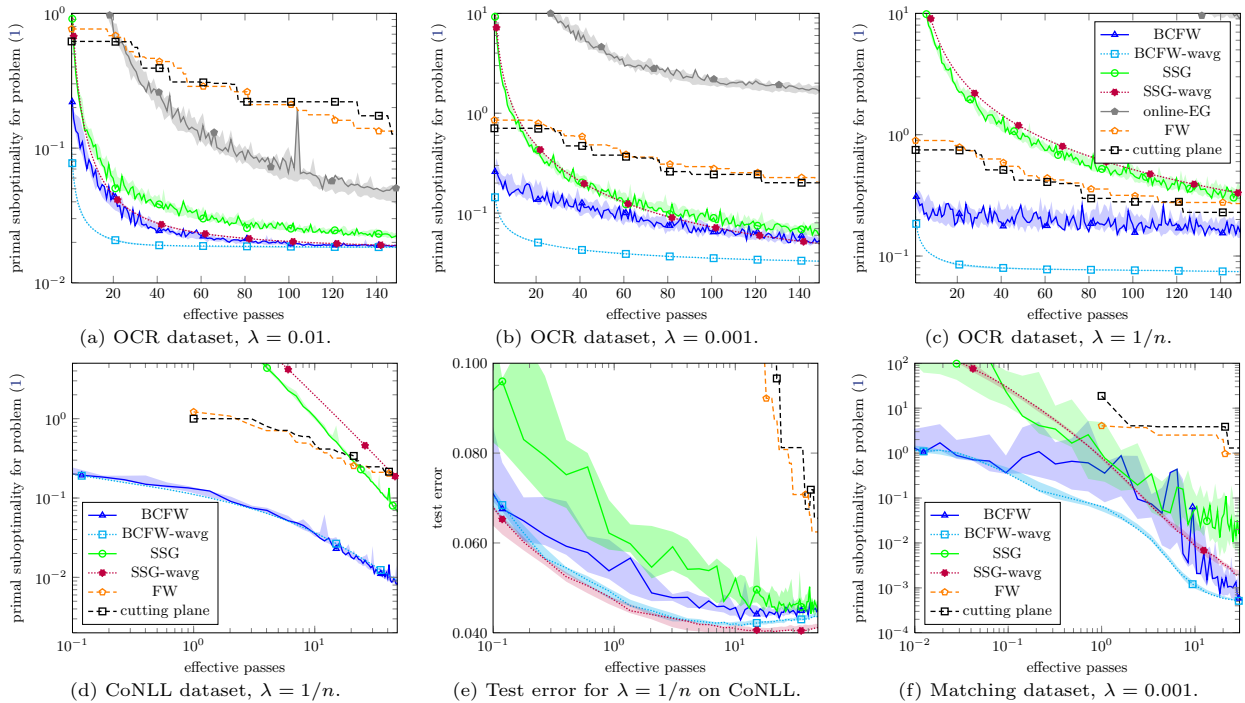

We compare our novel Frank-Wolfe approach to existing algorithms
for training structural SVMs on the OCR dataset {\small($n=6251,d=4028$)} from~\citet{Taskar2003} and the CoNLL
dataset {\small($n=8936,d=1643026$)} from~\citet{Sang2000}. Both datasets are sequence
labeling tasks, where the loss-augmented decoding problem can be solved
exactly by the Viterbi algorithm.
Our third application is a word alignment problem between sentences in different languages in the setting of~\citet{Taskar06extrag} {\small($n=5000,d=82$)}. Here, the structured labels are bipartite matchings, for which computing marginals over labels as required by the methods of~\citet{Collins2008,Zhang:2011:ATM3net} is intractable, but loss-augmented decoding can be done efficiently by solving a min-cost flow problem.

We compare Algorithms~\ref{alg:FW_SVM} and~\ref{alg:FW_product_SVM}, the batch Frank-Wolfe method (\emph{FW})\footnote{This is equivalent to the batch subgradient method with an adaptive step-size, as mentioned in Section~\ref{sec:FW_SVM}.} and our novel block-coordinate Frank-Wolfe method (\emph{BCFW}), to 
the \emph{cutting plane} algorithm implemented in SVMstruct~\citep{Joachims:2009ex} with its default options,
the online exponentiated gradient (\emph{online-EG}) method of~\citet{Collins2008}, and
the stochastic subgradient method (\emph{SSG}) with step-size chosen as in the `Pegasos' version of~\citet{ShalevShwartz:2010cg}. We also include the weighted average $\bar{\weightv}^{(k)}:=\frac{2}{k(k+1)} \sum_{t=1}^k t \weightv^{(t)}$ of the iterates from \emph{SSG} (called \emph{SSG-wavg}) which was recently shown to converge at the faster rate of $O(1/k)$ instead of $O\left((\log{k})/k\right)$ \citep{LacosteJulien:2012uo, Shamir:2013vw}. 
Analogously, we average the iterates from \emph{BCFW} the same way to obtain the \emph{BCFW-wavg} method (implemented efficiently with the optional line in Algorithm~\ref{alg:FW_product_SVM}), which also has a provable $O(1/k)$ convergence rate (Theorem~\ref{thm:primalDualGreedy1Regime}).
The performance of the different algorithms according to several criteria is
visualized in Figure~\ref{fig:results}. The results are discussed in the
caption, while additional experiments can be found in
Appendix~\ref{sec:additional_experiments}. In most of the experiments, %
the \emph{BCFW-wavg} method dominates all competitors. The
superiority is especially striking for the first few iterations, and when
using a small regularization strength $\lambda$, which is often needed in practice. In term of test error, a peculiar observation is that the weighted average of the iterates seems to help both methods significantly: $\emph{SSG-wavg}$ sometimes slightly outperforms $\emph{BCFW-wavg}$ despite having the worst objective value amongst all methods. This phenomenon is worth further investigation.

\section{Related Work}
\label{sec:related}

\begin{table*}[t]
\caption{\small{Convergence rates given in the \emph{number of calls to the oracles} for different optimization algorithms for the structural SVM objective~\eqref{eq:svmstruct_nslack_primal} in the case of a Markov random field structure, to reach a specific accuracy $\varepsilon$ measured for different types of gaps, in term of the number of training examples $n$, regularization parameter $\regularizerweight$, size of the label space~$|\outputdomain|$, maximum feature norm $R:= \max_{i,\outputvarv } \norm{\featuremapdiffv_i(\outputvarv)}_2$ (some minor terms were ignored for succinctness). Table inspired from~\citep{Zhang:2011:ATM3net}.
Notice that only stochastic subgradient and our proposed algorithm have rates independent of $n$.
\vspace{2mm}
}}
\label{tab:rates}
\centering
\resizebox{0.95\textwidth}{!}{
\begin{tabular}{m{4.65cm}|cccc|c}
\hline
\textbf{Optimization algorithm} & \textbf{\!Online} & \textbf{\!Primal/Dual} & \textbf{Type of guarantee} & \textbf{Oracle type} & \textbf{\# Oracle calls\!\!} \\ \hline\hline
dual extragradient \citep{Taskar06extrag} & no & primal-`dual' & saddle point gap & \!\!\!\!Bregman projection & $O\left(\frac{n R \log |\outputdomain|}{\regularizerweight \varepsilon}\right)$ \\ \hline
online exponentiated gradient \citep{Collins2008}  & yes & dual & expected dual error & expectation &  $O\left(\frac{(n + \log |\outputdomain|) R^2 }{\regularizerweight \varepsilon}\right)$ \\ \hline
excessive gap reduction \citep{Zhang:2011:ATM3net}  & no & primal-dual & duality gap & expectation &  $O\left(n R \sqrt{\frac{\log |\outputdomain|}{\regularizerweight \varepsilon}}\right)$ \\ \hline
BMRM \citep{Teo:2010bundle} & no & primal & $\geq$primal error & maximization & $O\left(\frac{n R^2}{\regularizerweight \varepsilon}\right)$ \\ \hline
1-slack SVM-Struct \citep{Joachims:2009ex} & no & primal-dual & duality gap & maximization & $O\left(\frac{n R^2}{\regularizerweight \varepsilon}\right)$ \\ \hline
stochastic \! subgradient \citep{ShalevShwartz:2010cg} & yes & primal & primal error w.h.p. & maximization & $\tilde{O}\left(\frac{R^2}{\regularizerweight \varepsilon}\right)$ \\ \hline
this paper: block-coordinate Frank-Wolfe & yes & primal-dual & expected duality gap &  maximization & $O\left(\frac{R^2}{\regularizerweight \varepsilon}\right)$ Thm.~\ref{thm:convergence_FW_product_SVM} \\ \hline
\end{tabular}
}
\vspace{-1em}
\end{table*}

There has been substantial work on dual coordinate descent for SVMs,
including the original sequential minimal optimization (SMO) algorithm. The SMO algorithm
was generalized to structural SVMs~\citep[Chapter 6]{taskar04thesis},
but its convergence
rate scales badly with the size of the output space: it was estimated as
$O\left(n |\outputdomain|/\regularizerweight \varepsilon\right)$ in~\citet{Zhang:2011:ATM3net}.
Further, this method requires an expectation oracle to work with its factored dual parameterization.
As in our algorithm,~\citet{Rousu:2006vv} propose updating one training example at a time, but using multiple Frank-Wolfe updates to optimize along the subspace. However,
they do not obtain any rate guarantees and their algorithm is less general because it again requires an expectation oracle.
In
the degenerate \emph{binary} SVM case, our block-coordinate Frank-Wolfe algorithm is actually equivalent to the method of~\citet{Hsieh:2008bd},
where because each datapoint has a unique dual variable, exact coordinate optimization can be accomplished by the line-search step of our algorithm.
\citet{Hsieh:2008bd} show a local linear convergence rate
in the dual, and our results complement theirs by providing a global \emph{primal} convergence guarantee for their
algorithm of $O\left(1/ \varepsilon\right)$. 
After our paper had appeared on arXiv, \citet{ShalevShwartz:2012tn} have proposed a generalization of dual coordinate descent applicable to several regularized losses, including the structural SVM objective. Despite being motivated from a different perspective, a version of their algorithm (Option II of Figure 1) gives the exact same step-size and update direction as \emph{BCFW} with line-search, and their Corollary 3 gives a similar convergence rate as our Theorem~\ref{thm:convergence_FW_product_SVM}. \citet{balamurugan11SDM} %
propose to approximately solve a quadratic problem on each example using \emph{SMO}, but they do not provide any rate guarantees.
The \emph{online-EG} method implements a variant of dual coordinate descent, but it requires an expectation oracle
and~\citet{Collins2008} estimate its primal convergence at only~$O\left(1/\varepsilon^2\right)$.

Besides coordinate descent methods, a variety of other algorithms have been proposed for structural SVMs. We summarize a few of the most popular in Table~\ref{tab:rates}, with their convergence rates quoted in number of oracle calls to reach an accuracy of $\varepsilon$. However, we note that almost no guarantees are given for the optimization of
structural SVMs with approximate oracles.
A regret analysis in the context of online optimization was considered by~\citet{Ratliff:2007subgradient}, but they do not analyze the effect of this on solving the optimization problem.
The cutting plane algorithm of~\citet{Tsochantaridis2005} was considered with approximate maximization by~\citet{finley08svmstruct-app}, though the dependence of the running time on the the approximation error was left unclear. In contrast, we provide guarantees for batch subgradient, cutting plane, and block-coordinate Frank-Wolfe, for achieving an $\varepsilon$-approximate solution as long as the error of the oracle is appropriately bounded. %

\vspace{-2mm}
\section{Discussion}

This work proposes a novel randomized block-coordinate generalization of the classic Frank-Wolfe algorithm for optimization with block-separable constraints. Despite its potentially much lower iteration cost, the new algorithm achieves a similar convergence rate in the duality gap as the full Frank-Wolfe method. For the dual structural SVM optimization problem, it leads to a simple online algorithm that  yields a solution to an issue that is notoriously difficult to address for stochastic algorithms: no step-size sequence needs to be tuned since the optimal step-size can be efficiently computed in closed-form. Further, at the cost of an additional pass through the data %
(which could be done alongside a full Frank-Wolfe iteration), it allows us to compute a duality gap guarantee that can be used to decide when to terminate the algorithm. Our experiments indicate that empirically it converges faster than other stochastic algorithms for the structural SVM problem, especially in the realistic setting where only a few passes through the data are possible.

Although our structural SVM experiments use an exact maximization oracle, the duality gap guarantees, the optimal step-size, 
and a computable bound on the duality gap are all still available when only an appropriate approximate maximization oracle is used. Finally, %
although the structural SVM problem is what motivated this work, we expect that the block-coordinate Frank-Wolfe algorithm may be useful for other problems in machine learning where a complex objective with block-separable constraints arises.
\vspace{-3mm}
\paragraph{Acknowledgements.}
We thank
Francis Bach, %
Bernd G{\"a}rtner and %
Ronny Luss %
for helpful discussions, and
Robert Carnecky for the 3D illustration. %
MJ is supported by the ERC Project SIPA, and by the Swiss National Science Foundation. %
SLJ and MS are partly supported by the ERC (SIERRA-ERC-239993). 
SLJ is supported by a Research in Paris fellowship. MS is supported 
by a NSERC postdoctoral fellowship. 

{\small
\bibliography{references}
\bibliographystyle{icml2013}
}

\clearpage
\appendix

\numberwithin{definition}{section}
\numberwithin{algorithm}{section}

\onecolumn

\icmltitle{Supplementary Material\\
Block-Coordinate Frank-Wolfe Optimization for Structural SVMs}
\vskip 0.3in

\paragraph{Outline.}
In Appendix \ref{sec:app_curvature}, we discuss the curvature constants and compute them for the structural SVM problem.
In Appendix \ref{sec:app_SVM_algos}, we give additional details on applying the Frank-Wolfe algorithms to the structural SVM and provide proofs for Theorems \ref{thm:convergence_FW_SVM} and \ref{thm:convergence_FW_product_SVM}.
In the main Appendix \ref{sec:app_convergence_proof}, we give a self-contained presentation and analysis of the new block-coordinate Frank-Wolfe method (Algorithm~\ref{alg:FW_product}), and prove the main convergence Theorem~\ref{thm:convergence_FW_product}.
In Appendix \ref{sec:Fenchel_gap_equivalence}, the `linearization'-duality gap is interpreted in terms of Fenchel duality. For completeness, we include a short derivation of the dual problem to the structural SVM in Appendix \ref{sec:app_duals}. Finally, we present in Appendix~\ref{sec:additional_experiments} additional experimental results as well as more detailed information about the implementation.

\section{The Curvature Constants $\Cf$ and $\CfTotal$}\label{sec:app_curvature}

\paragraph{The Curvature Constant $\Cf$.}
The \emph{curvature constant} $\Cf$ is given by the maximum relative deviation of the objective function $f$ from its linear approximations, over the domain~$\domain$~\citep{Clarkson:2010hv,Jaggi:2013wg}.
Formally,
\begin{equation}\label{eq:Cf}
  \Cf := \sup_{\substack{\x,\sv\in \domain, \\
                      \stepsize\in[0,1],\\
                      \y = \x+\stepsize(\sv-\x)}}
          \frac{2}{\stepsize^2}\left( f(\y)-f(\x)-\langle \y-\x, \nabla f(\x)\rangle \right) \ .
\end{equation}

The assumption of bounded $\Cf$ corresponds to a slightly weaker, affine invariant form of a \emph{smoothness} assumption on $f$.
It is known that $\Cf$ is upper bounded by the Lipschitz constant of the
gradient $\nabla f$ times the squared diameter of $\domain$, for any arbitrary choice of a norm~\citep[Lemma 8]{Jaggi:2013wg}; %
but it can also be much smaller (in particular, when the dimension of the affine hull of $\domain$ is smaller than the ambient space), so it is a more fundamental quantity in the analysis of the Frank-Wolfe algorithm than the Lipschitz constant of the gradient. As pointed out by~\citet[Section 2.4]{Jaggi:2013wg}, $\Cf$ is invariant under affine transformations, as is the Frank-Wolfe algorithm.

\paragraph{The Product Curvature Constant $\CfTotal$.}
The curvature concept can be generalized to our setting of product domains $\domain := \domain^{(1)}\times\mathellipsis\times\domain^{(n)}$ as follows: over each individual coordinate block, the curvature is given by
\begin{equation}\label{eq:CfBlock}
  \Cf^{(i)} := \sup_{\substack{\x\in \domain,\,\sv_{(i)}\in \domain^{(i)}, \\
                      \stepsize\in[0,1],\\ %
                      \y = \x+\stepsize(\sv_{[i]}-\x_{[i]})}}
          \frac{2}{\stepsize^2}\left( f(\y)-f(\x)-\langle \y_{(i)}-\x_{(i)}, \nabla_{\!(i)} f(\x)\rangle \right) \ ,
\end{equation}
where the notation $\x_{[i]}$ refers to the zero-padding of $\x_{(i)}$ so that $\x_{[i]} \in \domain$. By considering the Taylor expansion of $f$, it is not hard to see that also the `partial' curvature $\Cf^{(i)}$ is upper bounded by the Lipschitz constant of the partial gradient $\nabla_{\!(i)} f$ times the squared diameter of just one domain block $\domain^{(i)}$. %
See also the proof of Lemma \ref{lem:curvature_FW_SVM_product} below.

We define the global \emph{product curvature constant} as the sum of these curvatures for each block, i.e.
\begin{equation}\label{eq:CfProductGlobal}
  \CfTotal := \sum_{i=1}^n \Cf^{(i)}
\end{equation}
Observe that for the classical Frank-Wolfe case when $n=1$, we recover the original curvature constant.

\paragraph{Computing the Curvature Constant $\Cf$ in the SVM Case.}
\begin{lemma}\label{lem:curvature_FW_SVM}
For the dual structural SVM objective function~\eqref{eq:svmstruct_nslack_dual} over the domain $\domain := \simplex_{|\outputdomain_1|}\times\mathellipsis\times\simplex_{|\outputdomain_n|}$, the curvature constant $\Cf$,  as defined in \eqref{eq:Cf}, is upper bounded by
\[
\Cf \le \frac{4R^2}{\lambda} \ ,
\]
where $R$ is the maximal length of a difference feature vector, i.e. $R:= \displaystyle\max_{i\in[n],\outputvarv \in \outputdomain_i}\norm{\featuremapdiffv_i(\outputvarv)}_2$~.
\end{lemma}
\begin{proof}[Proof of Lemma \ref{lem:curvature_FW_SVM}]
If the objective function is twice differentiable, we can plug-in the second degree
Taylor expansion of~$f$ into the above definition~\eqref{eq:Cf} of the
curvature, see e.g.  \citep[Inequality (2.12)]{Jaggi:2011ux} or \citep[Section 4.1]{Clarkson:2010hv}. %
In our case, the gradient at $\dualvarv$ is given by $\lambda A^TA\dualvarv - \bv$, so that the Hessian is $\lambda A^TA$, being a constant matrix independent of $\dualvarv$. This gives the following upper bound\footnote{Because our function is a quadratic function, this is actually an equality.} on~$\Cf$, which we can separate into two identical matrix-vector products with our matrix~$A$:
\[
\begin{split}
\Cf \le& \sup_{\substack{\x,\y \in \domain, \\
                    \zz \in [\x,\y] \subseteq \domain}}
            (\y-\x)^T \nabla^2 f(\zz) (\y-\x)  \\
  =& \  \lambda \cdot \sup_{\x,\y \in \domain}
            (A(\y-\x))^T A(\y-\x) \\
  =& \  \lambda \cdot \sup_{\vv,\ww \in A\domain}
            \norm{\vv-\ww}_2^2
        \ \leq  \lambda\cdot \sup_{\vv \in A\domain}
            \norm{2\vv}_2^2
\end{split}
\]
By definition of our compact domain $\domain$, we have that each vector $\vv\in A\domain$ is precisely the sum of $n$ vectors, each of these being a convex combination of the feature vectors for the possible labelings for datapoint $i$.

Therefore, the norm $\norm{\vv}_2$ is upper bounded by $n$ times the longest column of the matrix $A$, or more formally $\norm{\vv}_2 \le n \frac1{\lambda n} R$ with $R$ being the longest\footnote{%
This choice of the radius $R$ then gives $\frac1{\lambda n} R
= \max_{i\in[n],\outputvarv \in \outputdomain_i} \norm{ \frac1{\lambda n} \featuremapdiffv_i(\outputvarv)}_2
= \max_{i\in[n],\outputvarv \in \outputdomain_i} \norm{ A_{(i,\outputvarv)}} $.%
} feature vector, i.e.
\[
R:= \max_{i\in[n],\outputvarv \in \outputdomain_i}\norm{\featuremapdiffv_i(\outputvarv)}_2 \ .
\]
Altogether, we have obtained that the curvature $\Cf$ is upper bounded by $\frac{4R^2}{\lambda}$.

We also note that in the worst case, this bound is tight. For example, we can make $\Cf = \frac{4R^2}{\lambda}$ by having for each datapoint $i$, two labelings which give opposite difference feature vectors $\psi_i$ of the same maximal norm $R$.
\end{proof}

\paragraph{Computing the Product Curvature Constant $\CfTotal$ in the SVM Case.}
\begin{lemma}\label{lem:curvature_FW_SVM_product}
For the dual structural SVM objective function~\eqref{eq:svmstruct_nslack_dual} over the domain $\domain := \simplex_{|\outputdomain_1|}\times\mathellipsis\times\simplex_{|\outputdomain_n|}$, the total curvature constant $\CfTotal$ on the product domain $\domain$, as defined in \eqref{eq:CfProductGlobal}, is upper bounded by
\[
\CfTotal \le \frac{4R^2}{\lambda n} \
\]
where $R$ is the maximal length of a difference feature vector, i.e. $R:= \displaystyle\max_{i\in[n],\outputvarv \in \outputdomain_i}\norm{\featuremapdiffv_i(\outputvarv)}_2$~.
\end{lemma}
\begin{proof}
We follow the same lines as in the above proof of Lemma \ref{lem:curvature_FW_SVM}, but now applying the same bound to the block-wise definition \eqref{eq:CfBlock} of the curvature on the $i$-th block.
Here, the change from $\x$ to $\y$ is now restricted to only affect the coordinates in the $i$-th block $\domain^{(i)}$. To simplify the notation, let $\domain^{[i]}$ be $\domain^{(i)}$ augmented with the zero domain for all the other blocks -- i.e. the analog of $\x_{(i)} \in \domain^{(i)}$ is $\x_{[i]} \in \domain^{[i]}$. $\x_{(i)}$ is the $i$-th block of $\x$ whereas $\x_{[i]} \in \domain$ is $\x_{(i)}$ padded with zeros for all the other blocks. We thus require that $\y-\x \in \domain^{[i]}$ for a valid change from $\x$ to $\y$. Again by the degree-two Taylor expansion, we obtain
\[
\begin{split}
\Cf^{(i)} \le& \sup_{\substack{\x,\y \in \domain, \\
                    (\y-\x) \in \domain^{[i]} \\
                    \zz \in [\x,\y] \subseteq \domain}}
            (\y-\x)^T \nabla^2 f(\zz) (\y-\x)  \\
  =& \  \lambda \cdot \sup_{\substack{\x,\y \in \domain\\
                    (\y-\x) \in \domain^{[i]}}}
            (A(\y-\x))^T A(\y-\x) \\
  =& \  \lambda \cdot \sup_{\vv,\ww \in A\domain^{(i)}}
            \norm{\vv-\ww}_2^2
        \ \leq  \lambda\cdot \sup_{\vv \in A\domain^{(i)}}
            \norm{2\vv}_2^2
\end{split}
\]
In other words, by definition of our compact domain $\domain^{(i)}=\simplex_{|\outputdomain_i|}$, we have that each vector $\vv\in A\domain^{(i)}$ %
 is a convex combination of the feature vectors corresponding to the possible labelings for datapoint $i$.
Therefore, the norm $\norm{\vv}_2$ is again upper bounded by the longest column of the matrix~$A$, which means $\norm{\vv}_2 \le \frac1{\lambda n} R$ with $R:= \max_{i\in[n],\outputvarv \in \outputdomain_i}\norm{\featuremapdiffv_i(\outputvarv)}_2$.
Summing up over the $n$ blocks $\domain^{(i)}$, we obtain that the product curvature $\CfTotal$ is upper bounded by $\frac{4R^2}{\lambda n}$.

For the same argument as at the end of the proof for Lemma~\ref{lem:curvature_FW_SVM}, this bound is actually tight in the worst case.
\end{proof}

\section{More Details on the Algorithms for Structural SVMs}\label{sec:app_SVM_algos}

\subsection{Equivalence of an Exact Frank-Wolfe Step and Loss-Augmented Decoding} \label{ssec:FW_equivalent_decoding}

To see that the proposed Algorithm~\ref{alg:FW_SVM} indeed exactly corresponds to the standard Frank-Wolfe Algorithm~\ref{alg:FW} applied to the SVM dual problem \eqref{eq:svmstruct_nslack_dual}, we verify that the search direction $\sv$ giving the update $\weightv_{\sv} = A\sv$ is in fact an exact Frank-Wolfe step, which can be seen as follows:
\begin{lemma}\label{lem:condGradient}
The sparse vector $\sv\in\R^n$ constructed in the inner for-loop of Algorithm \ref{alg:FW_SVM}
is an exact solution to $\sv = \argmin_{\sv'\in \domain} \left\langle \sv', \nabla f(\dualvarv^{(k)}) \right\rangle$ for optimization problem~\eqref{eq:svmstruct_nslack_dual}.
\end{lemma}
\begin{proof}
Over the product domain $\domain =
\simplex_{|\outputdomain_1|}\times\mathellipsis\times\simplex_{|\outputdomain_n|}$,
the minimization $\min_{\sv' \in \domain} \langle \sv', \nabla f(\dualvarv)
\rangle$ decomposes as $\sum_i \min_{\bm{s}_i \in \simplex_{|\outputdomain_i|} }
\langle \bm{s}_i, \nabla_i f(\dualvarv) \rangle$. %
The minimization of a linear function over the simplex reduces to a search over its corners -- in this case, it amounts for each $i$ to find the minimal component of $-H_i(\outputvarv;\weightv)$ over $\outputvarv\in\outputdomain_i$, i.e. solving the loss-augmented decoding problem as used in Algorithm \ref{alg:FW_SVM} to construct the domain vertex $\sv$. To see this, note that for our choice of primal variables $\weightv = A\dualvarv$, the gradient of the dual objective, $\nabla f(\dualvarv) = \lambda A^TA\dualvarv - \bv$, writes as $\lambda A^T\weightv - \bv$. This vector is precisely the loss-augmented decoding function $-\frac1n H_i(\outputvarv;\weightv)$, for $i\in[n],\,\outputvarv\in\outputdomain_i$, as defined in~\eqref{eq:subproblem_loss_augm}.
\end{proof}

\subsection{Relation between the Lagrange Duality Gap and the `Linearization' Gap for the Structural SVM}\label{sec:gap_comparison}
We show here that the simple `linearization' gap~\eqref{eq:duality_gap}, evaluated on the structural SVM dual problem~\eqref{eq:svmstruct_nslack_dual} is actually equivalent to the standard Lagrangian duality gap for the structural SVM primal objective~\eqref{eq:svmstruct_nslack_primal} (these two duality gaps are not the same in general\footnote{For example, the two gaps are different when evaluated on the dual of the conditional random field objective (see, for example, \citet{Collins2008} for the formulation), which does not have a Lipschitz continuous gradient.}). This is important for the duality gap convergence rate results of our Frank-Wolfe algorithms to be transferable as primal convergence rates on the original structural SVM objective~\eqref{eq:svmstruct_nslack_primal_nonsmooth}, which is the one with statistical meaning (for example with generalization error bounds as given in~\citet{Taskar2003}).

\begin{proof}
So consider the difference of our objective at $\weightv := A\dualvarv$ in the primal problem~\eqref{eq:svmstruct_nslack_primal_nonsmooth}, and the dual objective at $\dualvarv$ in problem~\eqref{eq:svmstruct_nslack_dual} (in the maximization version). This difference is
\begin{eqnarray*}
g_{\text{\tiny Lagrange}}(\weightv,\dualvarv)&=&\frac\lambda2 \weightv^T\weightv + \frac1n \sum_{i=1}^n \tilde{H}_i(\weightv)
- \left(\bv^T\dualvarv - \frac\lambda2 \weightv^T\weightv \right) \\
&=& \lambda \weightv^T\weightv - \bv^T\dualvarv + \frac1n \sum_{i=1}^n \max_{\outputvarv\in\outputdomain_i} H_i(\outputvarv;\weightv) \ .
\end{eqnarray*}
Now recall that by the definition of $A$ and $\bv$, we have that $\frac{1}{n} H_i(\outputvarv;\weightv) = (\bv-\lambda A^T\weightv)_{(i,\outputvarv)} = (-\nabla f(\dualvarv))_{(i,\outputvarv)}$. By summing up over all points and re-using a similar argument as in Lemma~\ref{lem:condGradient} above, we get that
\[
\frac1n \sum_{i=1}^n \max_{\outputvarv\in\outputdomain_i} H_i(\outputvarv;\weightv)
= \sum_{i=1}^n \max_{\outputvarv\in\outputdomain_i} (-\nabla f(\dualvarv))_{(i,\outputvarv)} = \max_{\sv'\in \domain} \left\langle \sv', -\nabla f(\dualvarv) \right\rangle \ ,
\]
\begin{eqnarray*}
g_{\text{\tiny Lagrange}}(\weightv,\dualvarv)&=& (\lambda \weightv^T A - \bv^T) \dualvarv + \frac1n \sum_{i=1}^n \max_{\outputvarv\in\outputdomain_i} H_i(\outputvarv;\weightv) \\
	&=& \left\langle  \nabla f(\dualvarv), \dualvarv \right\rangle + \max_{\sv'\in \domain} \left\langle -\sv', \nabla f(\dualvarv) \right\rangle 
	\;=\; \left\langle \dualvarv - \sv, \nabla f(\dualvarv)\right\rangle 
	\;=\; g(\dualvarv) \ ,
\end{eqnarray*}
as defined in~\eqref{eq:duality_gap}.
\end{proof}

\subsection{Convergence Analysis} \label{ssec:app_SVM_proofs}
\subsubsection{Convergence of the Batch Frank-Wolfe Algorithm~\ref{alg:FW_SVM} on the Structural SVM Dual}
\begin{reptheorem}{thm:convergence_FW_SVM}%
Algorithm~\ref{alg:FW_SVM} obtains an $\varepsilon$-approximate solution to the structural SVM dual problem~\eqref{eq:svmstruct_nslack_dual} and duality gap $g(\dualvarv^{(k)}) \le\varepsilon$ after at most $O\left(\frac{R^2}{\lambda\varepsilon}\right)$ iterations, where each iteration costs $n$ oracle calls.
\end{reptheorem}
\begin{proof}
We apply the known convergence results for the standard Frank-Wolfe
Algorithm~\ref{alg:FW}, as given e.g. in~\citep{Frank:1956vp,Dunn:1978di,Jaggi:2013wg}, or as given in the paragraph just after the proof of Theorem~\ref{thm:primalGreedyProduct}:
For each $k\ge 1$, the iterate $\dualvarv^{(k)}$ of Algorithm~\ref{alg:FW} (either using the predefined step-sizes, or using line-search) satisfies
$
\E[ f(\dualvarv^{(k)})] - f(\dualvarv^*) \le \frac{2 \Cf}{k+2} \ ,
$
where $\dualvarv^*\in \domain$ is an optimal solution to problem~\eqref{eq:svmstruct_nslack_dual}.

Furthermore, if Algorithm~\ref{alg:FW} is run for $K\ge 1$ iterations, then it has an iterate $\dualvarv^{(\hat k)}$, $1\le \hat k\le K$, with duality gap bounded by
$
\E [g(\dualvarv^{(\hat k)})] \le \frac{6 \Cf}{K+1}
$.
This was shown e.g. in \citep{Jaggi:2013wg} with slightly different constants, or also in our analysis presented below (see the paragraph after the generalized analysis provided in Theorem~\ref{thm:primalDualGreedy1Regime}, %
when the number of blocks $n$ is set to one).

Now for the SVM problem and the equivalent Algorithm~\ref{alg:FW_SVM}, the claim follows from the curvature bound $\Cf \le \frac{4R^2}{\lambda}$ for the dual structural SVM objective function~\eqref{eq:svmstruct_nslack_dual} over the domain $\domain := \simplex_{|\outputdomain_1|}\times\mathellipsis\times\simplex_{|\outputdomain_n|}$, as given in the above Lemma \ref{lem:curvature_FW_SVM}.%
\end{proof}

\subsubsection{Convergence of the Block-Coordinate Frank-Wolfe Algorithm~\ref{alg:FW_product_SVM} on the Structural SVM Dual}

\begin{reptheorem}{thm:convergence_FW_product_SVM}%
If $\Lmax \le \frac{4R^2}{\lambda n}$ (so $h_0 \leq \frac{4R^2}{\lambda n}$), then Algorithm~\ref{alg:FW_product_SVM} obtains an $\varepsilon$-approximate solution to the structural SVM dual problem~\eqref{eq:svmstruct_nslack_dual} and expected duality gap $\E [g(\dualvarv^{(k)})] \le\varepsilon$ after at most $O\left(\frac{R^2}{\lambda\varepsilon}\right)$ iterations, where each iteration costs a single oracle call.

If $\Lmax > \frac{4R^2}{\lambda n}$, then it requires at most an additional (constant in $\varepsilon$) number of $O\left(n \log\left( \frac{\lambda n \Lmax}{R^2}\right) \right)$ %
steps to get the same error and duality gap guarantees, whereas the predefined step-size variant will require an additional $O\left(\frac{n \Lmax}{\varepsilon}\right)$ steps.
\end{reptheorem}
\begin{proof}

Writing $h_0=f(\dualvarv^{(0)}) - f(\dualvarv^*)$ for the error at the starting point used by the algorithm, the convergence Theorem~\ref{thm:convergence_FW_product} states that if $k\ge0$ and $k\ge \frac{2n}{\varepsilon}(\CfTotal+h_0)$, then the expected error is $\E [f(\dualvarv^{(k)})]-f(\dualvarv^*)\le\varepsilon$ and analogously for the expected duality gap. The result then follows by plugging in the curvature bound $\CfTotal \le \frac{4R^2}{\lambda n}$ for the dual structural SVM objective function~\eqref{eq:svmstruct_nslack_dual} over the domain $\domain := \simplex_{|\outputdomain_1|}\times\mathellipsis\times\simplex_{|\outputdomain_n|}$, as detailed in Lemma \ref{lem:curvature_FW_SVM_product} %
(notice that it is $n$ times smaller than the curvature $\Cf$ needed for the batch algorithm) and then bounding $h_0$. To bound $h_0$, we observe that by the choice of the starting point $\dualvarv^{(0)}$ using only the observed labels, the initial error is bounded as $h_0 \le g(\dualvarv^{(0)}) = \bv^T \sv = \frac1n \sum_{i=1}^n \max_{\outputvarv\in\outputdomain_i}\errorterm_i(\outputvarv) \leq \Lmax$. Thus, if $\Lmax \leq \frac{4R^2}{\lambda n}$, then we have $\CfTotal + h_0 \leq \frac{8R^2}{\lambda n}$, which proves the first part of the theorem.

In the case $\Lmax > \frac{4R^2}{\lambda n}$, then the predefined step-size variant will require an additional $\frac{2n h_0}{\varepsilon} \leq \frac{2n \Lmax}{\varepsilon}$ steps as we couldn't use the fact that $h_0  \leq \CfTotal$. 
For the line-search variant, on the other hand, we can use the improved convergence Theorem~\ref{thm:fast_warmup}, %
which shows that the algorithm require at most $k_0 \leq n \log( h_0/\CfTotal)$ steps to reach the condition $h_0  \leq \CfTotal$; once this condition is satisfied, we can simply re-use Theorem~\ref{thm:convergence_FW_product} with $k$ redefined as $k - k_0$ to get the final convergence rates. We also point out that the statement of Theorem~\ref{thm:fast_warmup} stays valid by replacing~$\CfTotal$ with any $\CfTotal{'} \geq \CfTotal$ in it. So plugging in $\CfTotal{'} = \frac{R^2}{\lambda n}$ and the bound $h_0 \leq \Lmax$ in the $k_0$ quantity gives back the number of additional steps mentioned in the second part of the theorem statement an $\varepsilon$-approximate solution. 
A similar argument can be made for the expected duality gap by using the improved convergence Theorem~\ref{thm:primalDualFaster}, which simply adds the requirement $K \geq 5 k_0$.
\end{proof}

We note that the condition $\Lmax \le \frac{4R^2}{\lambda n}$ is not necessarily too restrictive in the case of the structural SVM setup. In particular, the typical range of $\lambda$ which is needed for a problem is around $O(1/n)$
-- and so the condition becomes $\Lmax \leq 4R^2$ which is typically satisfied when the loss function is normalized.

\subsection{Implementation}\label{ssec:app_block_implement}
We comment on three practical implementation aspects of Algorithm \ref{alg:FW_product_SVM} on large structural SVM problems:

\paragraph{Memory.} For each datapoint $i$, our Algorithm \ref{alg:FW_product_SVM} stores an additional vector $\weightv_i\in\R^d$ holding the contribution of its corresponding dual variables $\dualvarv_{(i)}$ to the primal vector $\weightv = A\dualvarv$, i.e. $\weightv_i = A\dualvarv_{[i]}$, where $\dualvarv_{[i]}$ is $\dualvarv_{(i)}$ padded with zeros so that $\dualvarv_{[i]} \in \R^m$ and $\dualvarv = \sum_i \dualvarv_{[i]}$. This means the algorithm needs more memory than the direct (or batch) Frank-Wolfe structural SVM Algorithm~\ref{alg:FW_SVM}, but the additional memory can sometimes be bounded by a constant times the size of the input data itself.
In particular, in the case that the feature vectors $\featuremapdiffv_i(\outputvarv)$ are sparse, we can sometimes get the same improvement in memory requirements for $\weightv_i$, since for fixed $i$, all vectors $\featuremapdiffv_i(\outputvarv)$ usually have the same sparsity pattern. On the other hand, if the feature vectors are not sparse, it might be more efficient to only work with the dual variables instead of the primal variables (see the kernelized version in Appendix~\ref{ssec:app_kernelized} for more details).

\paragraph{Duality Gap as a Stopping Criterion.} 
Analogous as in the `classical Frank-Wolfe' structural SVM Algorithm \ref{alg:FW_SVM} explained in Section~\ref{sec:FW_SVM}, we would again like to use the duality gap $g(\dualvarv^{(k)})\le\varepsilon$ as the stopping criterion for the faster Algorithm \ref{alg:FW_product_SVM}.
Unfortunately, since now in every step we only update a single one of the many blocks, such a single direction $\sv_{(i)}$ will only determine the partial gap~$g^{(i)}(\dualvarv^{(k)})$ in the $i$-th block, but not the full information needed to know the total gap $g(\dualvarv^{(k)})$.
Instead, to compute the total gap, a single complete (batch) pass through all datapoints as in Algorithm~\ref{alg:FW_SVM} is necessary, to obtain a full linear minimizer $\sv\in\domain$. For efficiency reason, we could therefore compute the duality gap every say $Nn$ iterations for some constant $N>1$. Then stopping as soon as $g(\dualvarv^{(k)}) = g(\weightv^{(k)},\ell^{(k)},\weightv_{\sv},\ell_{\sv}) \le \varepsilon$ will not affect our convergence results.

\paragraph{Line-Search.} To compute the line-search step-size for Frank-Wolfe on the structural SVM, we recall that the analytic formula was given by $\stepsize_{opt} := \frac{\langle \dualvarv - \sv, \nabla f(\dualvarv) \rangle}{\lambda\norm{A(\dualvarv-\sv)}^2}$, and finally taking $\stepsize_{LS} := \max\left\{0,\min\left\{1,\stepsize_{opt} \right\} \right\}$. This is valid for any $\sv \in \domain$. For the block-coordinate Frank-Wolfe Algorithm~\ref{alg:FW_product_SVM}, $\sv$ is equal to $\dualvarv$ for all blocks, except for the $i$-th block -- this means that $\dualvarv - \sv = \dualvarv_{[i]} - \sv_{[i]}$, i.e. is zero everywhere except on the $i$-th block. By recalling that $\weightv_i = A\dualvarv_{[i]}$ is the individual contribution to $\weightv$ from $\dualvarv_{(i)}$ which is stored during the algorithm, we see that the denominator thus becomes $\lambda\norm{A(\dualvarv-\sv)}^2 = \lambda\norm{\weightv_i-\weight_{\sv}}^2$. The numerator is $\langle \dualvarv - \sv, \nabla f(\dualvarv) \rangle = (\dualvarv - \sv)^T(\lambda A^TA\dualvarv - \bv) = \lambda ( \weightv_i-\weight_{\sv}) ^T \weightv - \ell_i + \ell_{\sv}$, where as before $\ell_i = \bv^T \dualvarv_{[i]}$ is maintained during Algorithm~\ref{alg:FW_product_SVM} and so the line-search step-size can be computed efficiently. We mention in passing that when $\sv_{(i)}$ is the exact minimizer of the linear subproblem on $\domain^{(i)}$, then the numerator is actually a duality gap component $g^{(i)}(\dualvarv)$ as defined in~\eqref{eqn:gapProduct} -- the total duality gap then is $g(\dualvarv) = \sum_i g^{(i)}(\dualvarv)$ which can only be computed if we do a batch pass over all the datapoints, as explained in the previous paragraph.

\subsection{More details on the Kernelized Algorithm}\label{ssec:app_kernelized}

Both Algorithms~\ref{alg:FW_SVM} and \ref{alg:FW_product_SVM} can be used with kernels by explicitly maintaining the sparse dual variables $\dualvarv^{(k)}$ instead of the primal variables $\weightv^{(k)}$. In this case, the classifier is only given implicitly as a sparse combination of the corresponding kernel functions, i.e. $\weightv = A\dualvarv$, where $\featuremapdiffv_i(\outputvarv) = k(\inputvarv_i,\outputvarv_i; \cdot, \cdot) - k(\inputvarv_i,\outputvarv; \cdot, \cdot)$ for a structured kernel $k : (\inputdomain \times \outputdomain) \times (\inputdomain \times \outputdomain) \rightarrow \R$. Note that the number of non-zero dual variables is upper-bounded by the number of iterations, and so the time to take dot products grows quadratically in the number of iterations.

\begin{algorithm}[h!]
    \caption{Kernelized Dual Block-Coordinate Frank-Wolfe for Structural SVM}
    \label{alg:FW_product_SVM_kernel}
\begin{algorithmic}
         \STATE Let $\dualvarv^{(0)} := (\unit^{\outputvarv_1},\dots,\unit^{\outputvarv_n}) \in \domain =  \simplex_{|\outputdomain_1|}\times\mathellipsis\times\simplex_{|\outputdomain_n|}$ and $\bar{\dualvarv}^{(0)} = \dualvarv^{(0)}$ 
    \FOR{$k=0\dots K$}
                 \STATE  Pick $i$ uniformly at random in $\{1, \ldots, n\}$
                 \STATE Solve $\outputvarv_i^* := \displaystyle\argmax_{\outputvarv\in\outputdomain_i} \ H_i(\outputvarv;A\dualvarv^{(k)})$
                 \hspace{0.7cm}{\small\textit{(solve the loss-augmented decoding problem~\eqref{eq:subproblem_loss_augm})}}
            \STATE $\sv_{(i)} := \unit^{\outputvarv_i^*} \in \domain^{(i)}$
            \hspace{3cm}{\small\textit{(having only a single non-zero entry)}}
            \STATE Let $\stepsize := \frac{2n}{k+2n}$,\  {\small or optimize $\stepsize$ by line-search}
            \STATE Update $\dualvarv^{(k+1)}_{(i)}:= (1-\stepsize)\dualvarv^{(k)}_{(i)}+\stepsize \sv_{(i)}$
            \STATE {\small(Optionally: Update $\bar{\dualvarv}^{(k+1)}:=\frac{k}{k+2}\bar{\dualvarv}^{(k)}+\frac{2}{k+2}\dualvarv^{(k+1)}$)   \hspace{0.4cm}{\small\textit{(maintain a weighted average of the iterates)}}}
    \ENDFOR
\end{algorithmic}
\end{algorithm}
To compute the line-search step-size, we simply re-use the same formula as in Algorithm~\ref{alg:FW_product_SVM}, but reconstructing (implicitly) on the fly the missing quantities such as $\ell_i = \bv^T \dualvarv_{[i]}$, $\weightv_i = A\dualvarv_{[i]}$ and $\weightv^{(k)} = A \dualvarv^{(k)}$, and re-interpreting dot products such as $\weightv_i^T \weightv^{(k)}$ as the suitable sum of kernel evaluations (which has $O(k^2/n)$ terms, where $k$ is the number of iterations since the beginning).

\bigskip
\section{Analysis of the Block-Coordinate Frank-Wolfe Algorithm~\ref{alg:FW_product}}\label{sec:app_convergence_proof}

This section gives a self-contained presentation and analysis of the new block-coordinate Frank-Wolfe optimization Algorithm~\ref{alg:FW_product}. The main goal is to prove the convergence Theorem~\ref{thm:convergence_FW_product}, which here is split into two parts, the \emph{primal} convergence rate in Theorem~\ref{thm:primalGreedyProduct}, and the \emph{primal-dual} convergence rate in Theorem~\ref{thm:primalDualGreedy1Regime}.
Finally, we will present a faster convergence result for the line-search variant in Theorem~\ref{thm:fast_warmup} and Theorem~\ref{thm:primalDualFaster}, which we have used in the convergence for the structural SVM case as presented above in Theorem~\ref{thm:convergence_FW_product_SVM}.

\paragraph{Coordinate Descent Methods.}
Despite their simplicity and very early appearance in the literature,
surprisingly few results were known on the convergence (and convergence rates
in particular) of coordinate descent type methods. Recently, the interest in
these methods has grown again due to their good scalability to very large
scale problems as e.g. in machine learning, and also sparked new theoretical
results such as \citep{Nesterov:2012fa}.

\paragraph{Constrained Convex Optimization over Product Domains.}
We consider the general constrained convex optimization problem
\begin{equation}\label{eqn:opt}
\min_{\x\in \domain} f(\x)
\end{equation}
over a Cartesian product domain
$\domain = \domain^{(1)}\times\mathellipsis\times \domain^{(n)} \subseteq \R^m$, where each factor $\domain^{(i)} \subseteq\R^{m_i}$ is convex and \emph{compact}, and $\sum_{i=1}^n m_i = m$. We will write $\x_{(i)}\in\R^{m_i}$ for the $i$-th block of coordinates of a vector $\x\in\R^m$, and $\x_{[i]}$ for the padding of $\x_{(i)}$ with zeros so that $\x_{[i]} \in\R^m$.

\paragraph{Nesterov's `Huge Scale' Coordinate Descent.}
If the objective function $f$ is strongly smooth (i.e. has Lipschitz continuous partial gradients $\nabla_{(i)} f(\x)\in\R^{m_i}$), then the following algorithm converges%
\footnote{
By additionally assuming strong convexity of $f$ w.r.t. the $\ell_1$-norm
(global on $\domain$, not only on the individual factors), one can even get
linear convergence rates, see again \citep{Nesterov:2012fa} and the follow-up
paper \citepsup{Richtarik:2011vg}.} %
at a rate of $\frac1k~$, or more precisely $\frac{n}{k+n}$, as shown in
\citep[Section 4]{Nesterov:2012fa}:

\begin{algorithm}[h!]
  \caption{Uniform Coordinate Descent Method, \citep[Section 4]{Nesterov:2012fa}}
  \label{alg:UCDM}
\begin{algorithmic}
  \STATE Let $\x^{(0)} \in \domain$%
  \FOR{$k=0\dots\infty$}
  \STATE Pick $i$ uniformly at random in $\{1, \ldots, n\}$
  \STATE Compute $\sv_{(i)} := \displaystyle\argmin_{\sv_{(i)}\in \domain^{(i)}} \textstyle\left\langle \sv_{(i)}%
  , \nabla_{(i)} f(\x^{(k)})
  \right\rangle + \frac{L_i}{2}\norm{\sv_{(i)}-\x_{(i)}}^2$
  \STATE Update $\x^{(k+1)}_{(i)}:=\x^{(k)}_{(i)}+\big(\sv_{(i)} - \x^{(k)}_{(i)}\big)$   \hspace{0.4cm}{\small\textit{(only affecting the $i$-th coordinate block)}}
  \ENDFOR
\end{algorithmic}
\end{algorithm}

\paragraph{Using Simpler Update Steps: Frank-Wolfe / Conditional Gradient Methods.}
In some large-scale applications, the above computation of the update direction $\sv_{(i)}$ can be problematic, e.g. if the Lipschitz constants $L_i$ are unknown, or ---more importantly--- if the domains $\domain^{(i)}$ are such that the quadratic term makes the subproblem for $\sv_{(i)}$ hard to solve.

The structural SVM is a nice example where this makes a big difference. Here,
each domain block~$\domain^{(i)}$ is a simplex of exponentially many
variables, but nevertheless the linear subproblem over one such factor (also known as loss-augmented decoding) is often relatively easy to solve.

We would therefore like to replace the above computation of $\sv_{(i)}$ by a simpler one, as proposed in the following algorithm variant:

\begin{algorithm}[h!]
  \caption{Cheaper Coordinate Descent: Block-Coordinate Frank-Wolfe Algorithm}
  \label{alg:FW_product_again}
\begin{algorithmic}
  \STATE Let $\x^{(0)} \in \domain$ and $\bar{\x}_w^{(0)} = \x^{(0)}$ %
  \FOR{$k=0\dots\infty$}
  \STATE Pick $i$ uniformly at random in $\{1, \ldots, n\}$
  \STATE Compute $\sv_{(i)} := \displaystyle\argmin_{\sv_{(i)}\in \domain^{(i)}} \textstyle\left\langle \sv_{(i)}, \nabla_{(i)} f(\x^{(k)}) \right\rangle$
  \STATE \hspace{1cm}{\small\textit{(or alternatively, find $\sv_{(i)}$ that solves this linear problem approximately,\\
              \hspace{1cm}                             either up to an additive error \eqref{eqn:qualityAdd} or up to a multiplicative error \eqref{eqn:qualityMult})}}
  \STATE Let $\stepsize := \frac{2n}{k+2n}$,\  {\small or perform line-search for the step-size: $\stepsize := \displaystyle\argmin_{\stepsize\in[0,1]} \, f\left(\x^{(k)}+\stepsize\big(\sv_{[i]} - \x^{(k)}_{[i]}\big)\right)$}
  \STATE Update $\x^{(k+1)}_{(i)}:=\x^{(k)}_{(i)}+\stepsize\big(\sv_{(i)} - \x^{(k)}_{(i)}\big)$   \hspace{1.9cm}{\small\textit{(only affecting the $i$-th coordinate block)}}
  \STATE {\small(Optionally: Update $\bar{\x}_w^{(k+1)}:=\frac{k}{k+2}\bar{\x}_w^{(k)}+\frac{2}{k+2}\x^{(k+1)}$)   \hspace{0.4cm}{\small\textit{(maintain a weighted average of the iterates)}}}
  \ENDFOR
\end{algorithmic}
\end{algorithm}

This natural coordinate descent type optimization method picks a single one of the $n$ blocks uniformly at random, and in each step leaves all other blocks unchanged. 

If there is only one factor ($n=1$), then Algorithm \ref{alg:FW_product_again}
becomes the standard Frank-Wolfe (or conditional gradient) algorithm, which is known to converge at a rate of $O(1/k)$ \citep{Frank:1956vp,Dunn:1978di,Clarkson:2010hv,Jaggi:2013wg}.

\paragraph{Using Approximate Linear Minimizers.}
If approximate linear minimizers are used internally in Algorithm~\ref{alg:FW_product_again}, then the necessary approximation quality for the candidate directions $\sv_{(i)}$ is determined as follows (in either additive or multiplicative quality):

In the \emph{additive} case, we choose a fixed additive error parameter $\delta\ge0$ such that the candidate direction $\sv_{(i)}$ satisfies
\begin{equation}\label{eqn:qualityAdd}
\left\langle \sv_{(i)}, \nabla_{(i)} f(\x) \right\rangle
~\le~
\displaystyle\min_{\sv'_{(i)}\in \domain^{(i)}} \textstyle\left\langle \sv'_{(i)}, \nabla_{(i)} f(\x) \right\rangle + \frac12 \delta \,\addFactor_k \,\Cf^{(i)}  \ ,\vspace{-1mm}
\end{equation}
where $\addFactor_k := \frac{2n}{k+2n}$ comes from the default step-size and is used for the convergence results to come. Note that if line-search is used to determine a different step-size, the candidate direction is still defined with respect to the default $\addFactor_k$. 

In the \emph{multiplicative} case, we choose a fixed multiplicative error parameter $0<\mapprox\le1$ such that the candidate directions $\sv_{(i)}$ attain the current `duality gap' on the $i$-th factor up to a \emph{multiplicative} approximation error of~$\mapprox$, i.e.\vspace{-1mm}
\begin{equation}\label{eqn:qualityMult}
\left\langle \x-\sv_{(i)}, \nabla_{(i)} f(\x) \right\rangle
~\ge~
\displaystyle \mapprox\cdot \max_{\sv'_{(i)}\in \domain^{(i)}} \textstyle\left\langle \x-\sv'_{(i)}, \nabla_{(i)} f(\x) \right\rangle
\ .\vspace{-1mm}
\end{equation}
If a multiplicative approximate internal oracle is used together with the predefined step-size instead of doing line-search, then the step-size in Algorithm~\ref{alg:FW_product_again} needs to be increased to $\stepsize_k:= \frac{2n}{\mapprox k+2n}$ instead of the original~$\frac{2n}{k+2n}$.

Both types of errors can be combined together with the following property for the candidate direction $\sv_{(i)}$:
\begin{equation}\label{eqn:qualityBoth}
\left\langle \x-\sv_{(i)}, \nabla_{(i)} f(\x) \right\rangle
~\ge~
\displaystyle \mapprox\cdot \max_{\sv'_{(i)}\in \domain^{(i)}} \textstyle\left\langle \x-\sv'_{(i)}, \nabla_{(i)} f(\x) \right\rangle -  \frac12 \delta \,\addFactor_k \,\Cf^{(i)} \ ,
\end{equation}
where $\addFactor_k:= \frac{2n}{\mapprox k+2n}$.

\paragraph{Averaging the Iterates.}
In the above Algorithm~\ref{alg:FW_product_again} we have also added an optional last line which maintains the following weighted average $\bar{\x}_w^{(k)}$ which is defined for $k \geq 1$ as
\begin{equation} \label{eqn:wavg}
	\bar{\x}_w^{(k)} := \frac{2}{k (k+1)} \sum_{t=1}^k t \, \x^{(t)} \ ,
\end{equation}
and by convention we also define $\bar{\x}_w^{(0)} := \x^{(0)}$. As our convergence analysis will show, the weighted average of the iterates can yield more robust duality gap convergence guarantees when the duality gap function $g$ is convex in $\x$ (see Theorem~\ref{thm:primalDualGreedy1Regime}) -- this is for example the case for quadratic functions such as in the structural SVM objective~\eqref{eq:svmstruct_nslack_dual}. We will also consider in our proofs a scheme which averages the last $(1-\mu)$-fraction of the iterates for some fixed $0 < \mu < 1$:
\begin{equation} \label{eqn:tavg} 
	\bar{\x}_\mu^{(k)} := \frac{1}{k- \ceil{\mu k}+1} \sum_{t= \ceil{\mu k} }^k \x^{(t)} \ .
\end{equation}
This is what \citet{Rakhlin2012} calls \emph{$(1-\mu)$-suffix averaging} and it appeared in the context of getting a stochastic subgradient method with $O(1/k)$ convergence rate for strongly convex functions instead of the standard $O((\log k)/k)$ rate that one can prove for the individual iterates $\x^{(k)}$. The problem with $(1-\mu)$-suffix averaging is that to implement it for a fixed $\mu$ (say $\mu = 0.5$) without storing a fraction of all the iterates, one needs to know when they will stop the algorithm. An alternative mentioned in~\citet{Rakhlin2012} is to maintain a uniform average over rounds of exponentially increasing size (the so-called `doubling trick'). This can give very good performance towards the end of the rounds as we will see in our additional experiments in Appendix~\ref{sec:additional_experiments}, but the performance varies widely towards the beginning of the rounds. This motivates the simpler and more robust weighted averaging scheme~\eqref{eqn:wavg}, which in the case of the stochastic subgradient method, was also recently proven to have $O(1/k)$ convergence rate by~\citet{LacosteJulien:2012uo}\footnote{In this paper, they considered a $(k+1)$-weight instead of our $k$-weight, but similar rates can be proven for shifted versions. We motivate skipping the first iterate $\x^{(0)}$ in our weighted averaging scheme as sometimes bounds can be proven on the quality of $\x^{(1)}$ irrespective of $\x^{(0)}$ for Frank-Wolfe (see the paragraph after the proof of Theorem~\ref{thm:primalGreedyProduct} for example, looking at the $n=1$ case).} and independently by~\citet{Shamir:2013vw}, who called such schemes `polynomial-decay averaging'.

\paragraph{Related Work.}
In contrast to the randomized choice of coordinate which we use here, the analysis of \emph{cyclic} coordinate descent algorithms (going through the blocks sequentially) seems to be notoriously difficult, such that until today, no analysis proving a global convergence rate has been obtained as far as we know. \citetsup{Luo:1992fy} has proven a \emph{local} linear convergence rate for the strongly convex case.

For product domains, such a cyclic analogue of our Algorithm \ref{alg:FW_product_again} has already been proposed in~\citetsup{Patriksson:1998hg}, using a generalization of Frank-Wolfe iterations under the name `cost approximation'.
The analysis of~\citetsup{Patriksson:1998hg} shows asymptotic convergence, but since the method goes through the blocks sequentially, no convergence rates could be proven so far.

\subsection{Setup for Convergence Analysis}
We review below the important concepts needed for analyzing the convergence of the block-coordinate Frank-Wolfe Algorithm~\ref{alg:FW_product_again}.

\paragraph{Decomposition of the Duality Gap.}
The product structure of our domain has a crucial effect on the duality gap, namely that it decomposes into a sum over the $n$ components of the domain.
The `linearization' duality gap as defined in~\eqref{eq:duality_gap} %
(see also \citet{Jaggi:2013wg}) for any constrained convex problem of the above form (\ref{eqn:opt}), for a fixed feasible point $\x \in \domain$, is given by
\begin{equation}\label{eqn:gapProduct}
  \begin{array}{rl}
g(\x) := %
   & \displaystyle\max_{\sv\in \domain} \, \left\langle \x-\sv, \nabla f(\x)\right\rangle \\
 =& \displaystyle\sum_{i=1}^n \max_{\sv_{(i)}\in \domain^{(i)}} \, \left\langle \x_{(i)}-\sv_{(i)}, \nabla_{(i)} f(\x)\right\rangle\\
 =:& \displaystyle\sum_{i=1}^n \ g^{(i)}(\x) \ .
  \end{array}
\end{equation}

\paragraph{Curvature.}
Also, the curvature can now be defined on the individual factors, %
\begin{equation}\label{eqn:CfProduct}
  \Cf^{(i)} := \sup_{\substack{\x\in \domain,\,\sv_{(i)}\in \domain^{(i)}, \\
                      \stepsize\in[0,1],\\
                      \y = \x+\stepsize(\sv_{[i]}-\x_{[i]})}}
          \textstyle \frac{2}{\stepsize^2}\left( f(\y)-f(\x)-\langle \y_{(i)}-\x_{(i)}, \nabla_{(i)} f(\x)\rangle \right) \ .\vspace{-2pt}
\end{equation}
We recall that the notation $\x_{[i]}$ and $\x_{(i)}$ is defined just below~\eqref{eqn:opt}.
We define the global product curvature as the sum of these curvatures for each block, i.e.
\begin{equation}\label{eqn:CfProductGlobal}
  \CfTotal := \sum_{i=1}^n \Cf^{(i)} .
\end{equation}

\subsection{Primal Convergence on Product Domains}
The following main theorem shows that after $O\big(\frac1\varepsilon\big)$ many iterations, Algorithm~\ref{alg:FW_product_again} obtains an $\varepsilon$-approximate solution.

\begin{theorem}[Primal Convergence]\label{thm:primalGreedyProduct}
For each $k\ge 0$, the iterate $\x^{(k)}$ of the exact variant of Algorithm~\ref{alg:FW_product_again} satisfies
\[
\E[f(\x^{(k)})] - f(\x^*) \le \frac{2n}{k+2n}\big(\CfTotal+ f(\x^{(0)}) - f(\x^*)\big) \ ,
\]
For the approximate variant of Algorithm~\ref{alg:FW_product_again} with \emph{additive} approximation quality~(\ref{eqn:qualityAdd}) for $\delta\ge0$, it holds that
\[
\E[f(\x^{(k)})] - f(\x^*) \le \frac{2n}{k+2n}\big(\CfTotal(1+\delta)+ f(\x^{(0)}) - f(\x^*)\big) \ .
\]
For the approximate variant of Algorithm~\ref{alg:FW_product_again}, with \emph{multiplicative} approximation quality~(\ref{eqn:qualityMult}) for $0<\mapprox\le1$, it holds that
\[
\E[f(\x^{(k)})] - f(\x^*) \le \frac{2n}{\mapprox k+2n}\big(\frac1\mapprox\CfTotal+ f(\x^{(0)}) - f(\x^*)\big) \ .
\]

All convergence bounds hold both if the predefined step-sizes, or line-search is used in the algorithm. Here $\x^*\in \domain$ is an optimal solution to problem~(\ref{eqn:opt}), and the expectation is with respect to the random choice of blocks during the algorithm.
(In other words all three algorithm variants deliver a solution of (expected) primal error at most~$\varepsilon$ after $O(\frac1\varepsilon)$ many iterations.)
\end{theorem}

The proof of the above theorem on the convergence rate of the primal error crucially depends on the following Lemma~\ref{lem:stepProduct} on the improvement in each iteration.

\begin{lemma}\label{lem:stepProduct}
Let $\stepsize\in[0,1]$ be an arbitrary fixed step-size.
Moving only within the $i$-th block of the domain, we consider two variants of steps towards a direction $\sv_{(i)}\in \domain^{(i)}$:
Let $\x^{(k+1)}_{\stepsize} := \x(\stepsize)$ be the point obtained by moving towards~$\sv_{(i)}$ using step-size $\stepsize$,
and let $\x^{(k+1)}_{LS} := \x(\stepsize_{LS})$ be the corresponding point obtained by line-search, i.e. $\stepsize_{LS} := \displaystyle\argmin_{\bar\stepsize\in[0,1]} \, f\left(\x(\bar\stepsize)\right)$.
Here for convenience we have used the notation $\x(\bar\stepsize) := \x^{(k)}+\bar\stepsize\big(\sv_{[i]} - \x^{(k)}_{[i]}\big)$ for $\bar\stepsize\in[0,1]$.

If for each $i$ the candidate direction $\sv_{(i)}$ satisfies the \emph{additive} approximation quality~(\ref{eqn:qualityAdd}) for $\delta\ge 0$ and some fixed $\addFactor_k$, then in expectation over the random choice of the block $i$ and conditioned on $\x^{(k)}$, it holds that
\[
  \E\big[ f(\x^{(k+1)}_{LS})\,|\, \x^{(k)}\big]
  \le \E\big[ f(\x^{(k+1)}_\stepsize) \,|\, \x^{(k)}\big]
  ~\le~ f(\x^{(k)}) - \frac{\stepsize}{n} g(\x^{(k)}) + \frac{1}{2n} (\stepsize^2 + \delta \addFactor_k \stepsize) \CfTotal \ .
\]
On the other hand, if $\sv_{(i)}$ attains the duality gap $g^{(i)}(\x)$ on the $i$-th block up to a \emph{multiplicative} approximation quality~(\ref{eqn:qualityMult}) for $0<\mapprox\le1$, then\vspace{-2mm}
\[
  \E\big[ f(\x^{(k+1)}_{LS}) \,|\, \x^{(k)}\big]
  \le \E\big[ f(\x^{(k+1)}_\stepsize) \,|\, \x^{(k)}\big]
  ~\le~ f(\x^{(k)}) - \frac{\stepsize}{n} \mapprox \, g(\x^{(k)}) + \frac{\stepsize^2}{2n} \CfTotal \ .
\]
All expectations are taken over the random choice of the block $i$ and conditioned on~$\x^{(k)}$.
\end{lemma}
\begin{proof}
  We write $\x := \x^{(k)}$, $\y := \x^{(k+1)}_\stepsize = \x+\stepsize(\sv_{[i]}-\x_{[i]})$, with $\x_{[i]}$ and $\sv_{[i]}$ being zero everywhere except in their $i$-th block. We also write $d_\x := \nabla_{(i)} f(\x)$ to simplify the notation.
  From the definition~(\ref{eqn:CfProduct}) of the curvature constant $\Cf^{(i)}$ of our convex function $f$ over the factor $\domain^{(i)}$, we have
  \[
  \begin{array}{rl}
	f(\y) = & f(\x+\stepsize(\sv_{[i]}-\x_{[i]})) \\
    \le & f(\x) + \stepsize \langle \sv_{(i)}-\x_{(i)}, d_\x\rangle + \frac{\stepsize^2}{2}\Cf^{(i)} \ .
  \end{array}
  \]
  Now we use that by \eqref{eqn:qualityAdd}, the choice of $\sv_{(i)}$ with $\left\langle \sv_{(i)}, \nabla_{(i)} f(\x) \right\rangle \le \displaystyle\min_{\sv'_{(i)}\in \domain^{(i)}} \textstyle\left\langle \sv'_{(i)}, \nabla_{(i)} f(\x) \right\rangle + \frac12\delta \addFactor_k \Cf^{(i)}$
 is a good descent direction for the linear approximation to $f$ at $\x$, on the $i$-th factor $\domain^{(i)}$, giving
  \begin{equation}\label{eqn:lemmaStepProofGap}\textstyle
  \langle \sv_{(i)}-\x_{(i)}, d_\x\rangle
  \le -g^{(i)}(\x) + \frac{\delta \addFactor_k}{2} \Cf^{(i)} \ ,
  \end{equation}
  by the definition~(\ref{eqn:gapProduct}) of the duality gap. Altogether, we have obtained
  \[
  \begin{array}{rl}
  f(\y) \le& f(\x) + \stepsize (-g^{(i)}(\x) + \frac{\delta \addFactor_k}{2} \Cf^{(i)}) + \frac{\stepsize^2}{2}\Cf^{(i)} \\
         =& f(\x) - \stepsize g^{(i)}(\x) + \frac{1}{2} (\stepsize^2 + \delta \addFactor_k \stepsize) \Cf^{(i)} \ .
  \end{array}
  \]
  Using that the line-search by definition must lead to an objective value at least as good as the one at the fixed~$\stepsize$, we therefore have shown the inequality
\[\textstyle
  f(\x^{(k+1)}_{LS})
  \le f(\x^{(k+1)}_\stepsize)
  ~\le~ f(\x^{(k)}) - \stepsize g^{(i)}(\x^{(k)}) + \frac{1}{2} (\stepsize^2 + \delta \addFactor_k \stepsize) \Cf^{(i)} \ .
\]
  Finally the claimed bound on the expected improvement directly follows by taking the expectation:
  With respect to the (uniformly) random choice of the block $i$, the expected value of the gap $g^{(i)}(\x^{(k)})$ corresponding to the picked $i$ is exactly $\frac1n g(\x^{(k)})$. Also, the expected curvature of the $i$-th factor is~$\frac1n \CfTotal$.

  The proof for the case of multiplicative approximation follows completely analogously, using
  $
  \langle \sv_{(i)}-\x_{(i)}, d_\x\rangle \le -\mapprox\, g^{(i)}(\x) ,
  $
  which then gives a step improvement of
  $
  f(\y) \le f(\x) - \stepsize \mapprox g^{(i)}(\x) + \frac{\stepsize^2}{2} \Cf^{(i)} \ .
  $
\end{proof}

Having Lemma~\ref{lem:stepProduct} at hand, we will now prove our above primal
convergence Theorem~\ref{thm:primalGreedyProduct} using similar ideas as for
general domains, such as in~\citet{Jaggi:2013wg}.

\begin{proof}[Proof of Theorem~\ref{thm:primalGreedyProduct}]
We first prove the theorem for the approximate variant of Algorithm~\ref{alg:FW_product_again} with multiplicative approximation quality~\eqref{eqn:qualityMult} of $0<\mapprox\le1$ -- the exact variant of the algorithm is simply the special case $\mapprox = 1$. 
From the above Lemma~\ref{lem:stepProduct}, we know that for every inner step of Algorithm~\ref{alg:FW_product_again} and conditioned on $\x^{(k)}$, we have that
$\E[ f(\x_\stepsize^{(k+1)}) \,|\, \x^{(k)}] \le f(\x^{(k)}) - \frac{\stepsize \mapprox}{n} g(\x^{(k)}) + \frac{\stepsize^2}{2n} \CfTotal$,
   where the expectation is over the random choice of the block~$i$ (this bound holds independently whether line-search is used or not). Writing $h(\x) := f(\x) - f(\x^*)$ for the (unknown) primal error at any point $\x$, this reads as
  \begin{equation}\label{eqn:hStepProductExp_conditional}
  \begin{array}{rl}
                   \E[ h(\x_\stepsize^{(k+1)}) \,|\, \x^{(k)}] \le& h(\x^{(k)}) - \frac{\stepsize \mapprox}{n} g(\x^{(k)}) + \frac{\stepsize^2}{2n} \CfTotal \\
                    \le& h(\x^{(k)}) - \frac{\stepsize \mapprox}{n} h(\x^{(k)}) + \frac{\stepsize^2}{2n} \CfTotal  \\
                    =& (1- \frac{\stepsize \mapprox}{n}) h(\x^{(k)}) + \frac{\stepsize^2}{2n} \CfTotal,
  \end{array}
  \end{equation}
where in the second line, we have used \emph{weak duality} $h(\x) \le g(\x)$ (which follows directly from the definition of the duality gap, together with convexity of $f$%
). The inequality~\eqref{eqn:hStepProductExp_conditional} is conditioned on $\x^{(k)}$, which is a random quantity given the previous random choices of blocks to update. We get a deterministic inequality by taking the expectation of both sides with respect to the random choice of previous blocks, yielding:
  \begin{equation}\label{eqn:hStepProductExp}
  \begin{array}{rl}
                   \E[ h(\x_\stepsize^{(k+1)})] \le& (1- \frac{\stepsize \mapprox}{n}) \E[ h(\x^{(k)}) ] + \frac{\stepsize^2}{2n} \CfTotal.
  \end{array}
  \end{equation}

We observe that the resulting inequality~\eqref{eqn:hStepProductExp} with $\mapprox=1$ is of the
same form as the one appearing in the standard Frank-Wolfe primal convergence
proof such as in~\citet{Jaggi:2013wg}, though with a crucial difference of the $1/n$ factor (and that we are now working with the expected values $\E[ h(\x^{(k)})]$ instead of the original $h(\x^{(k)})$). We will thus follow a similar induction argument over~$k$, but we will see that the~$1/n$ factor will yield a slightly different induction base case (which for $n=1$ can be analyzed separately to obtain a better bound). To simplify the notation, let $h_k := \E[ h(\x^{(k)})]$.

By induction, we are now going to prove that
  \[
  h_k \le \frac{2n C}{\mapprox k+2n}~~~~~~ \text{ for }\ k\geq0 \ .
  \]
  for the choice of constant $C:=\frac1\mapprox \CfTotal + h_0$.

  The \emph{base-case} $k=0$ follows immediately from the definition of $C$, given that $C \geq h_0$.

  Now we consider the \emph{induction step} for $k \geq 0$. Here the bound~(\ref{eqn:hStepProductExp}) for the particular choice of step-size $\stepsize_k:=\frac{2n}{\mapprox k+2n} \in[0,1]$ given by Algorithm~\ref{alg:FW_product_again} gives us (the same bound also holds for the line-search variant, given that the corresponding objective value $f(\x^{(k+1)}_{\text{\tiny Line-Search}})\le f(\x^{(k+1)}_{\stepsize})$ only improves): %
  \[
  \begin{array}{rl}
  h_{k+1} \le& (1- \frac{\stepsize_k\mapprox}{n}) h_k + (\stepsize_k)^2 \frac{C\mapprox}{2n} \\[3pt]
                      =& (1-\frac{2\mapprox}{\mapprox k+2n}) h_k + (\frac{2n}{\mapprox k+2n})^2 \frac{C\mapprox}{2n}  \\[3pt]
                      \le& (1-\frac{2\mapprox}{\mapprox k+2n}) \frac{2n C}{\mapprox k+2n} + (\frac{1}{\mapprox k+2n})^2 2nC\mapprox \ ,
  \end{array}
  \]
  where in the first line we have used that $\CfTotal \le C \mapprox$, and in the last inequality we have plugged in the induction hypothesis for $h_k$. Simply rearranging the terms gives
  \[
  \begin{array}{rl}
   h_{k+1} \le& \frac{2nC}{\mapprox k+2n} \left(1 - \frac{2\mapprox}{\mapprox k+2n} + \frac{\mapprox}{\mapprox k+2n}\right) \\[3pt]
                      =& \frac{2nC}{\mapprox k+2n} \frac{\mapprox k+2n-\mapprox}{\mapprox k+2n} \\[3pt]
                     \le& \frac{2nC}{\mapprox k+2n} \frac{\mapprox k+2n}{\mapprox k+2n+\mapprox} \\[3pt]
                      =& \frac{2nC}{\mapprox (k+1)+2n} \ ,
  \end{array}
  \]
  which is our claimed bound for $k\ge 0$.
  	
  The analogous claim for Algorithm~\ref{alg:FW_product_again} using the approximate linear primitive with \emph{additive} approximation quality~\eqref{eqn:qualityAdd} with $\addFactor_k = \frac{2n}{\mapprox k+2n}$ follows from exactly the same argument, by replacing every occurrence of $\CfTotal$ in the proof here by~$\CfTotal(1+\delta)$ instead (compare to Lemma~\ref{lem:stepProduct} also -- note that $\stepsize = \addFactor_k$ here). Note moreover that one can combine easily both a multiplicative approximation with an additive one as in~\eqref{eqn:qualityBoth}, and modify the convergence statement accordingly.
\end{proof}

\paragraph{Domains Without Product Structure: $n=1$.}
Our above convergence result also holds for the case of the standard Frank-Wolfe algorithm, when no product structure on the domain is assumed, i.e. for the case $n=1$.
In this case, the constant in the convergence can even be improved for the  variant of the algorithm without a multiplicative approximation ($\mapprox = 1$), since the additive term given by $h_0$, i.e. the error at the starting point, can be removed. This is because already after the first step, we obtain a bound for $h_1$ which is independent of $h_0$. More precisely, plugging $\stepsize_0 :=1$ and $\mapprox=1$ in the bound~\eqref{eqn:hStepProductExp} when $n=1$ gives $h_1 \leq 0+\CfTotal (1+\delta) \leq C$. Using $k=1$ as the base case for the same induction proof as above, we obtain that for $n=1$:
\[
    h_k \leq \frac{2}{k+2} \CfTotal (1+\delta) ~~\text{ for all }\ k\geq1 \ ,
\]
which matches the convergence rate given in~\citet{Jaggi:2013wg}.
Note that in the traditional Frank-Wolfe setting, i.e. $n=1$, our defined curvature constant becomes $\CfTotal = \Cf$.%

\paragraph{Dependence on $h_0$.} We note that the only use of including $h_0$ in the constant $C = \mapprox^{-1}\CfTotal + h_0$ was to satisfy the base case in the induction proof, at $k=0$. If from the structure of the problem we can get a guarantee that $h_0 \leq \mapprox^{-1}\CfTotal$, then the smaller constant $C' = \mapprox^{-1}\CfTotal$ will satisfy the base case and the whole proof will go through with it, without needing the extra $h_0$ factor. See also Theorem~\ref{thm:fast_warmup} for a better convergence result with a weaker dependence on $h_0$ in the case where the line-search is used.

\subsection{Obtaining Small Duality Gap}

The following theorem shows that after $O\big(\frac1\varepsilon\big)$ many iterations, Algorithm~\ref{alg:FW_product_again} will have visited a solution with $\varepsilon$-small duality gap in expectation. Because the block-coordinate Frank-Wolfe algorithm is only looking at one block at a time, it doesn't know what is its current true duality gap without doing a full (batch) pass over all blocks. Without monitoring this quantity, the algorithm could miss which iterate had a low duality gap. This is why, if one is interested in having a good duality gap (such as in the structural SVM application), then the averaging schemes considered in~\eqref{eqn:wavg} and~\eqref{eqn:tavg} become interesting: the following theorem also says that the bound hold for \emph{each} of the averaged iterates, \emph{if the duality gap function $g$ is convex}, which is the case for example when $f$ is a quadratic function.\footnote{To see that $g$ is convex when $f$ is quadratic, we refer to the equivalence between the gap $g(\x)$ and the Fenchel duality $p(\x)-d(\nabla f(\x)))$ as shown in Appendix~\ref{sec:Fenchel_gap_equivalence}. The dual function $d(\cdot)$ is concave, so if $\nabla f(\x))$ is an affine function of $\x$ (which is the case for a quadratic function), then $d$ will be a concave function of $\x$, implying that $g(\x) = p(\x)-d(\nabla f(\x)))$ is convex in $\x$, since the primal function $p$ is convex.}

\begin{theorem}[Primal-Dual Convergence]\label{thm:primalDualGreedy1Regime}
 For each $K\ge 0$, the variants of Algorithm~\ref{alg:FW_product_again} (either using the predefined step-sizes, or using line-search) will yield at least one iterate $\x^{(\hat k)}$ with $\hat k\le K$ with expected duality gap bounded by
\[
\E\big[g(\x^{(\hat k)})\big] \le \beta \frac{2 n}{\mapprox (K+1)} C  \ ,
\]
where $\beta=3$ and $C = \mapprox^{-1}\CfTotal(1+\delta)+ f(\x^{(0)}) - f(\x^*)$. $\delta \geq 0$ and $0 < \mapprox \leq 1$ are the approximation quality parameters as defined in~(\ref{eqn:qualityBoth}) %
-- use $\delta=0$ and $\mapprox=1$ for the exact variant. 

Moreover, if the duality gap $g$ is a convex function of $\x$, then the above bound also holds both for $\E\big[g(\bar{\x}_w^{(K)})\big]$ and $\E\big[g(\bar{\x}_{0.5}^{(K)})\big]$ for each $K \geq 0$, where $\bar{\x}_w^{(K)}$ is the weighted average of the iterates as defined in~(\ref{eqn:wavg}) and  $\bar{\x}_{0.5}^{(K)}$ is the $0.5$-suffix average of the iterates as defined in~(\ref{eqn:tavg}) with $\mu=0.5$.
\end{theorem}

\begin{proof}
To simplify notation, we will again denote the expected primal error and expected duality gap for any iteration $k\ge0$ in the algorithm by $h_k := \E[ h(\x^{(k)})] :=  \E[ f(\x^{(k)}) - f(\x^{*})]$ and $g_k := \E[g(\x^{(k)})]$ respectively.

The proof starts again by using the crucial improvement Lemma~\ref{lem:stepProduct} with $\stepsize = \stepsize_k := \frac{2n}{\mapprox k + 2n}$ to cover both variants of Algorithm~\ref{alg:FW_product_again} at the same time. As in the beginning of the proof of Theorem~\ref{thm:primalGreedyProduct}, we take the expectation with respect to $\x^{(k)}$ in Lemma~\ref{lem:stepProduct} and subtract $f(\x^{*})$ to get that for each $k \geq 0$ (for the general approximate variant of the algorithm):
  \[
  \begin{array}{rl}
   h_{k+1} \le& h_k - \frac{1}{n}\stepsize_k \mapprox \, g_k + \frac{1}{2n} ( {\stepsize_k}^2 + \delta \addFactor_k \stepsize_k) \CfTotal \\[3pt]
   =& h_k - \frac{1}{n}\stepsize_k \mapprox \, g_k + \frac{1}{2n}  {\stepsize_k}^2 \CfTotal (1+\delta) \ ,\\[3pt]
  \end{array}
  \]
since $\addFactor_k \leq \stepsize_k$. By isolating $g_k$ and using the fact that $C \geq \mapprox^{-1}\CfTotal (1+\delta)$, we get the crucial inequality for the expected duality gap:
\begin{equation} \label{eqn:gk_inequality}
	g_k \leq \frac{n}{\mapprox \stepsize_k} (h_k - h_{k+1}) + \stepsize_k \frac{C}{2} \ . 
\end{equation}
The general proof idea to get an handle on $g_k$ is to take a convex combination over multiple $k$'s of the inequality~\eqref{eqn:gk_inequality}, to obtain a new upper bound. Because a convex combination of numbers is upper bounded by its maximum, we know that the new bound has to upper bound at least one of the $g_k$'s (this gives the existence $\hat{k}$ part of the theorem). Moreover, if $g$ is convex, we can also obtain an upper bound for the expected duality gap of the same convex combination of the iterates.

So let $\{w_k\}_{k=0}^K$ be a set of non-negative weights, and let $\rho_k := w_k / S_K$, where $S_K := \sum_{k=0}^K w_k$. Taking the convex combination of inequality~\eqref{eqn:gk_inequality} with coefficient $\rho_k$, we get
\begin{align}
	\sum_{k=0}^K \rho_k g_k &\leq \frac{n}{\mapprox} \sum_{k=0}^K \rho_k \left(\frac{h_k}{\stepsize_k} - \frac{h_{k+1}}{\stepsize_k} \right) + \sum_{k=0}^K \rho_k \stepsize_k \frac{C}{2} \nonumber \\
	&=  \frac{n}{\mapprox} \left(h_0 \frac{\rho_0}{\stepsize_0} - h_{K+1}  \frac{\rho_K}{\stepsize_K}\right) +  \frac{n}{\mapprox} \sum_{k=0}^{K-1} h_{k+1}  
	  \left(\frac{\rho_{k+1}}{\stepsize_{k+1}} -  \frac{\rho_{k}}{\stepsize_k} \right) 
	  + \sum_{k=0}^K \rho_k \stepsize_k \frac{C}{2} \nonumber \\
	&\leq  \frac{n}{\mapprox} h_0 \frac{\rho_0}{\stepsize_0} +  \frac{n}{\mapprox} \sum_{k=0}^{K-1} h_{k+1}  
		  \left(\frac{\rho_{k+1}}{\stepsize_{k+1}} -  \frac{\rho_{k}}{\stepsize_k} \right) 
		  + \sum_{k=0}^K \rho_k \stepsize_k \frac{C}{2} \ , \label{eqn:convex_gap_combo}
\end{align}
using $h_{K+1} \geq 0$. Inequality~\eqref{eqn:convex_gap_combo} can be seen as a master inequality to derive various bounds on $g_k$. In particular, if we define $\bar{\x} := \sum_{k=0}^K \rho_k \x^{(k)}$ and we suppose that $g$ is convex (which is the case for example when $f$ is a quadratic function), then we have $\E[ g(\bar{\x}) ] \leq  \sum_{k=0}^K \rho_k g_k$ by convexity and linearity of the expectation.

\paragraph{Weighted-averaging case.} We first consider the weights $w_k = k$ which appear in the definition of the weighted average of the iterates $\bar{\x}_w^{(K)}$ in~\eqref{eqn:wavg} and suppose $K \geq1$. In this case, we have $\rho_k = k / S_K$ where $S_K = K(K+1)/2$. With the predefined step-size $\stepsize_k = 2n / (\mapprox k + 2n)$, we then have 
\begin{align*}
    \frac{\rho_{k+1}}{\stepsize_{k+1}} -  \frac{\rho_{k}}{\stepsize_k} &= \frac{1}{2n S_K} \left( (k+1)(\mapprox(k+1)+2n) - k (\mapprox k+2n)\right) \\
   &= \frac{\mapprox(2k+1)+2n}{2n S_K} \ .
\end{align*}
Plugging this in the master inequality~\eqref{eqn:convex_gap_combo} as well as using the convergence rate $h_k \leq \frac{2nC}{\mapprox k + 2n}$ from Theorem~\ref{thm:primalGreedyProduct}, we obtain
\begin{align*} 
	\sum_{k=0}^K \rho_k g_k 
	&\leq \frac{n}{\mapprox S_K} \left[0 + \sum_{k=0}^{K-1} \frac{2nC}{\mapprox(k+1) + 2n} \frac{\mapprox(2k+1)+2n}{2n}
			   \right] 
			  + \sum_{k=0}^K \frac{2nk}{\mapprox k + 2n} \frac{C}{2 S_K} \nonumber \\
	&\leq \frac{nC}{\mapprox S_K} \left[ 2 \sum_{k=0}^{K-1} 1 + \sum_{k=1}^{K} 1 \right] \nonumber \\
	&= \frac{2nC}{\mapprox (K+1)} \cdot 3 .
\end{align*}
Hence we have proven the bound with $\beta = 3$ for $K \geq 1$. For $K=0$, the master inequality~\eqref{eqn:convex_gap_combo} becomes
\[
   g_0 \le \frac{n}{\mapprox} h_0 + \frac{1}{2}C \leq \frac{nC}{\mapprox} \left(1+\frac{1}{2n}\right)
\]
since $h_0 \leq C$ and $\nu \leq 1$. Given that $n\geq 1$, we see that the bound also holds for $K=0$.

\paragraph{Suffix-averaging case.} For the proof of convergence of the $0.5$-suffix averaging of the iterates $\bar{\x}_{0.5}^{(K)}$, we refer the reader to the proof of Theorem~\ref{thm:primalDualFaster} which can be re-used for this case (see the last paragraph of the proof to explain how).
\end{proof}

\paragraph{Domains Without Product Structure: $n=1$.} As we mentioned after the proof of the primal convergence Theorem~\ref{thm:primalGreedyProduct}, we note that if $n=1$, then we can replace $C$ in the statement of Theorem~\ref{thm:primalDualGreedy1Regime} by $\CfTotal (1+\delta)$ for $K\geq1$ when $\mapprox=1$, as then we can ensure that $h_1 \leq C$ which is all what was needed for the primal convergence induction. Again, $\CfTotal = \Cf$ when $n=1$.

\subsection{An Improved Convergence Analysis for the Line-Search Case}\label{ssec:fast_warmup}%

\subsubsection{Improved Primal Convergence for Line-Search}
If line-search is used, we can improve the convergence results of Theorem~\ref{thm:primalGreedyProduct} by showing a weaker dependence on the starting condition $h_0$ thanks to faster progress in the starting phase of the first few iterations:

\begin{theorem}[Improved Primal Convergence for Line-Search]\label{thm:fast_warmup}
For each $k\ge k_0$, the iterate $\x^{(k)}$ of the line-search variant of
Algorithm~\ref{alg:FW_product_again} (where the linear subproblem is solved with a multiplicative approximation quality~(\ref{eqn:qualityMult}) of $0<\mapprox \leq 1$) satisfies
\begin{equation} \label{eqn:PrimalFasterRate}
\E\big[f(\x^{(k)})\big] - f(\x^*) ~\le~
\frac{1}{\mapprox}\frac{2n\CfTotal}{\mapprox(k-k_0)+2n}
\end{equation}
where $k_0 := \max\big\{0, \left\lceil 
\log\left(\frac{2 \mapprox h(\x^{(0)})}{\CfTotal}\right) 
\Big/ (-\log\xi_n) \right\rceil \big\}$ 
is the number of steps required to guarantee that $\E\big[f(\x^{(k)})\big] - f(\x^*) \leq \mapprox^{-1}\CfTotal$, with $\x^*\in \domain$ being an optimal solution to problem~(\ref{eqn:opt}), and $h(\x^{(0)}) := f(\x^{(0)}) - f(\x^*)$ is the primal error at the starting point, and
$\xi_n:=1-\frac{\mapprox}{n}<1$ is the geometric decrease rate of the primal error in the first phase while $k < k_0$ --- i.e. $\E\big[f(\x^{(k)})\big] - f(\x^*) \leq (\xi_n)^k \  h(\x^{(0)}) + \CfTotal/2\mapprox$ for $k < k_0$.

If the linear subproblem is solved with an additive approximation quality~(\ref{eqn:qualityAdd}) of $\delta \geq 0$ instead, then replace all appearances of $\CfTotal$ above with $\CfTotal (1+\delta)$.
\end{theorem}
\begin{proof}
For the line-search case, the expected improvement guaranteed by Lemma~\ref{lem:stepProduct} for the multiplicative approximation variant of Algorithm~\ref{alg:FW_product_again}, in expectation as in~\eqref{eqn:hStepProductExp}, is valid for \emph{any} choice of $\stepsize\in[0,1]$:
\begin{equation}\label{eqn:gammaLSimpr}
\begin{array}{rl}
  \E\big[ h(\x^{(k+1)}_{LS})\big]  \le& (1 - \frac{\mapprox\stepsize }{n}) \E \big[ h(\x^{(k)}) \big] + \frac{\stepsize^2}{2n} \CfTotal \ .

\end{array}
\end{equation}
Because the bound~\eqref{eqn:gammaLSimpr} holds for any $\stepsize$, we are free to choose the one which minimizes it subject to $\stepsize \in [0,1]$, that is $\stepsize^*:=\min\left\{1,\frac{\mapprox h_k}{\CfTotal}\right\}$, where we have again used the identification $h_k := \E\big[ h(\x^{(k)}_{LS})\big]$.
Now we distinguish two cases:

If $\stepsize^*=1$, then $\mapprox h_k \ge \CfTotal$. By unrolling the inequality~\eqref{eqn:gammaLSimpr} recursively to the beginning and using $\stepsize = 1$ at each step, we get:
\[
\begin{array}{rl}
  h_{k+1}  \le& \left(1-\frac{\mapprox}{n}\right) h_k + \frac{1}{2n} \CfTotal \\ 
  \le& \left(1-\frac{\mapprox}{n}\right)^{k+1} h_0 +  \frac{1}{2n} \CfTotal \sum_{t=0}^k \left(1-\frac{\mapprox}{n}\right)^t \\
   \le& \left(1-\frac{\mapprox}{n}\right)^{k+1} h_0 +  \frac{1}{2n} \CfTotal \sum_{t=0}^\infty \left(1-\frac{\mapprox}{n}\right)^t \\
   =& \left(1-\frac{\mapprox}{n}\right)^{k+1} h_0 + \frac{1}{2n} \CfTotal \left(\frac{1}{1-(1-\mapprox/n)}\right) \\
   =& \left(1-\frac{\mapprox}{n}\right)^{k+1} h_0 + \frac{1}{2\mapprox} \CfTotal \ .\\
\end{array}
\]
We thus have a geometric decrease with rate $\xi_n:=1-\frac{\mapprox}{n}$ in this phase. We then get $h_k \leq \mapprox^{-1}\CfTotal$ as soon as $(\xi_n)^k h_0 \leq \CfTotal/2\mapprox$, i.e. when $k\ge\log_{1/\xi_n}(2 \mapprox h_0/\CfTotal)
= \log(2 \mapprox h_0/\CfTotal)/-\log(1-\frac{\mapprox}{ n})$. We thus have obtained a logarithmic bound on the number of steps that fall into the first regime case here, i.e. where $h_k$ is still `large'. Here it is crucial to note that the primal error $h_k$ is always decreasing in each step, due to the line-search, so once we leave this regime of $h_k\ge \mapprox^{-1}\CfTotal$, then we will never enter it again in subsequent steps.

On the other hand, as soon as we reach a step $k$ (e.g. when $k=k_0)$ such that $\stepsize^*<1$ or equivalently $h_k<\mapprox^{-1}\CfTotal$, then we are always in the second phase where $\stepsize^* = \frac{\mapprox h_k}{\CfTotal}$. Plugging this value of $\stepsize^*$ in (\ref{eqn:gammaLSimpr}) yields the recurrence bound:
\begin{equation}\label{eqn:fastRecurrence}
  h_{k+1} \le h_k - \frac{1}{\zeta} h_k^2   \quad \forall k \geq k_0
\end{equation}
where $\zeta := \frac{2n\CfTotal}{\mapprox^2}$, with the initial condition $h_{k_0} \leq \frac{\CfTotal}{\mapprox}
= \frac{\mapprox \zeta}{2n}$. This is a standard recurrence inequality which appeared for
example in \citet[Theorem 5, see their Equation (23)]{Joachims:2009ex} or
in the appendix of~\citetsup{Teo:2007bi}. We can solve the recurrence
\eqref{eqn:fastRecurrence} by following the argument of \citetsup{Teo:2007bi}, where it was pointed out that since $h_k$ is monotonically decreasing, we can upper bound $h_k$ by the solution to the corresponding differential equations $h'(t) = -h^2(t) / \zeta$, with initial condition $h(k_0) = h_{k_0}$. Integrating both sides, we get the solution $h(t) = \frac{\zeta}{t-k_0+\zeta/h_{k_0}}$. Plugging in the value for $h_{k_0}$ and since $h_k \leq h(k)$, we thus get the bound:
\begin{equation}\label{eqn:recurrenceSolution}
  h_{k} \le \frac{1}{\mapprox}\frac{2n\CfTotal}{\mapprox(k-k_0)+2n}  \quad \forall k \geq k_0,
\end{equation}
which completes the proof for the multiplicative approximation variant.

For the \emph{additive approximation} variant, the inequality~\eqref{eqn:gammaLSimpr} with $\stepsize = 1$ in Lemma~\ref{lem:stepProduct} becomes:
\[
\begin{array}{rl}
  h_{k+1}  \le& \left(1-\frac{\mapprox}{n}\right) h_k + \frac{1}{2n} (1+ \delta \addFactor_k) \CfTotal \\ 
  \leq& \left(1-\frac{\mapprox}{n}\right) h_k + \frac{1}{2n} (1+ \delta) \CfTotal \ , \\ 
\end{array}
\]
since $\addFactor_k \leq 1$. By unrolling this inequality as before, we get the geometric rate of decrease in the initial phase by using $\stepsize = 1$ until $k=k_0$ where we can ensure that $h_{k_0} \leq \CfTotal (1+ \delta)/ \mapprox$. We then finish the proof by re-using the induction proof from Theorem~\ref{thm:primalGreedyProduct}, but with Equation~\eqref{eqn:PrimalFasterRate} as the induction hypothesis, replacing $\CfTotal$ with $\CfTotal (1+ \delta)$. The base case at $k = k_0$ is satisfied by the definition of $k_0$. For the induction step, we use $\stepsize_k = \frac{2n}{\mapprox (k-k_0) + 2n}$ (note that because we use line-search, we are free to use any $\stepsize$ we want in the inequality from Lemma~\ref{lem:stepProduct}), and use the crucial fact that $\addFactor_k = \frac{2n}{\mapprox k + 2n}\leq \stepsize_k$ to get a similar argument as in Theorem~\ref{thm:primalGreedyProduct}.
\end{proof}

\paragraph{Number of Iterations.} We now make some observations in the case of $\delta = 0$ (for simplicity). Note that since for $n > 0.5$ and $-\log\left(1-\frac{\mapprox}{n}\right) > \frac{\mapprox}{n}$ for the natural logarithm, we get that $k_0 \leq  \left\lceil \frac{n}{\mapprox} \log\left( \frac{2 \mapprox h(\x^{(0)})}{\CfTotal} \right) \right\rceil$ and so unless the structure of our problem can guarantee that $h(\x^{(0)}) \leq \CfTotal/ \mapprox$, we get a linear number of steps in~$n$ required to reach the second phase, but the dependence is logarithmic in $h(\x^{(0)})$ -- instead of linear in~$h(\x^{(0)})$ as given by our previous convergence Theorem~\ref{thm:primalGreedyProduct} for the fixed step-size variant (in the fixed step-size variant, we would need $k_0 = \left\lceil 2n \frac{h(\x^{(0)})}{\CfTotal} \right\rceil$ steps to guarantee $h_{k_0} \leq \CfTotal / \mapprox$). Therefore, for the line-search variant of our Algorithm~\ref{alg:FW_product_again}, we have obtained guaranteed $\varepsilon$-small error after
\[
\left\lceil \frac{n}{\mapprox} \log\left( \frac{2 \mapprox h(\x^{(0)})}{\CfTotal} \right) \right\rceil + \left\lceil \frac{2n\CfTotal}{\mapprox^2 \, \varepsilon} \right\rceil
\]
iterations.

\paragraph{Effect of Line-Search.} It is also interesting to point out that even though we were using the optimal step-size in the second phase of the above proof (which yielded the recurrence~\eqref{eqn:fastRecurrence}), the second phase bound is not better than what we could have obtained by using a fixed step-size schedule of $\frac{2n}{\mapprox (k-k_0)+2n}$ and following the same induction proof line as in the previous Theorem~\ref{thm:primalGreedyProduct} (using the base case $h_{k_0} \leq \CfTotal / \mapprox$ and so we could let $C := \mapprox^{-1}\CfTotal$). This thus means that the advantage of the line-search over the fixed step-size schedule only appears in knowing when to switch from a step-size of $1$ (in the first phase, when $h_k \geq \mapprox^{-1}\CfTotal$) to a step-size of $\frac{2n}{\mapprox(k-k_0)+2n}$ (in the second phase), which unless we know the value of $f(\x^*)$, we cannot know in general. In the standard Frank-Wolfe case where $n=1$ and $\mapprox=1$, there is no difference in the rates for line-search or fixed step-size schedule as in this case we know $h_1 \leq \CfTotal$ as explained at the end of the proof of Theorem~\ref{thm:primalGreedyProduct}. This also suggests that if $k_0 > n$, it might be more worthwhile in theory to first do one batch Frank-Wolfe step to ensure that $h_1 \leq \CfTotal$, and then proceed with the block-coordinate Frank-Wolfe algorithm afterwards.

\subsubsection{Improved Primal-Dual Convergence for Line-Search}
Using the improved primal convergence theorem for line-search, we can also get a better rate for the expected duality gap (getting rid of the dependence of $h_0$ in the constant $C$):
\begin{theorem}[Improved Primal-Dual Convergence for Line-Search]\label{thm:primalDualFaster}
Let $k_0$ be defined as in Theorem~\ref{thm:fast_warmup}. For each $K\ge 5 k_0$, the line-search variant of Algorithm~\ref{alg:FW_product_again} will yield at least one iterate $\x^{(\hat k)}$ with $\hat k\le K$ with expected duality gap bounded by
\[
\E\big[g(\x^{(\hat k)})\big] \le \beta \frac{2 n}{\mapprox (K+2)} C  \ ,
\]
where $\beta=3$ and $C = \mapprox^{-1}\CfTotal(1+\delta)$. $\delta \geq 0$ and $0 < \mapprox \leq 1$ are the approximation parameters as defined in~(\ref{eqn:qualityBoth}) %
 -- use $\delta=0$ and $\mapprox=1$ for the exact variant. 

Moreover, if the duality gap $g$ is a convex function of $\x$, then the above bound also holds for $\E\big[g(\bar{\x}_{0.5}^{(K)})\big]$ for each $K\ge 5 k_0$, where $\bar{\x}_{0.5}^{(K)}$ is the $0.5$-suffix average of the iterates as defined in~(\ref{eqn:tavg}) with $\mu=0.5$.
\end{theorem}

\begin{proof}
We follow a similar argument as in the proof of Theorem~\ref{thm:primalDualGreedy1Regime}, but making use of the better primal convergence Theorem~\ref{thm:fast_warmup} as well as using the $0.5$-suffix average for the master inequality~(\ref{eqn:convex_gap_combo}). Let $K \geq 5 k_0$ be given. Let $\stepsize_k := \frac{2n}{\mapprox (k-k_0) + 2n}$ for $k \geq k_0$. Note then that $\addFactor_k = \frac{2n}{\mapprox k + 2n} \leq \stepsize_k$ and so the gap inequality~\eqref{eqn:gk_inequality} appearing in the proof of Theorem~\ref{thm:primalDualGreedy1Regime} is valid for this $\stepsize_k$ (because we are considering the line-search variant of Algorithm~\ref{alg:FW_product_again}, we are free to choose any $\stepsize \in [0,1]$ in Lemma~\ref{lem:stepProduct}). This means that the master inequality~\eqref{eqn:convex_gap_combo} is also valid here with $C = \mapprox^{-1}\CfTotal(1+\delta)$.

We consider the weights which appear in the definition of the $0.5$-suffix average of iterates $\bar{\x}_{0.5}^{(K)}$ given in~(\ref{eqn:tavg}), i.e. the average of the iterates $\x^{(k)}$ from $k = K_s := \ceil{0.5 K}$ to $k=K$. We thus have $\rho_k = 1/S_K$ for $K_s \leq k \leq K$ and $\rho_k = 0$ otherwise, where $S_K = K - \ceil{0.5 K} + 1$. Notice that $K_s \geq k_0$ by assumption.

With these choices of $\rho_k$ and $\stepsize_k$, %
the master inequality~\eqref{eqn:convex_gap_combo} becomes
\begin{align} 
	\sum_{k=0}^K \rho_k g_k &\leq 
	\frac{n}{\mapprox S_K} \left[\frac{h_{K_s}}{\stepsize_{K_s}} + \sum_{k=K_s}^{K-1} h_{k+1}  
		  \left(\frac{1}{\stepsize_{k+1}} -  \frac{1}{\stepsize_k} \right) \right] 
		  + \sum_{k=K_s}^K \stepsize_k \frac{C}{2 S_K} \nonumber \\ 
	&\leq \frac{n}{\mapprox S_K} \left[C + \sum_{k=K_s}^{K-1} \frac{2nC}{\mapprox (k+1-k_0) + 2n} (\frac{\mapprox}{2n} )  
			   \right] 
			  + \sum_{k=K_s}^K \frac{2n}{\mapprox (k-k_0) + 2n} \frac{C}{2 S_K} \nonumber \\
	&= \frac{nC}{\mapprox S_K} \left[ 1 + \sum_{k=K_s}^{K-1} \frac{1}{k+1-k_0 + 2n/\mapprox} + \sum_{k=K_s}^K \frac{1}{k-k_0 + 2n/\mapprox} \right] \nonumber \\
	&\leq \frac{nC}{\mapprox S_K} \left[ 1 + 2 \sum_{k=K_s}^{K} \frac{1}{k-k_0 + 2n/\mapprox}\right]  \nonumber \\
	&\leq \frac{2nC}{\mapprox (K+2)} \left[ 1 + 2 \sum_{k=K_s}^{K} \frac{1}{k-k_0 + 2n/\mapprox}\right] , \label{eqn:tail_avg_master}
\end{align}
where in the second line we used the faster convergence rate $h_k \leq \frac{2nC}{\mapprox (k-k_0) + 2n}$ from Theorem~\ref{thm:fast_warmup}, given that $K_s \geq k_0$. In the last line, we used $S_K \leq 0.5K+1$. The rest of the proof simply amounts to get an upper bound of $\beta=3$ on the term between brackets in~\eqref{eqn:tail_avg_master}, thus concluding that $\sum_{k=0}^K \rho_k g_k \leq \beta \frac{2nC}{\mapprox(K+2)}$. Then following a similar argument as in Theorem~\ref{thm:primalDualGreedy1Regime}, this will imply that there exists some $g_{\hat k}$ similarly upper bounded (the existence part of the theorem); and that if $g$ is convex, we have that $\E\big[g(\bar{\x}_{0.5}^{(K)})\big]$ is also similarly upper bounded.

We can upper bound the summand term in~\eqref{eqn:tail_avg_master} by using the fact that for any non-negative decreasing integrable function~$f$, we  have $\sum_{k = K_{s}}^K f(k) \leq \int_{ K_{s}-1}^{K} f(t) dt$. Let $a_n := k_0-2n/\mapprox$. Using $f(k) := \frac{1}{k-a_n}$, we have that
\begin{align*}
\sum_{k = K_{s}}^K \frac{1}{k-a_n} \leq& \int_{ K_{s}-1}^{K} \frac{1}{t-a_n} dt = \big[\log(t-a_n) \big]_{t=K_{s}-1}^{t=K} \\
 =&  \log \frac{K -a_n}{K_s-1-a_n} \leq \log \frac{K - a_n}{0.5 K - 1 - a_n} =: b(K),
\end{align*}
where we used $K_s \geq 0.5 K$. We want to show that $b(K) \leq 1$ for $K \geq 5 k_0$ to conclude that $\beta = 3$ works as a bound in~\eqref{eqn:tail_avg_master} and thus completing the proof. By looking at the sign of the derivative of $b(K)$, we can see that it is an increasing function of $K$ if $a_n \leq -2$ i.e. if $2n/\mapprox \geq k_0 + 2$ (which is always the case if $k_0 = 0$ as $n \geq 1$), and a strictly decreasing function of $K$ otherwise. In the case where $b(K)$ is increasing, we have $b(K) \leq \lim_{K \mapsto \infty} b(K) = \log(2) < 1$. In the case where $b(K)$ is decreasing, we upper bound it by letting $K$ take its minimal value from the theorem, namely $K \geq 5 k_0$. From the definition of $a_n$, we then get that $
b(5 k_0) = \log \frac{ 4 k_0 + 2n/\mapprox}{ 1.5 k_0 - 1 + 2n / \mapprox}
$, which is an increasing function of $k_0$ as long as $2n/\mapprox \geq 2$ (which is indeed always the case). So letting $k_0\rightarrow \infty$, we get that $b(5 k_0) \leq \log(4/1.5) \approx 0.98 < 1$, thus completing the proof.

We finally note that statement for $\E\big[g(\bar{\x}_{0.5}^{(K)})\big]$ in Theorem~\ref{thm:primalDualGreedy1Regime} can be proven using the same argument as above, but with $k_0 = 0$ and $C = \mapprox^{-1}\CfTotal(1+\delta) + h_0$ and using the original primal convergence bound on $h_k$ in Theorem~\ref{thm:primalGreedyProduct} instead. This will work for both predefined step-size or the line search variants --- the only place where we used the line-search in the above proof was to use the different primal convergence result as well as shifted-by-$k_0$ step-sizes $\stepsize_k$ (which reduce to the standard step-sizes when $k_0$ = 0).
\end{proof}

We note that we cannot fully get rid of the dependence on $h_0$ for the convergence rate of the expected duality gap of the weighted averaged scheme because we average over $k < k_0$, a regime where the primal error depends on $h_0$. With a more refined analysis for the weighted average with line-search scheme though, we note that one can replace the $h_0 \frac{n}{K}$ dependence in the bound with a $h_0 (\frac{n}{K})^2$ one, i.e. a quadratic speed-up to forget the initial conditions when line-search is used.

We also note that a bound of $O(1/K)$ can be derived similarly for $\E\big[g(\bar{\x}_{\mu}^{(K)})\big]$ for $0 < \mu < 1$ --- namely using the $C$ as in Theorem~\ref{thm:primalDualGreedy1Regime} and $\beta = \beta_\mu := (1-\mu)^{-1} (0.5 - \log \mu)$ (notice that $\beta_\mu = \infty$ if $\mu = 0$ or $\mu = 1$). This result is similar as the one for the stochastic subgradient method and where the $O(1/K)$ rate was derived by~\citet{Rakhlin2012} for the $(1-\mu)$-suffix averaging scheme --- this provided a motivation for the scheme as the authors proved that the full averaging scheme has $\Omega((\log K)/K)$ rate in the worst case. If we use $\mu=0$ (i.e. we average from the beginning), then the sum in~\eqref{eqn:tail_avg_master} becomes $O(\log K)$, yielding $O((\log K)/K)$ for the expected gap. 
\section{Equivalence of the `Linearization'-Duality Gap to a Special Case of Fenchel Duality} \label{sec:Fenchel_gap_equivalence}
For our used constrained optimization framework, the notion of the simple duality gap was crucial.
Consider a general constrained optimization problem
\begin{equation}\label{eq:optConstr}
\min_{\x\in\domain} f(\x) \ ,
\end{equation}
where the domain (or feasible set) $\domain\subseteq\X$ is an arbitrary compact subset of a Euclidean space~$\X$. We assume that the objective function $f$ is convex, but not necessarily differentiable.

In this case, the general `linearization' duality gap (\ref{eq:duality_gap})
as proposed by \citep{Jaggi:2013wg} is given by
\begin{equation}\label{eq:lin_gap_fenchel}
g(\x;d_\x) = \id_\domain^*(-d_\x) + \langle \x, d_\x \rangle \ .
\end{equation}
Here $d_\x$ is an arbitrary subgradient to $f$ at the candidate position $\x$,
and $\id_\domain^*(\y) := \sup_{\sv\in\domain}\, \langle \sv,\y\rangle$ is the \emph{support function} of the set $\domain$.

Convexity of $f$ implies that the linearization $f(\x) + \big\langle \sv - \x, d_\x \big\rangle$ always lies below the graph of the function $f$, as illustrated by the figure in Section \ref{sec:FW}. This immediately gives the crucial property of the duality gap (\ref{eq:lin_gap_fenchel}), as being a \emph{certificate} for the current approximation quality,
i.e. upper-bounding the (unknown) error
$g(\x)\ge f(\x)-f(\x^*)$, where $\x^*$ is some optimal solution.

Note that for differentiable functions $f$, the gradient is the unique subgradient at $\x$, therefore the duality gap equals $g(\x) := g(\x;\nabla f(\x))$ as we defined in (\ref{eq:duality_gap}).

\paragraph{Fenchel Duality.}
Here we will additionally explain how the duality gap (\ref{eq:lin_gap_fenchel}) can also be interpreted as a special case of standard Fenchel convex duality.

We consider the equivalent formulation of our constrained problem (\ref{eq:optConstr}), given by
\[
\min_{\substack{\x\in\X}} \ f(\x) + \id_{\domain}(\x) \  .
\]
Here the \emph{set indicator function} $\id_\domain$ of a subset $\domain\subseteq\X$ is defined as $\id_\domain(\x):=0$ for $\x\in \domain$ and $\id_\domain(\x):=+\infty$ for $\x\notin \domain$.

The \emph{Fenchel conjugate} function $f^*$ of a function $f$ is given by $f^*(\y) := \sup_{\x\in\X} \langle \x, \y \rangle - f(\x)$.

For example, observe that the Fenchel conjugate of a set indicator function $\id_\domain(.)$ is given by its support function~$\id_\domain^*(.)$.

From the above definition of the conjugate, the \emph{Fenchel-Young inequality} $f(\x)+f^*(\y)\geq\langle \x,\y\rangle$ $\forall \x,\y\in\X$ follows directly.

Now we consider the \emph{Fenchel dual problem} of minimizing $p(\x) := f(\x) + \id_{\domain}(\x)$, which is defined as to maximize $d(\y) := -f^*(\y) - \id_\domain^*(-\y)$.
By the Fenchel-Young inequality, and assuming that $\x\in\domain$, we have that $\forall \y\in\X$,
\begin{eqnarray*}
p(\x) - d(\y) &=& f(\x) - (-f^*(\y) - \id_\domain^*(-\y))\\
 &\geq& \langle \x,\y\rangle + \id_\domain^*(-\y)\\
 &=& g(\x;\y) \ .
\end{eqnarray*}
Furthermore, this inequality becomes an equality if and only if $\y$ is chosen as a subgradient to~$f$ at $\x$, that is if $\y:=-d_\x$. The last fact follows from the known equivalent characterization of the subdifferential in terms of the Fenchel conjugate: $\partial f(\x) := \SetOf{\y\in\X}{f(\x)+f^*(\y)=\langle \x,\y\rangle}$.
For a more detailed explanation of Fenchel duality, we refer the reader to the
standard literature, e.g. \citepsup[Theorem 3.3.5]{Borwein:2006ts}.

To summarize, we have obtained that the simpler `linearization' duality gap $g(\x;d_\x)$ as given in \eqref{eq:lin_gap_fenchel} is indeed the difference of the current objective to the Fenchel dual problem, when being restricted to the particular choice of the dual variable $\y$ being a subgradient at the current position~$\x$.

\section{Derivation of the n-Slack Structural SVM Dual}\label{sec:app_duals}
\begin{proof}[Proof of the dual of the $n$-Slack-Formulation]
See also~\citet{Collins2008}.
For a self-contained explanation of Lagrange duality we refer the reader to
\citetsup[Section 5]{Boyd:2004uz}.
The Lagrangian of (\ref{eq:svmstruct_nslack_primal}) is
\[
L(\weightv,\bm{\xi},\dualvarv)
  = \frac\lambda2 \langle\weightv,\weightv\rangle+\frac1n\sum_{i=1}^n \xi_i
  + \sum_{i\in[n],\,\outputvarv \in \outputdomain_i}
  \frac1n\dualvar_i(\outputvarv)
  \left( -\xi_i +
     \langle \weightv, -\featuremapdiffv_i(\outputvarv) \rangle
     + \errorterm_i(\outputvarv)
  \right) ,
\]
where $\dualvarv = (\dualvarv_1,\dots,\dualvarv_n) \in\R^{|\outputdomain_1|}\times\dots\times\R^{|\outputdomain_n|} = \R^{m}$ are the corresponding (non-negative) Lagrange multipliers.
Here we have re-scaled the multipliers (dual variables) by a constant of $\frac1n$, corresponding to multiplying the corresponding original primal constraint by $\frac1n$ on both sides, which does not change the optimization problem. %

Since the objective as well as the constraints are continuously differentiable with respect to $(\weightv,\xi)$, the Lagrangian $L$ will attain its finite minimum over $\dualvarv$ when $\nabla_{\!(\weightv,\xi)} L(\weightv,\xi,\dualvarv) = 0$. Making this saddle-point condition explicit results in a simplified Lagrange dual problem, which is also known as the Wolfe dual.
In our case, this condition from differentiating w.r.t. $\weightv$ is
\begin{equation}\label{eq:saddle_n_slack_again}
\lambda \weightv  = \sum_{i\in[n],\,\outputvarv \in \outputdomain_i}
  \frac1n\dualvar_i(\outputvarv)
    \featuremapdiffv_i(\outputvarv) \ .
\end{equation}
And differentiating with respect to $\xi_i$ and setting the derivatives to zero gives\footnote{Note that because the Lagrangian is linear in $\xi_i$, if this condition is not satisfied, the minimization of the Lagrangian in $\xi_i$ yield $-\infty$ and so these points can be excluded.}
\[
\sum_{\outputvarv \in \outputdomain_i}
  \dualvar_i(\outputvarv) = 1 ~~~\forall i\in[n] \ .
\]
Plugging this condition and the expression (\ref{eq:saddle_n_slack_again}) for $\weightv$ back into the Lagrangian, we obtain the Lagrange dual problem
\begin{align}
    \max_{\dualvarv} \quad &
    -\frac{\lambda}{2}
    \norm{
    \sum_{i\in[n],\,\outputvarv \in \outputdomain_i}
  \dualvar_i(\outputvarv)
    \frac{\featuremapdiffv_i(\outputvarv)}{\lambda n}
    }^2
    + \sum_{i\in[n],\,\outputvarv \in \outputdomain_i} \dualvar_i(\outputvarv)
       \frac{\errorterm_i(\outputvarv)}{n}
    \notag\\
    \text{s.t.} \quad &
      \sum_{\outputvarv \in \outputdomain}  \dualvar_i(\outputvarv) = 1 ~~~\forall i\in[n],
    \notag\\
     &
      \text{ and }\  \dualvar_i(\outputvarv)  \ge 0 ~~~\forall i\in[n],\,\forall\outputvarv \in \outputdomain_i \ ,
      \notag
\end{align}
which is exactly the negative of the quadratic program claimed in (\ref{eq:svmstruct_nslack_dual}).
\end{proof}

\section{Additional Experiments}
\label{sec:additional_experiments}

Complementing the results presented in Figure~\ref{fig:results} in Section~\ref{sec:experiments} of the main paper, here we provide additional experimental results as well as give more information about the experimental setup used.

For the Frank-Wolfe methods, Figure~\ref{fig:fw_stepsize_results} presents results on OCR comparing setting the step-size by line-search against the simpler predefined step-size scheme of $\stepsize_k = 2n/(k+2n)$. There, \emph{BCFW} with predefined step-sizes does similarly as \emph{SSG}, indicating that most of the improvement of \emph{BCFW} with line-search over \emph{SSG} is coming from the optimal step-size choice (and not from the Frank-Wolfe formulation on the dual). We also see that \emph{BCFW} with predefined step-sizes can even do worse than batch Frank-Wolfe with line-search in the early iterations for small values of $\lambda$.
 
Figure~\ref{fig:ocr2_results} and Figure~\ref{fig:conll_results} show additional results of the stochastic solvers for several values of $\lambda$ on the OCR and CoNLL datasets. Here we also include the (uniformly) averaged stochastic subgradient method (\emph{SSG-avg}), which starts averaging at the beginning; as well as the $0.5$-suffix averaging versions of both \emph{SSG} and \emph{BCFW} (\emph{SSG-tavg} and \emph{BCFW-tavg} respectively), implemented using the `doubling trick'  as described just after Equation~\eqref{eqn:tavg} in Appendix~\ref{sec:app_convergence_proof}. The `doubling trick' uniformly averages all iterates since the last iteration which was a power of~2, and was described by~\citet{Rakhlin2012}, with experiments for \emph{SSG} in~\citet{LacosteJulien:2012uo}. In our experiments, \emph{BCFW-tavg} sometimes slightly outperforms the weighted average scheme \emph{BCFW-wavg}, but its performance fluctuates more widely, which is why we recommend the \emph{BCFW-wavg}, as mentioned in the main text. 
In our experiments, the objective value of \emph{SSG-avg} is always worse than the other stochastic methods (apart \emph{online-EG}), which is why it was excluded from the main text. \emph{Online-EG} performed substantially worse than the other stochastic solvers for the OCR dataset, and is therefore not included in the comparison for the other datasets.\footnote{The worse performance of the online exponentiated gradient method could be explained by the fact that it uses a log-parameterization of the dual variables and so its iterates are forced to be in the \emph{interior} of the probability simplex, whereas we know that the optimal solution for the structural SVM objective lies at the \emph{boundary} of the domain and thus these parameters need to go to infinity.}

Finally, Figure~\ref{fig:matching_results} presents additional results for the matching application from~\citet{Taskar06extrag}.

\begin{figure}[htb]
    \begin{subfigure}[t]{0.32\linewidth}
        \centering
        \def\xlabel{effective passes}
        \def\xmin{1}
        \def\xmax{149}
        \def\ymin{1e-2}
        \def\ymax{1e+3}
        \def\xmode{normal}
        \def\showlegend{0}
        \def\legendpos{north east}
        \def\experimentprefix{include/data/dataset=ocr2_lambda=0.010000}
        \small
\begin{tikzpicture}[scale=0.63]

\begin{axis}[
xlabel=\xlabel,
ylabel=primal suboptimality for problem \eqref{eq:svmstruct_nslack_primal},
xmin=\xmin,
xmax=\xmax,
ymin=\ymin,
ymax=\ymax,
enlargelimits=false, area style,
ymode=log,
xmode=\xmode,
line legend,
legend pos=\legendpos,
]


\addplot[fill=blue,draw=none,forget plot,opacity=0.2] table[x index=0,y
index=1, header=true, col sep=comma]
{\experimentprefix/product-LS_confidence.txt};

\addplot [
color=blue,
solid,
style=thick,
mark=triangle,
mark repeat=20,
]
table[x index=0,y index=2, header=true, col sep=comma]
{\experimentprefix/product-LS.txt};
\ifnum \showlegend=1
{
\addlegendentry{BCFW (line-search)}
}
\fi

\addplot[fill=black,draw=none,forget plot,opacity=0.2] table[x index=0,y
index=1, header=true, col sep=comma]
{\experimentprefix/product_confidence.txt};

\addplot [
color=black,
densely dotted,
style=thick,
mark=pentagon*,
mark repeat=20,
]
table[x index=0,y index=2, header=true, col sep=comma]
{\experimentprefix/product.txt};
\ifnum \showlegend=1
{
\addlegendentry{BCFW (predef. $\gamma$)}
}
\fi

\addplot[fill=purple,draw=none,forget plot,opacity=0.2] table[x index=0,y
index=1, header=true, col sep=comma]
{\experimentprefix/frankWolfe-LS_confidence.txt};

\addplot [
color=orange,
densely dashed,
style=thick,
mark=triangle,
mark repeat=20,
mark options=solid
]
table[x index=0,y index=1, header=true, col sep=comma]
{\experimentprefix/frankWolfe-LS.txt};
\ifnum \showlegend=1
{
\addlegendentry{FW (line-search)}
}
\fi
\addplot[fill=brown,draw=none,forget plot,opacity=0.2] table[x index=0,y
index=1, header=true, col sep=comma]
{\experimentprefix/frankWolfe_confidence.txt};

\addplot [
color=violet,
densely dotted,
style=thick,
mark=pentagon*,
mark repeat=20,
mark options=solid
]
table[x index=0,y index=1, header=true, col sep=comma]
{\experimentprefix/frankWolfe.txt};
\ifnum \showlegend=1
{
\addlegendentry{FW (predef. $\gamma$)}
}
\fi


\addplot[fill=green,draw=none,forget plot,opacity=0.2] table[x index=0,y
index=1, header=true, col sep=comma]
{\experimentprefix/pegasos_confidence.txt};

\addplot [
color=green,
solid,
style=thick,
mark=o,
mark repeat=20,
]
table[x index=0,y index=2, header=true, col sep=comma]
{\experimentprefix/pegasos.txt};
\ifnum \showlegend=1
{
\addlegendentry{SSG}
}
\fi

\ifdefined \showcuttingplanes
{

\addplot [
color=black,
densely dashed,
style=thick,
mark=square,
mark repeat=20,
mark options=solid
]
table[x index=0,y index=1, header=false, col sep=comma]
{\experimentprefix/svmStruct_lowerenvelope.txt};
\ifnum \showlegend=1
{
\addlegendentry{cutting plane}
}
\fi
}
\fi

\end{axis}

\end{tikzpicture}
\normalsize
        \caption{$\lambda=0.01$.}
    \end{subfigure}
    \begin{subfigure}[t]{0.32\linewidth}
        \centering
        \def\xlabel{effective passes}
        \def\xmin{1}
        \def\xmax{149}
        \def\ymin{1e-2}
        \def\ymax{1e+3}
        \def\xmode{normal}
        \def\showlegend{0}
        \def\legendpos{north east}
        \def\experimentprefix{include/data/dataset=ocr2_lambda=0.001000}
        \small
\begin{tikzpicture}[scale=0.63]

\begin{axis}[
xlabel=\xlabel,
ylabel=primal suboptimality for problem \eqref{eq:svmstruct_nslack_primal},
xmin=\xmin,
xmax=\xmax,
ymin=\ymin,
ymax=\ymax,
enlargelimits=false, area style,
ymode=log,
xmode=\xmode,
line legend,
legend pos=\legendpos,
]


\addplot[fill=blue,draw=none,forget plot,opacity=0.2] table[x index=0,y
index=1, header=true, col sep=comma]
{\experimentprefix/product-LS_confidence.txt};

\addplot [
color=blue,
solid,
style=thick,
mark=triangle,
mark repeat=20,
]
table[x index=0,y index=2, header=true, col sep=comma]
{\experimentprefix/product-LS.txt};
\ifnum \showlegend=1
{
\addlegendentry{BCFW (line-search)}
}
\fi

\addplot[fill=black,draw=none,forget plot,opacity=0.2] table[x index=0,y
index=1, header=true, col sep=comma]
{\experimentprefix/product_confidence.txt};

\addplot [
color=black,
densely dotted,
style=thick,
mark=pentagon*,
mark repeat=20,
]
table[x index=0,y index=2, header=true, col sep=comma]
{\experimentprefix/product.txt};
\ifnum \showlegend=1
{
\addlegendentry{BCFW (predef. $\gamma$)}
}
\fi

\addplot[fill=purple,draw=none,forget plot,opacity=0.2] table[x index=0,y
index=1, header=true, col sep=comma]
{\experimentprefix/frankWolfe-LS_confidence.txt};

\addplot [
color=orange,
densely dashed,
style=thick,
mark=triangle,
mark repeat=20,
mark options=solid
]
table[x index=0,y index=1, header=true, col sep=comma]
{\experimentprefix/frankWolfe-LS.txt};
\ifnum \showlegend=1
{
\addlegendentry{FW (line-search)}
}
\fi
\addplot[fill=brown,draw=none,forget plot,opacity=0.2] table[x index=0,y
index=1, header=true, col sep=comma]
{\experimentprefix/frankWolfe_confidence.txt};

\addplot [
color=violet,
densely dotted,
style=thick,
mark=pentagon*,
mark repeat=20,
mark options=solid
]
table[x index=0,y index=1, header=true, col sep=comma]
{\experimentprefix/frankWolfe.txt};
\ifnum \showlegend=1
{
\addlegendentry{FW (predef. $\gamma$)}
}
\fi


\addplot[fill=green,draw=none,forget plot,opacity=0.2] table[x index=0,y
index=1, header=true, col sep=comma]
{\experimentprefix/pegasos_confidence.txt};

\addplot [
color=green,
solid,
style=thick,
mark=o,
mark repeat=20,
]
table[x index=0,y index=2, header=true, col sep=comma]
{\experimentprefix/pegasos.txt};
\ifnum \showlegend=1
{
\addlegendentry{SSG}
}
\fi

\ifdefined \showcuttingplanes
{

\addplot [
color=black,
densely dashed,
style=thick,
mark=square,
mark repeat=20,
mark options=solid
]
table[x index=0,y index=1, header=false, col sep=comma]
{\experimentprefix/svmStruct_lowerenvelope.txt};
\ifnum \showlegend=1
{
\addlegendentry{cutting plane}
}
\fi
}
\fi

\end{axis}

\end{tikzpicture}
\normalsize
        \caption{$\lambda=0.001$.}
    \end{subfigure}
    \begin{subfigure}[t]{0.32\linewidth}
        \centering
        \def\xlabel{effective passes}
        \def\xmin{1}
        \def\xmax{149}
        \def\ymin{1e-2}
        \def\ymax{1e+3}
        \def\xmode{normal}
        \def\showlegend{1}
        \def\legendpos{north east}
        \def\experimentprefix{include/data/dataset=ocr2_lambda=0.000160}
        \small
\begin{tikzpicture}[scale=0.63]

\begin{axis}[
xlabel=\xlabel,
ylabel=primal suboptimality for problem \eqref{eq:svmstruct_nslack_primal},
xmin=\xmin,
xmax=\xmax,
ymin=\ymin,
ymax=\ymax,
enlargelimits=false, area style,
ymode=log,
xmode=\xmode,
line legend,
legend pos=\legendpos,
]


\addplot[fill=blue,draw=none,forget plot,opacity=0.2] table[x index=0,y
index=1, header=true, col sep=comma]
{\experimentprefix/product-LS_confidence.txt};

\addplot [
color=blue,
solid,
style=thick,
mark=triangle,
mark repeat=20,
]
table[x index=0,y index=2, header=true, col sep=comma]
{\experimentprefix/product-LS.txt};
\ifnum \showlegend=1
{
\addlegendentry{BCFW (line-search)}
}
\fi

\addplot[fill=black,draw=none,forget plot,opacity=0.2] table[x index=0,y
index=1, header=true, col sep=comma]
{\experimentprefix/product_confidence.txt};

\addplot [
color=black,
densely dotted,
style=thick,
mark=pentagon*,
mark repeat=20,
]
table[x index=0,y index=2, header=true, col sep=comma]
{\experimentprefix/product.txt};
\ifnum \showlegend=1
{
\addlegendentry{BCFW (predef. $\gamma$)}
}
\fi

\addplot[fill=purple,draw=none,forget plot,opacity=0.2] table[x index=0,y
index=1, header=true, col sep=comma]
{\experimentprefix/frankWolfe-LS_confidence.txt};

\addplot [
color=orange,
densely dashed,
style=thick,
mark=triangle,
mark repeat=20,
mark options=solid
]
table[x index=0,y index=1, header=true, col sep=comma]
{\experimentprefix/frankWolfe-LS.txt};
\ifnum \showlegend=1
{
\addlegendentry{FW (line-search)}
}
\fi
\addplot[fill=brown,draw=none,forget plot,opacity=0.2] table[x index=0,y
index=1, header=true, col sep=comma]
{\experimentprefix/frankWolfe_confidence.txt};

\addplot [
color=violet,
densely dotted,
style=thick,
mark=pentagon*,
mark repeat=20,
mark options=solid
]
table[x index=0,y index=1, header=true, col sep=comma]
{\experimentprefix/frankWolfe.txt};
\ifnum \showlegend=1
{
\addlegendentry{FW (predef. $\gamma$)}
}
\fi


\addplot[fill=green,draw=none,forget plot,opacity=0.2] table[x index=0,y
index=1, header=true, col sep=comma]
{\experimentprefix/pegasos_confidence.txt};

\addplot [
color=green,
solid,
style=thick,
mark=o,
mark repeat=20,
]
table[x index=0,y index=2, header=true, col sep=comma]
{\experimentprefix/pegasos.txt};
\ifnum \showlegend=1
{
\addlegendentry{SSG}
}
\fi

\ifdefined \showcuttingplanes
{

\addplot [
color=black,
densely dashed,
style=thick,
mark=square,
mark repeat=20,
mark options=solid
]
table[x index=0,y index=1, header=false, col sep=comma]
{\experimentprefix/svmStruct_lowerenvelope.txt};
\ifnum \showlegend=1
{
\addlegendentry{cutting plane}
}
\fi
}
\fi

\end{axis}

\end{tikzpicture}
\normalsize
        \caption{$\lambda=1/n=0.00016$.}
    \end{subfigure}
    \caption{Convergence of the Frank-Wolfe algorithms on the OCR dataset, depending on the choice of the step-size. We compare the line-search variants as used in Algorithms \ref{alg:FW_SVM} and \ref{alg:FW_product_SVM}, versus the simpler predefined step-sizes $\stepsize := \frac2{k+2}$ (and $\stepsize := \frac{2n}{k+2n}$ in the block-coordinate case respectively).
    See also the original optimization Algorithms \ref{alg:FW} and \ref{alg:FW_product}.
    }
    \label{fig:fw_stepsize_results}
\end{figure}
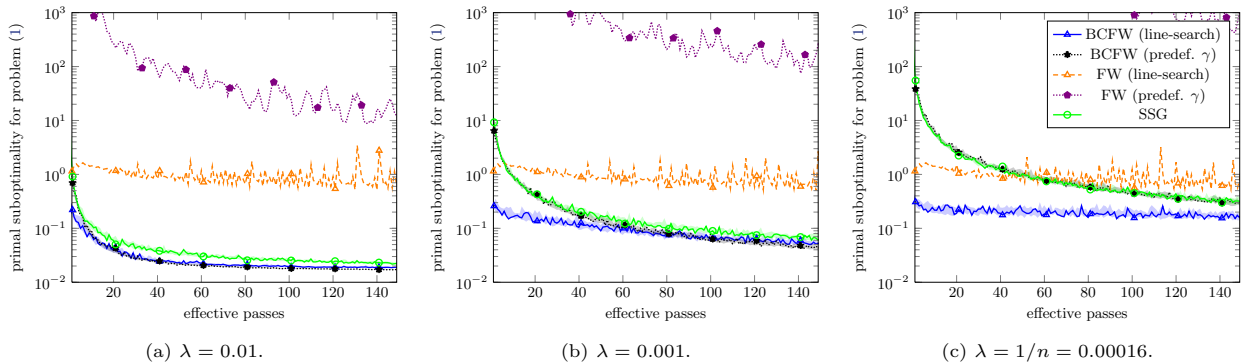

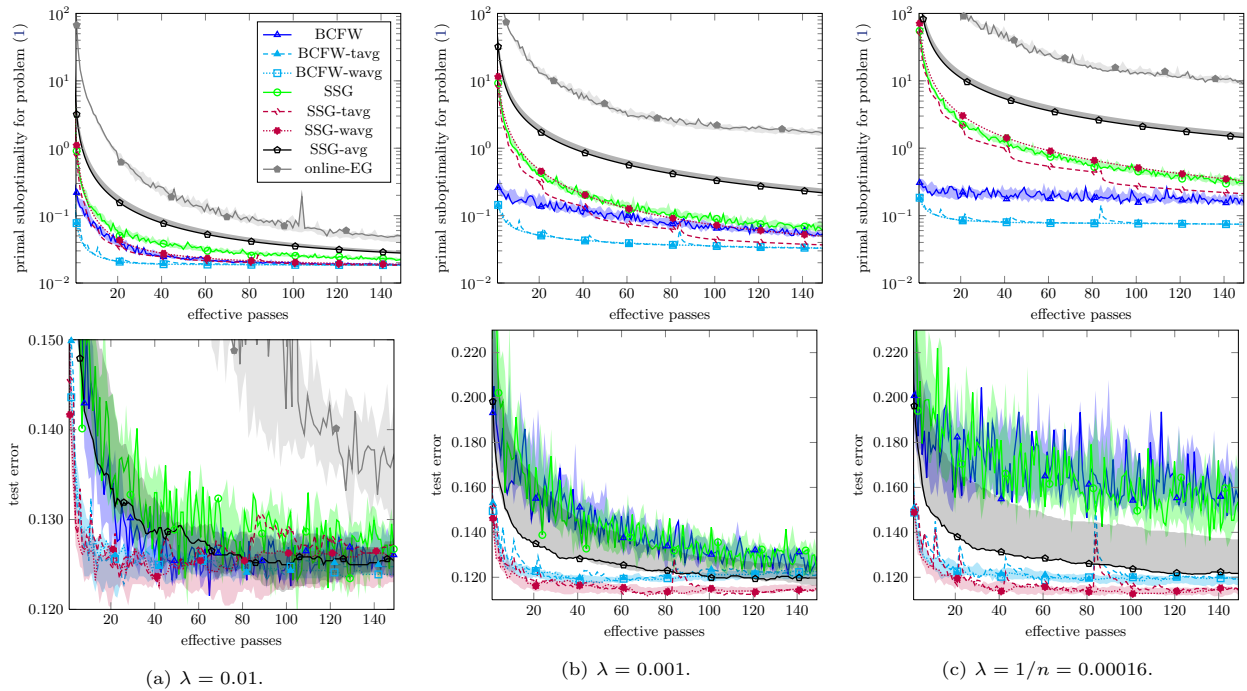
\begin{figure}[htb]
    \begin{subfigure}[t]{0.32\linewidth}
        \centering
        \def\xlabel{effective passes}
        \def\xmin{1}
        \def\xmax{149}
        \def\ymin{0.01}
        \def\ymax{100}
        \def\xmode{normal}
        \def\showlegend{1}
        \def\showeg{1}
        \def\showtavg{1}
        \def\legendpos{north east}
        \def\experimentprefix{include/data/dataset=ocr2_lambda=0.010000}
        \small
\begin{tikzpicture}[scale=0.63]

\begin{axis}[
xlabel=\xlabel,
ylabel=primal suboptimality for problem \eqref{eq:svmstruct_nslack_primal},
xmin=\xmin,
xmax=\xmax,
ymin=\ymin,
ymax=\ymax,
enlargelimits=false, area style,
ymode=log,
xmode=\xmode,
line legend,
legend pos=\legendpos,
]

\ifdefined \showsfwwithoutline
{

\addplot[fill=red,draw=none,forget plot,opacity=0.3] table[x index=0,y
index=1, header=true, col sep=comma]
{\experimentprefix/product_confidence.txt};

\addplot [
color=red,
solid,
style=thick,
mark=square,
mark repeat=20,
]
table[x index=0,y index=2, header=true, col sep=comma]
{\experimentprefix/product.txt};
\ifnum \showlegend=1
{
\addlegendentry{BCFW-fix}
}
\fi

}
\fi


\addplot[fill=blue,draw=none,forget plot,opacity=0.3] table[x index=0,y
index=1, header=true, col sep=comma]
{\experimentprefix/product-LS_confidence.txt};

\addplot [
color=blue,
solid,
style=thick,
mark=triangle,
mark repeat=20,
]
table[x index=0,y index=2, header=true, col sep=comma]
{\experimentprefix/product-LS.txt};
\ifnum \showlegend=1
{
\addlegendentry{BCFW}
}
\fi

\ifdefined \showtavg
{
\addplot [
color=cyan,
densely dashed,
style=thick,
mark=triangle*,
mark repeat=20,
mark options=solid
]
table[x index=0,y index=2, header=true, col sep=comma]
{\experimentprefix/product-LS-opt.txt};
\ifnum \showlegend=1
{
\addlegendentry{BCFW-tavg}
}
\fi
}
\fi

\addplot[fill=cyan,draw=none,forget plot,opacity=0.3] table[x index=0,y
index=1, header=true, col sep=comma]
{\experimentprefix/product-LS-wavg_confidence.txt};

\addplot [
color=cyan,
densely dotted,
style=thick,
mark=square,
mark repeat=20,
mark options=solid
]
table[x index=0,y index=2, header=true, col sep=comma]
{\experimentprefix/product-LS-wavg.txt};
\ifnum \showlegend=1
{
\addlegendentry{BCFW-wavg}
}
\fi


\addplot[fill=green,draw=none,forget plot,opacity=0.3] table[x index=0,y
index=1, header=true, col sep=comma]
{\experimentprefix/pegasos_confidence.txt};

\addplot [
color=green,
solid,
style=thick,
mark=o,
mark repeat=20,
]
table[x index=0,y index=2, header=true, col sep=comma]
{\experimentprefix/pegasos.txt};
\ifnum \showlegend=1
{
\addlegendentry{SSG}
}
\fi

\ifdefined \showtavg
{

\addplot [
color=purple,
densely dashed,
style=thick,
mark=diamond,
mark repeat=20,
]
table[x index=0,y index=2, header=true, col sep=comma]
{\experimentprefix/optimalSG.txt};
\ifnum \showlegend=1
{
\addlegendentry{SSG-tavg}
}
\fi
}
\fi


\addplot[fill=purple,draw=none,forget plot,opacity=0.2] table[x index=0,y
index=1, header=true, col sep=comma]
{\experimentprefix/pegasos-wavg_confidence.txt};

\addplot [
color=purple,
densely dotted,
style=thick,
mark=*,
mark repeat=20,
]
table[x index=0,y index=2, header=true, col sep=comma]
{\experimentprefix/pegasos-wavg.txt};
\ifnum \showlegend=1
{
\addlegendentry{SSG-wavg}
}
\fi


\addplot[fill=black,draw=none,forget plot,opacity=0.3] table[x index=0,y
index=1, header=true, col sep=comma]
{\experimentprefix/pegasos-avg_confidence.txt};

\addplot [
color=black,
solid,
style=thick,
mark=pentagon,
mark repeat=20,
]
table[x index=0,y index=2, header=true, col sep=comma]
{\experimentprefix/pegasos-avg.txt};
\ifnum \showlegend=1
{
\addlegendentry{SSG-avg}
}
\fi

\ifdefined \showeg
{

\addplot[fill=gray,draw=none,forget plot,opacity=0.2] table[x index=0,y
index=1, header=true, col sep=comma]
{\experimentprefix/OEG_confidence.txt};

\addplot [
color=gray,
solid,
style=thick,
mark=pentagon*,
mark repeat=20,
]
table[x index=0,y index=2, header=true, col sep=comma]
{\experimentprefix/OEG.txt};
\ifnum \showlegend=1
{
\addlegendentry{online-EG}
}
\fi
}
\fi

\end{axis}

\end{tikzpicture}
\normalsize
        \def\xlabel{effective passes}
        \def\showlegend{0}
        \def\showtavg{1}
        \def\xmin{1}
        \def\xmax{149}
        \def\ymin{0.12}
        \def\ymax{0.15}
        \def\xmode{normal}
        \def\experimentprefix{include/data/dataset=ocr2_lambda=0.010000}
        \small
\begin{tikzpicture}[scale=0.63]

\pgfplotsset{y tick label style={ 
         scaled ticks=false, 
         /pgf/number format/fixed zerofill, 
         /pgf/number format/fixed, 
         /pgf/number format/precision=3, 
     }
}

\begin{axis}[
xlabel=\xlabel,
ylabel=test error,
xmin=\xmin,
xmax=\xmax,
ymin=\ymin,
ymax=\ymax,
xmode=\xmode,
enlargelimits=false, area style,
line legend,
]

\ifdefined \showsfwwithoutline
{

\addplot[fill=red,draw=none,forget plot,opacity=0.3] table[x index=0,y
index=2, header=true, col sep=comma]
{\experimentprefix/product_confidence.txt};

\addplot [
color=red,
solid,
style=thick,
mark=square,
mark repeat=20,
]
table[x index=0,y index=3, header=true, col sep=comma]
{\experimentprefix/product.txt};
\ifnum \showlegend=1
{
\addlegendentry{BCFW-fix}
}
\fi

}
\fi


\addplot[fill=blue,draw=none,forget plot,opacity=0.3] table[x index=0,y
index=2, header=true, col sep=comma]
{\experimentprefix/product-LS_confidence.txt};

\addplot [
color=blue,
solid,
style=thick,
mark=none,
mark=triangle,
mark repeat=20,
]
table[x index=0,y index=3, header=true, col sep=comma]
{\experimentprefix/product-LS.txt};
\ifnum \showlegend=1
{
\addlegendentry{BCFW}
}
\fi

\ifdefined \showtavg
{

\addplot [
color=cyan,
densely dashed,
style=thick,
mark=triangle*,
mark repeat=20,
mark options=solid
]
table[x index=0,y index=3, header=true, col sep=comma]
{\experimentprefix/product-LS-opt.txt};
\ifnum \showlegend=1
{
\addlegendentry{BCFW-tavg}
}
\fi
}
\fi

\addplot[fill=cyan,draw=none,forget plot,opacity=0.3] table[x index=0,y
index=2, header=true, col sep=comma]
{\experimentprefix/product-LS-wavg_confidence.txt};

\addplot [
color=cyan,
densely dotted,
style=thick,
mark=square,
mark repeat=20,
mark options=solid
]
table[x index=0,y index=3, header=true, col sep=comma]
{\experimentprefix/product-LS-wavg.txt};
\ifnum \showlegend=1
{
\addlegendentry{BCFW-wavg}
}
\fi


\addplot[fill=green,draw=none,forget plot,opacity=0.3] table[x index=0,y
index=2, header=true, col sep=comma]
{\experimentprefix/pegasos_confidence.txt};

\addplot [
color=green,
solid,
style=thick,
mark=o,
mark repeat=20,
]
table[x index=0,y index=3, header=true, col sep=comma]
{\experimentprefix/pegasos.txt};
\ifnum \showlegend=1
{
\addlegendentry{SSG}
}
\fi

\ifdefined \showtavg
{
\addplot [
color=purple,
densely dashed,
style=thick,
mark=diamond,
mark repeat=20,
]
table[x index=0,y index=3, header=true, col sep=comma]
{\experimentprefix/optimalSG.txt};
\ifnum \showlegend=1
{
\addlegendentry{SSG-tavg}
}
\fi
}
\fi


\addplot[fill=purple,draw=none,forget plot,opacity=0.2] table[x index=0,y
index=2, header=true, col sep=comma]
{\experimentprefix/pegasos-wavg_confidence.txt};

\addplot [
color=purple,
densely dotted,
style=thick,
mark=*,
mark repeat=20,
]
table[x index=0,y index=3, header=true, col sep=comma]
{\experimentprefix/pegasos-wavg.txt};
\ifnum \showlegend=1
{
\addlegendentry{SSG-wavg}
}
\fi


\addplot[fill=black,draw=none,forget plot,opacity=0.2] table[x index=0,y
index=2, header=true, col sep=comma]
{\experimentprefix/pegasos-avg_confidence.txt};

\addplot [
color=black,
solid,
style=thick,
mark=pentagon,
mark repeat=20,
]
table[x index=0,y index=3, header=true, col sep=comma]
{\experimentprefix/pegasos-avg.txt};
\ifnum \showlegend=1
{
\addlegendentry{SSG-avg}
}
\fi

\ifdefined \showeg
{

\addplot[fill=gray,draw=none,forget plot,opacity=0.2] table[x index=0,y
index=2, header=true, col sep=comma]
{\experimentprefix/OEG_confidence.txt};

\addplot [
color=gray,
solid,
style=thick,
mark=pentagon*,
mark repeat=20,
]
table[x index=0,y index=3, header=true, col sep=comma]
{\experimentprefix/OEG.txt};
\ifnum \showlegend=1
{
\addlegendentry{online-EG}
}
\fi
}
\fi

\end{axis}
\end{tikzpicture}
\normalsize
        \caption{$\lambda=0.01$.}
    \end{subfigure}
    \begin{subfigure}[t]{0.32\linewidth}
        \centering
        \def\xlabel{effective passes}
        \def\xmin{1}
        \def\xmax{149}
        \def\ymin{0.01}
        \def\ymax{100}
        \def\xmode{normal}
        \def\showlegend{0}
        \def\showeg{1}
        \def\showtavg{1}
        \def\legendpos{north east}
        \def\experimentprefix{include/data/dataset=ocr2_lambda=0.001000}
        \small
\begin{tikzpicture}[scale=0.63]

\begin{axis}[
xlabel=\xlabel,
ylabel=primal suboptimality for problem \eqref{eq:svmstruct_nslack_primal},
xmin=\xmin,
xmax=\xmax,
ymin=\ymin,
ymax=\ymax,
enlargelimits=false, area style,
ymode=log,
xmode=\xmode,
line legend,
legend pos=\legendpos,
]

\ifdefined \showsfwwithoutline
{

\addplot[fill=red,draw=none,forget plot,opacity=0.3] table[x index=0,y
index=1, header=true, col sep=comma]
{\experimentprefix/product_confidence.txt};

\addplot [
color=red,
solid,
style=thick,
mark=square,
mark repeat=20,
]
table[x index=0,y index=2, header=true, col sep=comma]
{\experimentprefix/product.txt};
\ifnum \showlegend=1
{
\addlegendentry{BCFW-fix}
}
\fi

}
\fi


\addplot[fill=blue,draw=none,forget plot,opacity=0.3] table[x index=0,y
index=1, header=true, col sep=comma]
{\experimentprefix/product-LS_confidence.txt};

\addplot [
color=blue,
solid,
style=thick,
mark=triangle,
mark repeat=20,
]
table[x index=0,y index=2, header=true, col sep=comma]
{\experimentprefix/product-LS.txt};
\ifnum \showlegend=1
{
\addlegendentry{BCFW}
}
\fi

\ifdefined \showtavg
{
\addplot [
color=cyan,
densely dashed,
style=thick,
mark=triangle*,
mark repeat=20,
mark options=solid
]
table[x index=0,y index=2, header=true, col sep=comma]
{\experimentprefix/product-LS-opt.txt};
\ifnum \showlegend=1
{
\addlegendentry{BCFW-tavg}
}
\fi
}
\fi

\addplot[fill=cyan,draw=none,forget plot,opacity=0.3] table[x index=0,y
index=1, header=true, col sep=comma]
{\experimentprefix/product-LS-wavg_confidence.txt};

\addplot [
color=cyan,
densely dotted,
style=thick,
mark=square,
mark repeat=20,
mark options=solid
]
table[x index=0,y index=2, header=true, col sep=comma]
{\experimentprefix/product-LS-wavg.txt};
\ifnum \showlegend=1
{
\addlegendentry{BCFW-wavg}
}
\fi


\addplot[fill=green,draw=none,forget plot,opacity=0.3] table[x index=0,y
index=1, header=true, col sep=comma]
{\experimentprefix/pegasos_confidence.txt};

\addplot [
color=green,
solid,
style=thick,
mark=o,
mark repeat=20,
]
table[x index=0,y index=2, header=true, col sep=comma]
{\experimentprefix/pegasos.txt};
\ifnum \showlegend=1
{
\addlegendentry{SSG}
}
\fi

\ifdefined \showtavg
{

\addplot [
color=purple,
densely dashed,
style=thick,
mark=diamond,
mark repeat=20,
]
table[x index=0,y index=2, header=true, col sep=comma]
{\experimentprefix/optimalSG.txt};
\ifnum \showlegend=1
{
\addlegendentry{SSG-tavg}
}
\fi
}
\fi


\addplot[fill=purple,draw=none,forget plot,opacity=0.2] table[x index=0,y
index=1, header=true, col sep=comma]
{\experimentprefix/pegasos-wavg_confidence.txt};

\addplot [
color=purple,
densely dotted,
style=thick,
mark=*,
mark repeat=20,
]
table[x index=0,y index=2, header=true, col sep=comma]
{\experimentprefix/pegasos-wavg.txt};
\ifnum \showlegend=1
{
\addlegendentry{SSG-wavg}
}
\fi


\addplot[fill=black,draw=none,forget plot,opacity=0.3] table[x index=0,y
index=1, header=true, col sep=comma]
{\experimentprefix/pegasos-avg_confidence.txt};

\addplot [
color=black,
solid,
style=thick,
mark=pentagon,
mark repeat=20,
]
table[x index=0,y index=2, header=true, col sep=comma]
{\experimentprefix/pegasos-avg.txt};
\ifnum \showlegend=1
{
\addlegendentry{SSG-avg}
}
\fi

\ifdefined \showeg
{

\addplot[fill=gray,draw=none,forget plot,opacity=0.2] table[x index=0,y
index=1, header=true, col sep=comma]
{\experimentprefix/OEG_confidence.txt};

\addplot [
color=gray,
solid,
style=thick,
mark=pentagon*,
mark repeat=20,
]
table[x index=0,y index=2, header=true, col sep=comma]
{\experimentprefix/OEG.txt};
\ifnum \showlegend=1
{
\addlegendentry{online-EG}
}
\fi
}
\fi

\end{axis}

\end{tikzpicture}
\normalsize
        \def\xlabel{effective passes}
        \def\showlegend{0}
        \def\showtavg{1}
        \def\xmin{1}
        \def\xmax{149}
        \def\ymin{0.11}
        \def\ymax{0.23}
        \def\xmode{normal}
        \def\experimentprefix{include/data/dataset=ocr2_lambda=0.001000}
        \small
\begin{tikzpicture}[scale=0.63]

\pgfplotsset{y tick label style={ 
         scaled ticks=false, 
         /pgf/number format/fixed zerofill, 
         /pgf/number format/fixed, 
         /pgf/number format/precision=3, 
     }
}

\begin{axis}[
xlabel=\xlabel,
ylabel=test error,
xmin=\xmin,
xmax=\xmax,
ymin=\ymin,
ymax=\ymax,
xmode=\xmode,
enlargelimits=false, area style,
line legend,
]

\ifdefined \showsfwwithoutline
{

\addplot[fill=red,draw=none,forget plot,opacity=0.3] table[x index=0,y
index=2, header=true, col sep=comma]
{\experimentprefix/product_confidence.txt};

\addplot [
color=red,
solid,
style=thick,
mark=square,
mark repeat=20,
]
table[x index=0,y index=3, header=true, col sep=comma]
{\experimentprefix/product.txt};
\ifnum \showlegend=1
{
\addlegendentry{BCFW-fix}
}
\fi

}
\fi


\addplot[fill=blue,draw=none,forget plot,opacity=0.3] table[x index=0,y
index=2, header=true, col sep=comma]
{\experimentprefix/product-LS_confidence.txt};

\addplot [
color=blue,
solid,
style=thick,
mark=none,
mark=triangle,
mark repeat=20,
]
table[x index=0,y index=3, header=true, col sep=comma]
{\experimentprefix/product-LS.txt};
\ifnum \showlegend=1
{
\addlegendentry{BCFW}
}
\fi

\ifdefined \showtavg
{

\addplot [
color=cyan,
densely dashed,
style=thick,
mark=triangle*,
mark repeat=20,
mark options=solid
]
table[x index=0,y index=3, header=true, col sep=comma]
{\experimentprefix/product-LS-opt.txt};
\ifnum \showlegend=1
{
\addlegendentry{BCFW-tavg}
}
\fi
}
\fi

\addplot[fill=cyan,draw=none,forget plot,opacity=0.3] table[x index=0,y
index=2, header=true, col sep=comma]
{\experimentprefix/product-LS-wavg_confidence.txt};

\addplot [
color=cyan,
densely dotted,
style=thick,
mark=square,
mark repeat=20,
mark options=solid
]
table[x index=0,y index=3, header=true, col sep=comma]
{\experimentprefix/product-LS-wavg.txt};
\ifnum \showlegend=1
{
\addlegendentry{BCFW-wavg}
}
\fi


\addplot[fill=green,draw=none,forget plot,opacity=0.3] table[x index=0,y
index=2, header=true, col sep=comma]
{\experimentprefix/pegasos_confidence.txt};

\addplot [
color=green,
solid,
style=thick,
mark=o,
mark repeat=20,
]
table[x index=0,y index=3, header=true, col sep=comma]
{\experimentprefix/pegasos.txt};
\ifnum \showlegend=1
{
\addlegendentry{SSG}
}
\fi

\ifdefined \showtavg
{
\addplot [
color=purple,
densely dashed,
style=thick,
mark=diamond,
mark repeat=20,
]
table[x index=0,y index=3, header=true, col sep=comma]
{\experimentprefix/optimalSG.txt};
\ifnum \showlegend=1
{
\addlegendentry{SSG-tavg}
}
\fi
}
\fi


\addplot[fill=purple,draw=none,forget plot,opacity=0.2] table[x index=0,y
index=2, header=true, col sep=comma]
{\experimentprefix/pegasos-wavg_confidence.txt};

\addplot [
color=purple,
densely dotted,
style=thick,
mark=*,
mark repeat=20,
]
table[x index=0,y index=3, header=true, col sep=comma]
{\experimentprefix/pegasos-wavg.txt};
\ifnum \showlegend=1
{
\addlegendentry{SSG-wavg}
}
\fi


\addplot[fill=black,draw=none,forget plot,opacity=0.2] table[x index=0,y
index=2, header=true, col sep=comma]
{\experimentprefix/pegasos-avg_confidence.txt};

\addplot [
color=black,
solid,
style=thick,
mark=pentagon,
mark repeat=20,
]
table[x index=0,y index=3, header=true, col sep=comma]
{\experimentprefix/pegasos-avg.txt};
\ifnum \showlegend=1
{
\addlegendentry{SSG-avg}
}
\fi

\ifdefined \showeg
{

\addplot[fill=gray,draw=none,forget plot,opacity=0.2] table[x index=0,y
index=2, header=true, col sep=comma]
{\experimentprefix/OEG_confidence.txt};

\addplot [
color=gray,
solid,
style=thick,
mark=pentagon*,
mark repeat=20,
]
table[x index=0,y index=3, header=true, col sep=comma]
{\experimentprefix/OEG.txt};
\ifnum \showlegend=1
{
\addlegendentry{online-EG}
}
\fi
}
\fi

\end{axis}
\end{tikzpicture}
\normalsize
        \caption{$\lambda=0.001$.}
    \end{subfigure}
    \begin{subfigure}[t]{0.32\linewidth}
        \centering
        \def\xlabel{effective passes}
        \def\xmin{1}
        \def\xmax{149}
        \def\ymin{0.01}
        \def\ymax{100}
        \def\xmode{normal}
        \def\showlegend{0}
        \def\showeg{1}
        \def\showtavg{1}
        \def\legendpos{north east}
        \def\experimentprefix{include/data/dataset=ocr2_lambda=0.000160}
        \small
\begin{tikzpicture}[scale=0.63]

\begin{axis}[
xlabel=\xlabel,
ylabel=primal suboptimality for problem \eqref{eq:svmstruct_nslack_primal},
xmin=\xmin,
xmax=\xmax,
ymin=\ymin,
ymax=\ymax,
enlargelimits=false, area style,
ymode=log,
xmode=\xmode,
line legend,
legend pos=\legendpos,
]

\ifdefined \showsfwwithoutline
{

\addplot[fill=red,draw=none,forget plot,opacity=0.3] table[x index=0,y
index=1, header=true, col sep=comma]
{\experimentprefix/product_confidence.txt};

\addplot [
color=red,
solid,
style=thick,
mark=square,
mark repeat=20,
]
table[x index=0,y index=2, header=true, col sep=comma]
{\experimentprefix/product.txt};
\ifnum \showlegend=1
{
\addlegendentry{BCFW-fix}
}
\fi

}
\fi


\addplot[fill=blue,draw=none,forget plot,opacity=0.3] table[x index=0,y
index=1, header=true, col sep=comma]
{\experimentprefix/product-LS_confidence.txt};

\addplot [
color=blue,
solid,
style=thick,
mark=triangle,
mark repeat=20,
]
table[x index=0,y index=2, header=true, col sep=comma]
{\experimentprefix/product-LS.txt};
\ifnum \showlegend=1
{
\addlegendentry{BCFW}
}
\fi

\ifdefined \showtavg
{
\addplot [
color=cyan,
densely dashed,
style=thick,
mark=triangle*,
mark repeat=20,
mark options=solid
]
table[x index=0,y index=2, header=true, col sep=comma]
{\experimentprefix/product-LS-opt.txt};
\ifnum \showlegend=1
{
\addlegendentry{BCFW-tavg}
}
\fi
}
\fi

\addplot[fill=cyan,draw=none,forget plot,opacity=0.3] table[x index=0,y
index=1, header=true, col sep=comma]
{\experimentprefix/product-LS-wavg_confidence.txt};

\addplot [
color=cyan,
densely dotted,
style=thick,
mark=square,
mark repeat=20,
mark options=solid
]
table[x index=0,y index=2, header=true, col sep=comma]
{\experimentprefix/product-LS-wavg.txt};
\ifnum \showlegend=1
{
\addlegendentry{BCFW-wavg}
}
\fi


\addplot[fill=green,draw=none,forget plot,opacity=0.3] table[x index=0,y
index=1, header=true, col sep=comma]
{\experimentprefix/pegasos_confidence.txt};

\addplot [
color=green,
solid,
style=thick,
mark=o,
mark repeat=20,
]
table[x index=0,y index=2, header=true, col sep=comma]
{\experimentprefix/pegasos.txt};
\ifnum \showlegend=1
{
\addlegendentry{SSG}
}
\fi

\ifdefined \showtavg
{

\addplot [
color=purple,
densely dashed,
style=thick,
mark=diamond,
mark repeat=20,
]
table[x index=0,y index=2, header=true, col sep=comma]
{\experimentprefix/optimalSG.txt};
\ifnum \showlegend=1
{
\addlegendentry{SSG-tavg}
}
\fi
}
\fi


\addplot[fill=purple,draw=none,forget plot,opacity=0.2] table[x index=0,y
index=1, header=true, col sep=comma]
{\experimentprefix/pegasos-wavg_confidence.txt};

\addplot [
color=purple,
densely dotted,
style=thick,
mark=*,
mark repeat=20,
]
table[x index=0,y index=2, header=true, col sep=comma]
{\experimentprefix/pegasos-wavg.txt};
\ifnum \showlegend=1
{
\addlegendentry{SSG-wavg}
}
\fi


\addplot[fill=black,draw=none,forget plot,opacity=0.3] table[x index=0,y
index=1, header=true, col sep=comma]
{\experimentprefix/pegasos-avg_confidence.txt};

\addplot [
color=black,
solid,
style=thick,
mark=pentagon,
mark repeat=20,
]
table[x index=0,y index=2, header=true, col sep=comma]
{\experimentprefix/pegasos-avg.txt};
\ifnum \showlegend=1
{
\addlegendentry{SSG-avg}
}
\fi

\ifdefined \showeg
{

\addplot[fill=gray,draw=none,forget plot,opacity=0.2] table[x index=0,y
index=1, header=true, col sep=comma]
{\experimentprefix/OEG_confidence.txt};

\addplot [
color=gray,
solid,
style=thick,
mark=pentagon*,
mark repeat=20,
]
table[x index=0,y index=2, header=true, col sep=comma]
{\experimentprefix/OEG.txt};
\ifnum \showlegend=1
{
\addlegendentry{online-EG}
}
\fi
}
\fi

\end{axis}

\end{tikzpicture}
\normalsize
        \def\xlabel{effective passes}
        \def\showlegend{0}
        \def\showtavg{1}
        \def\xmin{1}
        \def\xmax{149}
        \def\ymin{0.11}
        \def\ymax{0.23}
        \def\xmode{normal}
        \def\experimentprefix{include/data/dataset=ocr2_lambda=0.000160}
        \small
\begin{tikzpicture}[scale=0.63]

\pgfplotsset{y tick label style={ 
         scaled ticks=false, 
         /pgf/number format/fixed zerofill, 
         /pgf/number format/fixed, 
         /pgf/number format/precision=3, 
     }
}

\begin{axis}[
xlabel=\xlabel,
ylabel=test error,
xmin=\xmin,
xmax=\xmax,
ymin=\ymin,
ymax=\ymax,
xmode=\xmode,
enlargelimits=false, area style,
line legend,
]

\ifdefined \showsfwwithoutline
{

\addplot[fill=red,draw=none,forget plot,opacity=0.3] table[x index=0,y
index=2, header=true, col sep=comma]
{\experimentprefix/product_confidence.txt};

\addplot [
color=red,
solid,
style=thick,
mark=square,
mark repeat=20,
]
table[x index=0,y index=3, header=true, col sep=comma]
{\experimentprefix/product.txt};
\ifnum \showlegend=1
{
\addlegendentry{BCFW-fix}
}
\fi

}
\fi


\addplot[fill=blue,draw=none,forget plot,opacity=0.3] table[x index=0,y
index=2, header=true, col sep=comma]
{\experimentprefix/product-LS_confidence.txt};

\addplot [
color=blue,
solid,
style=thick,
mark=none,
mark=triangle,
mark repeat=20,
]
table[x index=0,y index=3, header=true, col sep=comma]
{\experimentprefix/product-LS.txt};
\ifnum \showlegend=1
{
\addlegendentry{BCFW}
}
\fi

\ifdefined \showtavg
{

\addplot [
color=cyan,
densely dashed,
style=thick,
mark=triangle*,
mark repeat=20,
mark options=solid
]
table[x index=0,y index=3, header=true, col sep=comma]
{\experimentprefix/product-LS-opt.txt};
\ifnum \showlegend=1
{
\addlegendentry{BCFW-tavg}
}
\fi
}
\fi

\addplot[fill=cyan,draw=none,forget plot,opacity=0.3] table[x index=0,y
index=2, header=true, col sep=comma]
{\experimentprefix/product-LS-wavg_confidence.txt};

\addplot [
color=cyan,
densely dotted,
style=thick,
mark=square,
mark repeat=20,
mark options=solid
]
table[x index=0,y index=3, header=true, col sep=comma]
{\experimentprefix/product-LS-wavg.txt};
\ifnum \showlegend=1
{
\addlegendentry{BCFW-wavg}
}
\fi


\addplot[fill=green,draw=none,forget plot,opacity=0.3] table[x index=0,y
index=2, header=true, col sep=comma]
{\experimentprefix/pegasos_confidence.txt};

\addplot [
color=green,
solid,
style=thick,
mark=o,
mark repeat=20,
]
table[x index=0,y index=3, header=true, col sep=comma]
{\experimentprefix/pegasos.txt};
\ifnum \showlegend=1
{
\addlegendentry{SSG}
}
\fi

\ifdefined \showtavg
{
\addplot [
color=purple,
densely dashed,
style=thick,
mark=diamond,
mark repeat=20,
]
table[x index=0,y index=3, header=true, col sep=comma]
{\experimentprefix/optimalSG.txt};
\ifnum \showlegend=1
{
\addlegendentry{SSG-tavg}
}
\fi
}
\fi


\addplot[fill=purple,draw=none,forget plot,opacity=0.2] table[x index=0,y
index=2, header=true, col sep=comma]
{\experimentprefix/pegasos-wavg_confidence.txt};

\addplot [
color=purple,
densely dotted,
style=thick,
mark=*,
mark repeat=20,
]
table[x index=0,y index=3, header=true, col sep=comma]
{\experimentprefix/pegasos-wavg.txt};
\ifnum \showlegend=1
{
\addlegendentry{SSG-wavg}
}
\fi


\addplot[fill=black,draw=none,forget plot,opacity=0.2] table[x index=0,y
index=2, header=true, col sep=comma]
{\experimentprefix/pegasos-avg_confidence.txt};

\addplot [
color=black,
solid,
style=thick,
mark=pentagon,
mark repeat=20,
]
table[x index=0,y index=3, header=true, col sep=comma]
{\experimentprefix/pegasos-avg.txt};
\ifnum \showlegend=1
{
\addlegendentry{SSG-avg}
}
\fi

\ifdefined \showeg
{

\addplot[fill=gray,draw=none,forget plot,opacity=0.2] table[x index=0,y
index=2, header=true, col sep=comma]
{\experimentprefix/OEG_confidence.txt};

\addplot [
color=gray,
solid,
style=thick,
mark=pentagon*,
mark repeat=20,
]
table[x index=0,y index=3, header=true, col sep=comma]
{\experimentprefix/OEG.txt};
\ifnum \showlegend=1
{
\addlegendentry{online-EG}
}
\fi
}
\fi

\end{axis}
\end{tikzpicture}
\normalsize
        \caption{$\lambda=1/n=0.00016$.}
    \end{subfigure}
    \caption{Convergence (top) and test error (bottom) of the stochastic
    solvers on the OCR dataset. While the block-coordinate Frank-Wolfe
    algorithm generally achieves the best objective, the averaging versions of the stochastic algorithms achieve a lower test error. While this interesting observation should be subject to further investigation, this could probably be due to the fact that these methods implicitly
    perform model averaging, which seems to lead to improved generalization performance. One can see some kind of `overfitting' for example in the $\lambda = 0.001$ case, where \emph{BCFW-wavg} reaches early a low test error which then starts to increase (and after running it for thousands of iterations, it does seem to converge to a parameter with higher test error than seen in the early iterations).
    }
    \label{fig:ocr2_results}
\end{figure}

\begin{figure}[htb]
    \begin{subfigure}[t]{0.32\linewidth}
        \centering
        \def\xlabel{effective passes}
        \def\xmin{0.1}
        \def\xmax{49}
        \def\ymin{0.001}
        \def\ymax{100}
        \def\xmode{log}
        \def\showlegend{1}
        \def\showtavg{1}
        \def\legendpos{north east}
        \def\experimentprefix{include/data/dataset=conll_lambda=0.010000}
        \small
\begin{tikzpicture}[scale=0.63]

\begin{axis}[
xlabel=\xlabel,
ylabel=primal suboptimality for problem \eqref{eq:svmstruct_nslack_primal},
xmin=\xmin,
xmax=\xmax,
ymin=\ymin,
ymax=\ymax,
enlargelimits=false, area style,
ymode=log,
xmode=\xmode,
line legend,
legend pos=\legendpos,
]

\ifdefined \showsfwwithoutline
{

\addplot[fill=red,draw=none,forget plot,opacity=0.3] table[x index=0,y
index=1, header=true, col sep=comma]
{\experimentprefix/product_confidence.txt};

\addplot [
color=red,
solid,
style=thick,
mark=square,
mark repeat=20,
]
table[x index=0,y index=2, header=true, col sep=comma]
{\experimentprefix/product.txt};
\ifnum \showlegend=1
{
\addlegendentry{BCFW-fix}
}
\fi

}
\fi


\addplot[fill=blue,draw=none,forget plot,opacity=0.3] table[x index=0,y
index=1, header=true, col sep=comma]
{\experimentprefix/product-LS_confidence.txt};

\addplot [
color=blue,
solid,
style=thick,
mark=triangle,
mark repeat=20,
]
table[x index=0,y index=2, header=true, col sep=comma]
{\experimentprefix/product-LS.txt};
\ifnum \showlegend=1
{
\addlegendentry{BCFW}
}
\fi

\ifdefined \showtavg
{
\addplot [
color=cyan,
densely dashed,
style=thick,
mark=triangle*,
mark repeat=20,
mark options=solid
]
table[x index=0,y index=2, header=true, col sep=comma]
{\experimentprefix/product-LS-opt.txt};
\ifnum \showlegend=1
{
\addlegendentry{BCFW-tavg}
}
\fi
}
\fi

\addplot[fill=cyan,draw=none,forget plot,opacity=0.3] table[x index=0,y
index=1, header=true, col sep=comma]
{\experimentprefix/product-LS-wavg_confidence.txt};

\addplot [
color=cyan,
densely dotted,
style=thick,
mark=square,
mark repeat=20,
mark options=solid
]
table[x index=0,y index=2, header=true, col sep=comma]
{\experimentprefix/product-LS-wavg.txt};
\ifnum \showlegend=1
{
\addlegendentry{BCFW-wavg}
}
\fi


\addplot[fill=green,draw=none,forget plot,opacity=0.3] table[x index=0,y
index=1, header=true, col sep=comma]
{\experimentprefix/pegasos_confidence.txt};

\addplot [
color=green,
solid,
style=thick,
mark=o,
mark repeat=20,
]
table[x index=0,y index=2, header=true, col sep=comma]
{\experimentprefix/pegasos.txt};
\ifnum \showlegend=1
{
\addlegendentry{SSG}
}
\fi

\ifdefined \showtavg
{

\addplot [
color=purple,
densely dashed,
style=thick,
mark=diamond,
mark repeat=20,
]
table[x index=0,y index=2, header=true, col sep=comma]
{\experimentprefix/optimalSG.txt};
\ifnum \showlegend=1
{
\addlegendentry{SSG-tavg}
}
\fi
}
\fi


\addplot[fill=purple,draw=none,forget plot,opacity=0.2] table[x index=0,y
index=1, header=true, col sep=comma]
{\experimentprefix/pegasos-wavg_confidence.txt};

\addplot [
color=purple,
densely dotted,
style=thick,
mark=*,
mark repeat=20,
]
table[x index=0,y index=2, header=true, col sep=comma]
{\experimentprefix/pegasos-wavg.txt};
\ifnum \showlegend=1
{
\addlegendentry{SSG-wavg}
}
\fi


\addplot[fill=black,draw=none,forget plot,opacity=0.3] table[x index=0,y
index=1, header=true, col sep=comma]
{\experimentprefix/pegasos-avg_confidence.txt};

\addplot [
color=black,
solid,
style=thick,
mark=pentagon,
mark repeat=20,
]
table[x index=0,y index=2, header=true, col sep=comma]
{\experimentprefix/pegasos-avg.txt};
\ifnum \showlegend=1
{
\addlegendentry{SSG-avg}
}
\fi

\ifdefined \showeg
{

\addplot[fill=gray,draw=none,forget plot,opacity=0.2] table[x index=0,y
index=1, header=true, col sep=comma]
{\experimentprefix/OEG_confidence.txt};

\addplot [
color=gray,
solid,
style=thick,
mark=pentagon*,
mark repeat=20,
]
table[x index=0,y index=2, header=true, col sep=comma]
{\experimentprefix/OEG.txt};
\ifnum \showlegend=1
{
\addlegendentry{online-EG}
}
\fi
}
\fi

\end{axis}

\end{tikzpicture}
\normalsize
        \def\xlabel{effective passes}
        \def\showlegend{0}
        \def\showtavg{1}
        \def\xmin{0.1}
        \def\xmax{49}
        \def\ymin{0.04}
        \def\ymax{0.07}
        \def\xmode{log}
        \def\experimentprefix{include/data/dataset=conll_lambda=0.010000}
        \small
\begin{tikzpicture}[scale=0.63]

\pgfplotsset{y tick label style={ 
         scaled ticks=false, 
         /pgf/number format/fixed zerofill, 
         /pgf/number format/fixed, 
         /pgf/number format/precision=3, 
     }
}

\begin{axis}[
xlabel=\xlabel,
ylabel=test error,
xmin=\xmin,
xmax=\xmax,
ymin=\ymin,
ymax=\ymax,
xmode=\xmode,
enlargelimits=false, area style,
line legend,
]

\ifdefined \showsfwwithoutline
{

\addplot[fill=red,draw=none,forget plot,opacity=0.3] table[x index=0,y
index=2, header=true, col sep=comma]
{\experimentprefix/product_confidence.txt};

\addplot [
color=red,
solid,
style=thick,
mark=square,
mark repeat=20,
]
table[x index=0,y index=3, header=true, col sep=comma]
{\experimentprefix/product.txt};
\ifnum \showlegend=1
{
\addlegendentry{BCFW-fix}
}
\fi

}
\fi


\addplot[fill=blue,draw=none,forget plot,opacity=0.3] table[x index=0,y
index=2, header=true, col sep=comma]
{\experimentprefix/product-LS_confidence.txt};

\addplot [
color=blue,
solid,
style=thick,
mark=none,
mark=triangle,
mark repeat=20,
]
table[x index=0,y index=3, header=true, col sep=comma]
{\experimentprefix/product-LS.txt};
\ifnum \showlegend=1
{
\addlegendentry{BCFW}
}
\fi

\ifdefined \showtavg
{

\addplot [
color=cyan,
densely dashed,
style=thick,
mark=triangle*,
mark repeat=20,
mark options=solid
]
table[x index=0,y index=3, header=true, col sep=comma]
{\experimentprefix/product-LS-opt.txt};
\ifnum \showlegend=1
{
\addlegendentry{BCFW-tavg}
}
\fi
}
\fi

\addplot[fill=cyan,draw=none,forget plot,opacity=0.3] table[x index=0,y
index=2, header=true, col sep=comma]
{\experimentprefix/product-LS-wavg_confidence.txt};

\addplot [
color=cyan,
densely dotted,
style=thick,
mark=square,
mark repeat=20,
mark options=solid
]
table[x index=0,y index=3, header=true, col sep=comma]
{\experimentprefix/product-LS-wavg.txt};
\ifnum \showlegend=1
{
\addlegendentry{BCFW-wavg}
}
\fi


\addplot[fill=green,draw=none,forget plot,opacity=0.3] table[x index=0,y
index=2, header=true, col sep=comma]
{\experimentprefix/pegasos_confidence.txt};

\addplot [
color=green,
solid,
style=thick,
mark=o,
mark repeat=20,
]
table[x index=0,y index=3, header=true, col sep=comma]
{\experimentprefix/pegasos.txt};
\ifnum \showlegend=1
{
\addlegendentry{SSG}
}
\fi

\ifdefined \showtavg
{
\addplot [
color=purple,
densely dashed,
style=thick,
mark=diamond,
mark repeat=20,
]
table[x index=0,y index=3, header=true, col sep=comma]
{\experimentprefix/optimalSG.txt};
\ifnum \showlegend=1
{
\addlegendentry{SSG-tavg}
}
\fi
}
\fi


\addplot[fill=purple,draw=none,forget plot,opacity=0.2] table[x index=0,y
index=2, header=true, col sep=comma]
{\experimentprefix/pegasos-wavg_confidence.txt};

\addplot [
color=purple,
densely dotted,
style=thick,
mark=*,
mark repeat=20,
]
table[x index=0,y index=3, header=true, col sep=comma]
{\experimentprefix/pegasos-wavg.txt};
\ifnum \showlegend=1
{
\addlegendentry{SSG-wavg}
}
\fi


\addplot[fill=black,draw=none,forget plot,opacity=0.2] table[x index=0,y
index=2, header=true, col sep=comma]
{\experimentprefix/pegasos-avg_confidence.txt};

\addplot [
color=black,
solid,
style=thick,
mark=pentagon,
mark repeat=20,
]
table[x index=0,y index=3, header=true, col sep=comma]
{\experimentprefix/pegasos-avg.txt};
\ifnum \showlegend=1
{
\addlegendentry{SSG-avg}
}
\fi

\ifdefined \showeg
{

\addplot[fill=gray,draw=none,forget plot,opacity=0.2] table[x index=0,y
index=2, header=true, col sep=comma]
{\experimentprefix/OEG_confidence.txt};

\addplot [
color=gray,
solid,
style=thick,
mark=pentagon*,
mark repeat=20,
]
table[x index=0,y index=3, header=true, col sep=comma]
{\experimentprefix/OEG.txt};
\ifnum \showlegend=1
{
\addlegendentry{online-EG}
}
\fi
}
\fi

\end{axis}
\end{tikzpicture}
\normalsize
        \caption{$\lambda=0.01$.}
    \end{subfigure}
    \begin{subfigure}[t]{0.32\linewidth}
        \centering
        \def\xlabel{effective passes}
        \def\xmin{0.1}
        \def\xmax{49}
        \def\ymin{0.001}
        \def\ymax{100}
        \def\xmode{log}
        \def\showlegend{0}
        \def\showtavg{1}
        \def\legendpos{north east}
        \def\experimentprefix{include/data/dataset=conll_lambda=0.001000}
        \small
\begin{tikzpicture}[scale=0.63]

\begin{axis}[
xlabel=\xlabel,
ylabel=primal suboptimality for problem \eqref{eq:svmstruct_nslack_primal},
xmin=\xmin,
xmax=\xmax,
ymin=\ymin,
ymax=\ymax,
enlargelimits=false, area style,
ymode=log,
xmode=\xmode,
line legend,
legend pos=\legendpos,
]

\ifdefined \showsfwwithoutline
{

\addplot[fill=red,draw=none,forget plot,opacity=0.3] table[x index=0,y
index=1, header=true, col sep=comma]
{\experimentprefix/product_confidence.txt};

\addplot [
color=red,
solid,
style=thick,
mark=square,
mark repeat=20,
]
table[x index=0,y index=2, header=true, col sep=comma]
{\experimentprefix/product.txt};
\ifnum \showlegend=1
{
\addlegendentry{BCFW-fix}
}
\fi

}
\fi


\addplot[fill=blue,draw=none,forget plot,opacity=0.3] table[x index=0,y
index=1, header=true, col sep=comma]
{\experimentprefix/product-LS_confidence.txt};

\addplot [
color=blue,
solid,
style=thick,
mark=triangle,
mark repeat=20,
]
table[x index=0,y index=2, header=true, col sep=comma]
{\experimentprefix/product-LS.txt};
\ifnum \showlegend=1
{
\addlegendentry{BCFW}
}
\fi

\ifdefined \showtavg
{
\addplot [
color=cyan,
densely dashed,
style=thick,
mark=triangle*,
mark repeat=20,
mark options=solid
]
table[x index=0,y index=2, header=true, col sep=comma]
{\experimentprefix/product-LS-opt.txt};
\ifnum \showlegend=1
{
\addlegendentry{BCFW-tavg}
}
\fi
}
\fi

\addplot[fill=cyan,draw=none,forget plot,opacity=0.3] table[x index=0,y
index=1, header=true, col sep=comma]
{\experimentprefix/product-LS-wavg_confidence.txt};

\addplot [
color=cyan,
densely dotted,
style=thick,
mark=square,
mark repeat=20,
mark options=solid
]
table[x index=0,y index=2, header=true, col sep=comma]
{\experimentprefix/product-LS-wavg.txt};
\ifnum \showlegend=1
{
\addlegendentry{BCFW-wavg}
}
\fi


\addplot[fill=green,draw=none,forget plot,opacity=0.3] table[x index=0,y
index=1, header=true, col sep=comma]
{\experimentprefix/pegasos_confidence.txt};

\addplot [
color=green,
solid,
style=thick,
mark=o,
mark repeat=20,
]
table[x index=0,y index=2, header=true, col sep=comma]
{\experimentprefix/pegasos.txt};
\ifnum \showlegend=1
{
\addlegendentry{SSG}
}
\fi

\ifdefined \showtavg
{

\addplot [
color=purple,
densely dashed,
style=thick,
mark=diamond,
mark repeat=20,
]
table[x index=0,y index=2, header=true, col sep=comma]
{\experimentprefix/optimalSG.txt};
\ifnum \showlegend=1
{
\addlegendentry{SSG-tavg}
}
\fi
}
\fi


\addplot[fill=purple,draw=none,forget plot,opacity=0.2] table[x index=0,y
index=1, header=true, col sep=comma]
{\experimentprefix/pegasos-wavg_confidence.txt};

\addplot [
color=purple,
densely dotted,
style=thick,
mark=*,
mark repeat=20,
]
table[x index=0,y index=2, header=true, col sep=comma]
{\experimentprefix/pegasos-wavg.txt};
\ifnum \showlegend=1
{
\addlegendentry{SSG-wavg}
}
\fi


\addplot[fill=black,draw=none,forget plot,opacity=0.3] table[x index=0,y
index=1, header=true, col sep=comma]
{\experimentprefix/pegasos-avg_confidence.txt};

\addplot [
color=black,
solid,
style=thick,
mark=pentagon,
mark repeat=20,
]
table[x index=0,y index=2, header=true, col sep=comma]
{\experimentprefix/pegasos-avg.txt};
\ifnum \showlegend=1
{
\addlegendentry{SSG-avg}
}
\fi

\ifdefined \showeg
{

\addplot[fill=gray,draw=none,forget plot,opacity=0.2] table[x index=0,y
index=1, header=true, col sep=comma]
{\experimentprefix/OEG_confidence.txt};

\addplot [
color=gray,
solid,
style=thick,
mark=pentagon*,
mark repeat=20,
]
table[x index=0,y index=2, header=true, col sep=comma]
{\experimentprefix/OEG.txt};
\ifnum \showlegend=1
{
\addlegendentry{online-EG}
}
\fi
}
\fi

\end{axis}

\end{tikzpicture}
\normalsize
        \def\xlabel{effective passes}
        \def\showlegend{0}
        \def\showtavg{1}
        \def\xmin{0.1}
        \def\xmax{49}
        \def\ymin{0.04}
        \def\ymax{0.07}
        \def\xmode{log}
        \def\experimentprefix{include/data/dataset=conll_lambda=0.001000}
        \small
\begin{tikzpicture}[scale=0.63]

\pgfplotsset{y tick label style={ 
         scaled ticks=false, 
         /pgf/number format/fixed zerofill, 
         /pgf/number format/fixed, 
         /pgf/number format/precision=3, 
     }
}

\begin{axis}[
xlabel=\xlabel,
ylabel=test error,
xmin=\xmin,
xmax=\xmax,
ymin=\ymin,
ymax=\ymax,
xmode=\xmode,
enlargelimits=false, area style,
line legend,
]

\ifdefined \showsfwwithoutline
{

\addplot[fill=red,draw=none,forget plot,opacity=0.3] table[x index=0,y
index=2, header=true, col sep=comma]
{\experimentprefix/product_confidence.txt};

\addplot [
color=red,
solid,
style=thick,
mark=square,
mark repeat=20,
]
table[x index=0,y index=3, header=true, col sep=comma]
{\experimentprefix/product.txt};
\ifnum \showlegend=1
{
\addlegendentry{BCFW-fix}
}
\fi

}
\fi


\addplot[fill=blue,draw=none,forget plot,opacity=0.3] table[x index=0,y
index=2, header=true, col sep=comma]
{\experimentprefix/product-LS_confidence.txt};

\addplot [
color=blue,
solid,
style=thick,
mark=none,
mark=triangle,
mark repeat=20,
]
table[x index=0,y index=3, header=true, col sep=comma]
{\experimentprefix/product-LS.txt};
\ifnum \showlegend=1
{
\addlegendentry{BCFW}
}
\fi

\ifdefined \showtavg
{

\addplot [
color=cyan,
densely dashed,
style=thick,
mark=triangle*,
mark repeat=20,
mark options=solid
]
table[x index=0,y index=3, header=true, col sep=comma]
{\experimentprefix/product-LS-opt.txt};
\ifnum \showlegend=1
{
\addlegendentry{BCFW-tavg}
}
\fi
}
\fi

\addplot[fill=cyan,draw=none,forget plot,opacity=0.3] table[x index=0,y
index=2, header=true, col sep=comma]
{\experimentprefix/product-LS-wavg_confidence.txt};

\addplot [
color=cyan,
densely dotted,
style=thick,
mark=square,
mark repeat=20,
mark options=solid
]
table[x index=0,y index=3, header=true, col sep=comma]
{\experimentprefix/product-LS-wavg.txt};
\ifnum \showlegend=1
{
\addlegendentry{BCFW-wavg}
}
\fi


\addplot[fill=green,draw=none,forget plot,opacity=0.3] table[x index=0,y
index=2, header=true, col sep=comma]
{\experimentprefix/pegasos_confidence.txt};

\addplot [
color=green,
solid,
style=thick,
mark=o,
mark repeat=20,
]
table[x index=0,y index=3, header=true, col sep=comma]
{\experimentprefix/pegasos.txt};
\ifnum \showlegend=1
{
\addlegendentry{SSG}
}
\fi

\ifdefined \showtavg
{
\addplot [
color=purple,
densely dashed,
style=thick,
mark=diamond,
mark repeat=20,
]
table[x index=0,y index=3, header=true, col sep=comma]
{\experimentprefix/optimalSG.txt};
\ifnum \showlegend=1
{
\addlegendentry{SSG-tavg}
}
\fi
}
\fi


\addplot[fill=purple,draw=none,forget plot,opacity=0.2] table[x index=0,y
index=2, header=true, col sep=comma]
{\experimentprefix/pegasos-wavg_confidence.txt};

\addplot [
color=purple,
densely dotted,
style=thick,
mark=*,
mark repeat=20,
]
table[x index=0,y index=3, header=true, col sep=comma]
{\experimentprefix/pegasos-wavg.txt};
\ifnum \showlegend=1
{
\addlegendentry{SSG-wavg}
}
\fi


\addplot[fill=black,draw=none,forget plot,opacity=0.2] table[x index=0,y
index=2, header=true, col sep=comma]
{\experimentprefix/pegasos-avg_confidence.txt};

\addplot [
color=black,
solid,
style=thick,
mark=pentagon,
mark repeat=20,
]
table[x index=0,y index=3, header=true, col sep=comma]
{\experimentprefix/pegasos-avg.txt};
\ifnum \showlegend=1
{
\addlegendentry{SSG-avg}
}
\fi

\ifdefined \showeg
{

\addplot[fill=gray,draw=none,forget plot,opacity=0.2] table[x index=0,y
index=2, header=true, col sep=comma]
{\experimentprefix/OEG_confidence.txt};

\addplot [
color=gray,
solid,
style=thick,
mark=pentagon*,
mark repeat=20,
]
table[x index=0,y index=3, header=true, col sep=comma]
{\experimentprefix/OEG.txt};
\ifnum \showlegend=1
{
\addlegendentry{online-EG}
}
\fi
}
\fi

\end{axis}
\end{tikzpicture}
\normalsize
        \caption{$\lambda=0.001$.}
    \end{subfigure}
    \begin{subfigure}[t]{0.32\linewidth}
        \centering
        \def\xlabel{effective passes}
        \def\xmin{0.1}
        \def\xmax{49}
        \def\ymin{0.01}
        \def\ymax{100}
        \def\xmode{log}
        \def\showlegend{0}
        \def\showtavg{1}
        \def\legendpos{north east}
        \def\experimentprefix{include/data/dataset=conll_lambda=0.000112}
        \small
\begin{tikzpicture}[scale=0.63]

\begin{axis}[
xlabel=\xlabel,
ylabel=primal suboptimality for problem \eqref{eq:svmstruct_nslack_primal},
xmin=\xmin,
xmax=\xmax,
ymin=\ymin,
ymax=\ymax,
enlargelimits=false, area style,
ymode=log,
xmode=\xmode,
line legend,
legend pos=\legendpos,
]

\ifdefined \showsfwwithoutline
{

\addplot[fill=red,draw=none,forget plot,opacity=0.3] table[x index=0,y
index=1, header=true, col sep=comma]
{\experimentprefix/product_confidence.txt};

\addplot [
color=red,
solid,
style=thick,
mark=square,
mark repeat=20,
]
table[x index=0,y index=2, header=true, col sep=comma]
{\experimentprefix/product.txt};
\ifnum \showlegend=1
{
\addlegendentry{BCFW-fix}
}
\fi

}
\fi


\addplot[fill=blue,draw=none,forget plot,opacity=0.3] table[x index=0,y
index=1, header=true, col sep=comma]
{\experimentprefix/product-LS_confidence.txt};

\addplot [
color=blue,
solid,
style=thick,
mark=triangle,
mark repeat=20,
]
table[x index=0,y index=2, header=true, col sep=comma]
{\experimentprefix/product-LS.txt};
\ifnum \showlegend=1
{
\addlegendentry{BCFW}
}
\fi

\ifdefined \showtavg
{
\addplot [
color=cyan,
densely dashed,
style=thick,
mark=triangle*,
mark repeat=20,
mark options=solid
]
table[x index=0,y index=2, header=true, col sep=comma]
{\experimentprefix/product-LS-opt.txt};
\ifnum \showlegend=1
{
\addlegendentry{BCFW-tavg}
}
\fi
}
\fi

\addplot[fill=cyan,draw=none,forget plot,opacity=0.3] table[x index=0,y
index=1, header=true, col sep=comma]
{\experimentprefix/product-LS-wavg_confidence.txt};

\addplot [
color=cyan,
densely dotted,
style=thick,
mark=square,
mark repeat=20,
mark options=solid
]
table[x index=0,y index=2, header=true, col sep=comma]
{\experimentprefix/product-LS-wavg.txt};
\ifnum \showlegend=1
{
\addlegendentry{BCFW-wavg}
}
\fi


\addplot[fill=green,draw=none,forget plot,opacity=0.3] table[x index=0,y
index=1, header=true, col sep=comma]
{\experimentprefix/pegasos_confidence.txt};

\addplot [
color=green,
solid,
style=thick,
mark=o,
mark repeat=20,
]
table[x index=0,y index=2, header=true, col sep=comma]
{\experimentprefix/pegasos.txt};
\ifnum \showlegend=1
{
\addlegendentry{SSG}
}
\fi

\ifdefined \showtavg
{

\addplot [
color=purple,
densely dashed,
style=thick,
mark=diamond,
mark repeat=20,
]
table[x index=0,y index=2, header=true, col sep=comma]
{\experimentprefix/optimalSG.txt};
\ifnum \showlegend=1
{
\addlegendentry{SSG-tavg}
}
\fi
}
\fi


\addplot[fill=purple,draw=none,forget plot,opacity=0.2] table[x index=0,y
index=1, header=true, col sep=comma]
{\experimentprefix/pegasos-wavg_confidence.txt};

\addplot [
color=purple,
densely dotted,
style=thick,
mark=*,
mark repeat=20,
]
table[x index=0,y index=2, header=true, col sep=comma]
{\experimentprefix/pegasos-wavg.txt};
\ifnum \showlegend=1
{
\addlegendentry{SSG-wavg}
}
\fi


\addplot[fill=black,draw=none,forget plot,opacity=0.3] table[x index=0,y
index=1, header=true, col sep=comma]
{\experimentprefix/pegasos-avg_confidence.txt};

\addplot [
color=black,
solid,
style=thick,
mark=pentagon,
mark repeat=20,
]
table[x index=0,y index=2, header=true, col sep=comma]
{\experimentprefix/pegasos-avg.txt};
\ifnum \showlegend=1
{
\addlegendentry{SSG-avg}
}
\fi

\ifdefined \showeg
{

\addplot[fill=gray,draw=none,forget plot,opacity=0.2] table[x index=0,y
index=1, header=true, col sep=comma]
{\experimentprefix/OEG_confidence.txt};

\addplot [
color=gray,
solid,
style=thick,
mark=pentagon*,
mark repeat=20,
]
table[x index=0,y index=2, header=true, col sep=comma]
{\experimentprefix/OEG.txt};
\ifnum \showlegend=1
{
\addlegendentry{online-EG}
}
\fi
}
\fi

\end{axis}

\end{tikzpicture}
\normalsize
        \def\xlabel{effective passes}
        \def\showlegend{0}
        \def\showtavg{1}
        \def\xmin{0.1}
        \def\xmax{49}
        \def\ymin{0.04}
        \def\ymax{0.07}
        \def\xmode{log}
        \def\experimentprefix{include/data/dataset=conll_lambda=0.000112}
        \small
\begin{tikzpicture}[scale=0.63]

\pgfplotsset{y tick label style={ 
         scaled ticks=false, 
         /pgf/number format/fixed zerofill, 
         /pgf/number format/fixed, 
         /pgf/number format/precision=3, 
     }
}

\begin{axis}[
xlabel=\xlabel,
ylabel=test error,
xmin=\xmin,
xmax=\xmax,
ymin=\ymin,
ymax=\ymax,
xmode=\xmode,
enlargelimits=false, area style,
line legend,
]

\ifdefined \showsfwwithoutline
{

\addplot[fill=red,draw=none,forget plot,opacity=0.3] table[x index=0,y
index=2, header=true, col sep=comma]
{\experimentprefix/product_confidence.txt};

\addplot [
color=red,
solid,
style=thick,
mark=square,
mark repeat=20,
]
table[x index=0,y index=3, header=true, col sep=comma]
{\experimentprefix/product.txt};
\ifnum \showlegend=1
{
\addlegendentry{BCFW-fix}
}
\fi

}
\fi


\addplot[fill=blue,draw=none,forget plot,opacity=0.3] table[x index=0,y
index=2, header=true, col sep=comma]
{\experimentprefix/product-LS_confidence.txt};

\addplot [
color=blue,
solid,
style=thick,
mark=none,
mark=triangle,
mark repeat=20,
]
table[x index=0,y index=3, header=true, col sep=comma]
{\experimentprefix/product-LS.txt};
\ifnum \showlegend=1
{
\addlegendentry{BCFW}
}
\fi

\ifdefined \showtavg
{

\addplot [
color=cyan,
densely dashed,
style=thick,
mark=triangle*,
mark repeat=20,
mark options=solid
]
table[x index=0,y index=3, header=true, col sep=comma]
{\experimentprefix/product-LS-opt.txt};
\ifnum \showlegend=1
{
\addlegendentry{BCFW-tavg}
}
\fi
}
\fi

\addplot[fill=cyan,draw=none,forget plot,opacity=0.3] table[x index=0,y
index=2, header=true, col sep=comma]
{\experimentprefix/product-LS-wavg_confidence.txt};

\addplot [
color=cyan,
densely dotted,
style=thick,
mark=square,
mark repeat=20,
mark options=solid
]
table[x index=0,y index=3, header=true, col sep=comma]
{\experimentprefix/product-LS-wavg.txt};
\ifnum \showlegend=1
{
\addlegendentry{BCFW-wavg}
}
\fi


\addplot[fill=green,draw=none,forget plot,opacity=0.3] table[x index=0,y
index=2, header=true, col sep=comma]
{\experimentprefix/pegasos_confidence.txt};

\addplot [
color=green,
solid,
style=thick,
mark=o,
mark repeat=20,
]
table[x index=0,y index=3, header=true, col sep=comma]
{\experimentprefix/pegasos.txt};
\ifnum \showlegend=1
{
\addlegendentry{SSG}
}
\fi

\ifdefined \showtavg
{
\addplot [
color=purple,
densely dashed,
style=thick,
mark=diamond,
mark repeat=20,
]
table[x index=0,y index=3, header=true, col sep=comma]
{\experimentprefix/optimalSG.txt};
\ifnum \showlegend=1
{
\addlegendentry{SSG-tavg}
}
\fi
}
\fi


\addplot[fill=purple,draw=none,forget plot,opacity=0.2] table[x index=0,y
index=2, header=true, col sep=comma]
{\experimentprefix/pegasos-wavg_confidence.txt};

\addplot [
color=purple,
densely dotted,
style=thick,
mark=*,
mark repeat=20,
]
table[x index=0,y index=3, header=true, col sep=comma]
{\experimentprefix/pegasos-wavg.txt};
\ifnum \showlegend=1
{
\addlegendentry{SSG-wavg}
}
\fi


\addplot[fill=black,draw=none,forget plot,opacity=0.2] table[x index=0,y
index=2, header=true, col sep=comma]
{\experimentprefix/pegasos-avg_confidence.txt};

\addplot [
color=black,
solid,
style=thick,
mark=pentagon,
mark repeat=20,
]
table[x index=0,y index=3, header=true, col sep=comma]
{\experimentprefix/pegasos-avg.txt};
\ifnum \showlegend=1
{
\addlegendentry{SSG-avg}
}
\fi

\ifdefined \showeg
{

\addplot[fill=gray,draw=none,forget plot,opacity=0.2] table[x index=0,y
index=2, header=true, col sep=comma]
{\experimentprefix/OEG_confidence.txt};

\addplot [
color=gray,
solid,
style=thick,
mark=pentagon*,
mark repeat=20,
]
table[x index=0,y index=3, header=true, col sep=comma]
{\experimentprefix/OEG.txt};
\ifnum \showlegend=1
{
\addlegendentry{online-EG}
}
\fi
}
\fi

\end{axis}
\end{tikzpicture}
\normalsize
        \caption{$\lambda=1/n=0.000112$.}
    \end{subfigure}
    \caption{Convergence (top) and test error (bottom) of the stochastic
    solvers on the CoNLL dataset, using a logarithmic x-axis to focus on the early iterations. %
    }
    \label{fig:conll_results}
\end{figure}
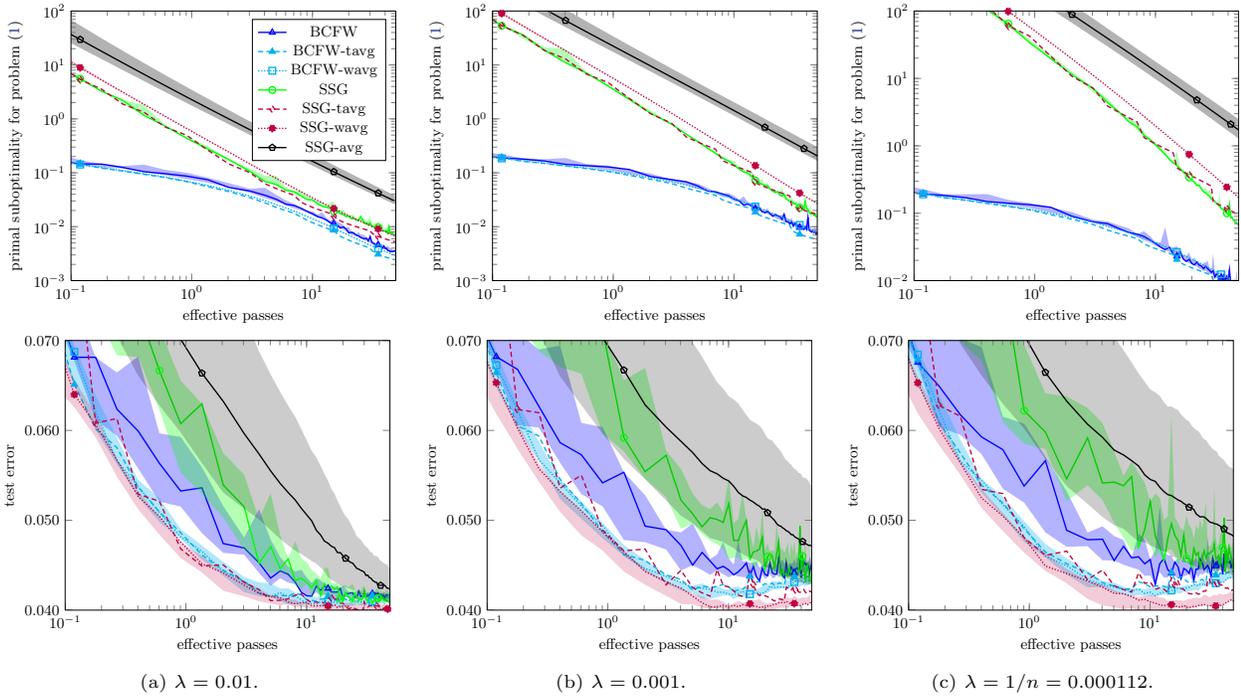

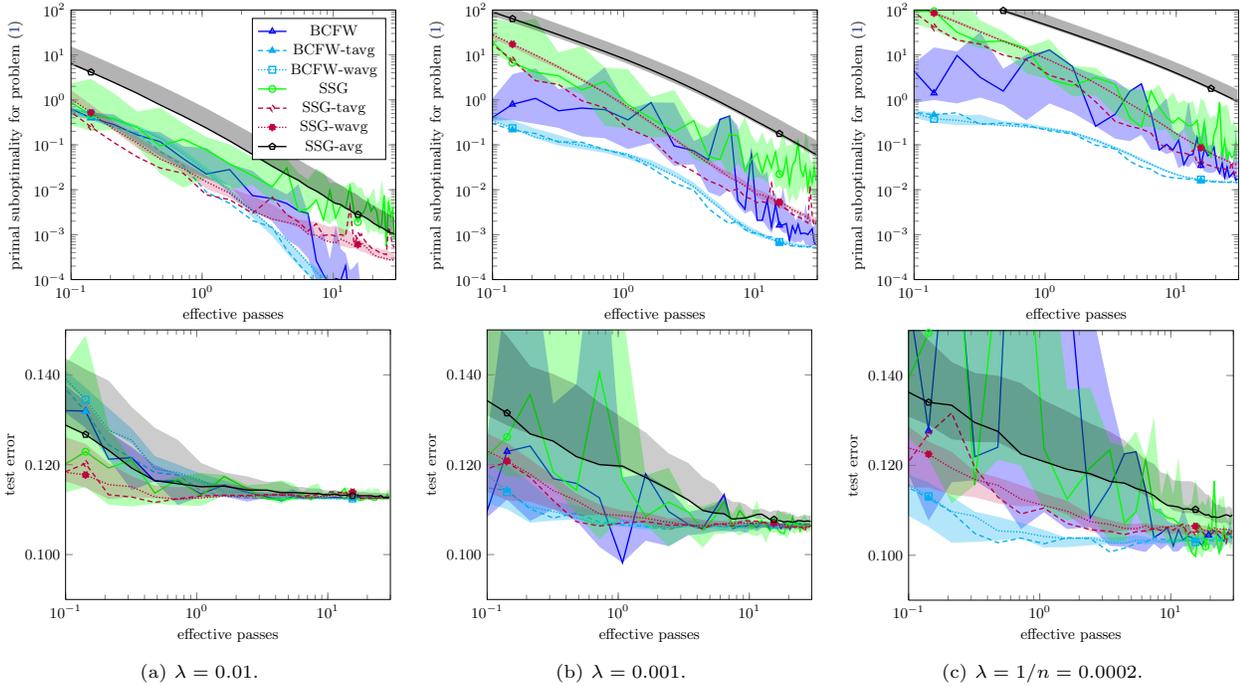
\begin{figure}[htb]
    \begin{subfigure}[t]{0.32\linewidth}
        \centering
        \def\xlabel{effective passes}
        \def\xmin{0.1}
        \def\xmax{30}
        \def\ymin{0.0001}
        \def\ymax{100}
        \def\xmode{log}
        \def\showlegend{1}
        \def\showtavg{1}
        \def\legendpos{north east}
        \def\experimentprefix{include/data/dataset=matching2_lambda=0.010000}
        \small
\begin{tikzpicture}[scale=0.63]

\begin{axis}[
xlabel=\xlabel,
ylabel=primal suboptimality for problem \eqref{eq:svmstruct_nslack_primal},
xmin=\xmin,
xmax=\xmax,
ymin=\ymin,
ymax=\ymax,
enlargelimits=false, area style,
ymode=log,
xmode=\xmode,
line legend,
legend pos=\legendpos,
]

\ifdefined \showsfwwithoutline
{

\addplot[fill=red,draw=none,forget plot,opacity=0.3] table[x index=0,y
index=1, header=true, col sep=comma]
{\experimentprefix/product_confidence.txt};

\addplot [
color=red,
solid,
style=thick,
mark=square,
mark repeat=20,
]
table[x index=0,y index=2, header=true, col sep=comma]
{\experimentprefix/product.txt};
\ifnum \showlegend=1
{
\addlegendentry{BCFW-fix}
}
\fi

}
\fi


\addplot[fill=blue,draw=none,forget plot,opacity=0.3] table[x index=0,y
index=1, header=true, col sep=comma]
{\experimentprefix/product-LS_confidence.txt};

\addplot [
color=blue,
solid,
style=thick,
mark=triangle,
mark repeat=20,
]
table[x index=0,y index=2, header=true, col sep=comma]
{\experimentprefix/product-LS.txt};
\ifnum \showlegend=1
{
\addlegendentry{BCFW}
}
\fi

\ifdefined \showtavg
{
\addplot [
color=cyan,
densely dashed,
style=thick,
mark=triangle*,
mark repeat=20,
mark options=solid
]
table[x index=0,y index=2, header=true, col sep=comma]
{\experimentprefix/product-LS-opt.txt};
\ifnum \showlegend=1
{
\addlegendentry{BCFW-tavg}
}
\fi
}
\fi

\addplot[fill=cyan,draw=none,forget plot,opacity=0.3] table[x index=0,y
index=1, header=true, col sep=comma]
{\experimentprefix/product-LS-wavg_confidence.txt};

\addplot [
color=cyan,
densely dotted,
style=thick,
mark=square,
mark repeat=20,
mark options=solid
]
table[x index=0,y index=2, header=true, col sep=comma]
{\experimentprefix/product-LS-wavg.txt};
\ifnum \showlegend=1
{
\addlegendentry{BCFW-wavg}
}
\fi


\addplot[fill=green,draw=none,forget plot,opacity=0.3] table[x index=0,y
index=1, header=true, col sep=comma]
{\experimentprefix/pegasos_confidence.txt};

\addplot [
color=green,
solid,
style=thick,
mark=o,
mark repeat=20,
]
table[x index=0,y index=2, header=true, col sep=comma]
{\experimentprefix/pegasos.txt};
\ifnum \showlegend=1
{
\addlegendentry{SSG}
}
\fi

\ifdefined \showtavg
{

\addplot [
color=purple,
densely dashed,
style=thick,
mark=diamond,
mark repeat=20,
]
table[x index=0,y index=2, header=true, col sep=comma]
{\experimentprefix/optimalSG.txt};
\ifnum \showlegend=1
{
\addlegendentry{SSG-tavg}
}
\fi
}
\fi


\addplot[fill=purple,draw=none,forget plot,opacity=0.2] table[x index=0,y
index=1, header=true, col sep=comma]
{\experimentprefix/pegasos-wavg_confidence.txt};

\addplot [
color=purple,
densely dotted,
style=thick,
mark=*,
mark repeat=20,
]
table[x index=0,y index=2, header=true, col sep=comma]
{\experimentprefix/pegasos-wavg.txt};
\ifnum \showlegend=1
{
\addlegendentry{SSG-wavg}
}
\fi


\addplot[fill=black,draw=none,forget plot,opacity=0.3] table[x index=0,y
index=1, header=true, col sep=comma]
{\experimentprefix/pegasos-avg_confidence.txt};

\addplot [
color=black,
solid,
style=thick,
mark=pentagon,
mark repeat=20,
]
table[x index=0,y index=2, header=true, col sep=comma]
{\experimentprefix/pegasos-avg.txt};
\ifnum \showlegend=1
{
\addlegendentry{SSG-avg}
}
\fi

\ifdefined \showeg
{

\addplot[fill=gray,draw=none,forget plot,opacity=0.2] table[x index=0,y
index=1, header=true, col sep=comma]
{\experimentprefix/OEG_confidence.txt};

\addplot [
color=gray,
solid,
style=thick,
mark=pentagon*,
mark repeat=20,
]
table[x index=0,y index=2, header=true, col sep=comma]
{\experimentprefix/OEG.txt};
\ifnum \showlegend=1
{
\addlegendentry{online-EG}
}
\fi
}
\fi

\end{axis}

\end{tikzpicture}
\normalsize
        \def\xlabel{effective passes}
        \def\showlegend{0}
        \def\showtavg{1}
        \def\xmin{0.1}
        \def\xmax{30}
        \def\ymin{0.09}
        \def\ymax{0.15}
        \def\xmode{log}
        \def\experimentprefix{include/data/dataset=matching2_lambda=0.010000}
        \small
\begin{tikzpicture}[scale=0.63]

\pgfplotsset{y tick label style={ 
         scaled ticks=false, 
         /pgf/number format/fixed zerofill, 
         /pgf/number format/fixed, 
         /pgf/number format/precision=3, 
     }
}

\begin{axis}[
xlabel=\xlabel,
ylabel=test error,
xmin=\xmin,
xmax=\xmax,
ymin=\ymin,
ymax=\ymax,
xmode=\xmode,
enlargelimits=false, area style,
line legend,
]

\ifdefined \showsfwwithoutline
{

\addplot[fill=red,draw=none,forget plot,opacity=0.3] table[x index=0,y
index=2, header=true, col sep=comma]
{\experimentprefix/product_confidence.txt};

\addplot [
color=red,
solid,
style=thick,
mark=square,
mark repeat=20,
]
table[x index=0,y index=3, header=true, col sep=comma]
{\experimentprefix/product.txt};
\ifnum \showlegend=1
{
\addlegendentry{BCFW-fix}
}
\fi

}
\fi


\addplot[fill=blue,draw=none,forget plot,opacity=0.3] table[x index=0,y
index=2, header=true, col sep=comma]
{\experimentprefix/product-LS_confidence.txt};

\addplot [
color=blue,
solid,
style=thick,
mark=none,
mark=triangle,
mark repeat=20,
]
table[x index=0,y index=3, header=true, col sep=comma]
{\experimentprefix/product-LS.txt};
\ifnum \showlegend=1
{
\addlegendentry{BCFW}
}
\fi

\ifdefined \showtavg
{

\addplot [
color=cyan,
densely dashed,
style=thick,
mark=triangle*,
mark repeat=20,
mark options=solid
]
table[x index=0,y index=3, header=true, col sep=comma]
{\experimentprefix/product-LS-opt.txt};
\ifnum \showlegend=1
{
\addlegendentry{BCFW-tavg}
}
\fi
}
\fi

\addplot[fill=cyan,draw=none,forget plot,opacity=0.3] table[x index=0,y
index=2, header=true, col sep=comma]
{\experimentprefix/product-LS-wavg_confidence.txt};

\addplot [
color=cyan,
densely dotted,
style=thick,
mark=square,
mark repeat=20,
mark options=solid
]
table[x index=0,y index=3, header=true, col sep=comma]
{\experimentprefix/product-LS-wavg.txt};
\ifnum \showlegend=1
{
\addlegendentry{BCFW-wavg}
}
\fi


\addplot[fill=green,draw=none,forget plot,opacity=0.3] table[x index=0,y
index=2, header=true, col sep=comma]
{\experimentprefix/pegasos_confidence.txt};

\addplot [
color=green,
solid,
style=thick,
mark=o,
mark repeat=20,
]
table[x index=0,y index=3, header=true, col sep=comma]
{\experimentprefix/pegasos.txt};
\ifnum \showlegend=1
{
\addlegendentry{SSG}
}
\fi

\ifdefined \showtavg
{
\addplot [
color=purple,
densely dashed,
style=thick,
mark=diamond,
mark repeat=20,
]
table[x index=0,y index=3, header=true, col sep=comma]
{\experimentprefix/optimalSG.txt};
\ifnum \showlegend=1
{
\addlegendentry{SSG-tavg}
}
\fi
}
\fi


\addplot[fill=purple,draw=none,forget plot,opacity=0.2] table[x index=0,y
index=2, header=true, col sep=comma]
{\experimentprefix/pegasos-wavg_confidence.txt};

\addplot [
color=purple,
densely dotted,
style=thick,
mark=*,
mark repeat=20,
]
table[x index=0,y index=3, header=true, col sep=comma]
{\experimentprefix/pegasos-wavg.txt};
\ifnum \showlegend=1
{
\addlegendentry{SSG-wavg}
}
\fi


\addplot[fill=black,draw=none,forget plot,opacity=0.2] table[x index=0,y
index=2, header=true, col sep=comma]
{\experimentprefix/pegasos-avg_confidence.txt};

\addplot [
color=black,
solid,
style=thick,
mark=pentagon,
mark repeat=20,
]
table[x index=0,y index=3, header=true, col sep=comma]
{\experimentprefix/pegasos-avg.txt};
\ifnum \showlegend=1
{
\addlegendentry{SSG-avg}
}
\fi

\ifdefined \showeg
{

\addplot[fill=gray,draw=none,forget plot,opacity=0.2] table[x index=0,y
index=2, header=true, col sep=comma]
{\experimentprefix/OEG_confidence.txt};

\addplot [
color=gray,
solid,
style=thick,
mark=pentagon*,
mark repeat=20,
]
table[x index=0,y index=3, header=true, col sep=comma]
{\experimentprefix/OEG.txt};
\ifnum \showlegend=1
{
\addlegendentry{online-EG}
}
\fi
}
\fi

\end{axis}
\end{tikzpicture}
\normalsize
        \caption{$\lambda=0.01$.}
    \end{subfigure}
    \begin{subfigure}[t]{0.32\linewidth}
        \centering
        \def\xlabel{effective passes}
        \def\xmin{0.1}
        \def\xmax{30}
        \def\ymin{0.0001}
        \def\ymax{100}
        \def\xmode{log}
        \def\showlegend{0}
        \def\showtavg{1}
        \def\legendpos{north east}
        \def\experimentprefix{include/data/dataset=matching2_lambda=0.001000}
        \small
\begin{tikzpicture}[scale=0.63]

\begin{axis}[
xlabel=\xlabel,
ylabel=primal suboptimality for problem \eqref{eq:svmstruct_nslack_primal},
xmin=\xmin,
xmax=\xmax,
ymin=\ymin,
ymax=\ymax,
enlargelimits=false, area style,
ymode=log,
xmode=\xmode,
line legend,
legend pos=\legendpos,
]

\ifdefined \showsfwwithoutline
{

\addplot[fill=red,draw=none,forget plot,opacity=0.3] table[x index=0,y
index=1, header=true, col sep=comma]
{\experimentprefix/product_confidence.txt};

\addplot [
color=red,
solid,
style=thick,
mark=square,
mark repeat=20,
]
table[x index=0,y index=2, header=true, col sep=comma]
{\experimentprefix/product.txt};
\ifnum \showlegend=1
{
\addlegendentry{BCFW-fix}
}
\fi

}
\fi


\addplot[fill=blue,draw=none,forget plot,opacity=0.3] table[x index=0,y
index=1, header=true, col sep=comma]
{\experimentprefix/product-LS_confidence.txt};

\addplot [
color=blue,
solid,
style=thick,
mark=triangle,
mark repeat=20,
]
table[x index=0,y index=2, header=true, col sep=comma]
{\experimentprefix/product-LS.txt};
\ifnum \showlegend=1
{
\addlegendentry{BCFW}
}
\fi

\ifdefined \showtavg
{
\addplot [
color=cyan,
densely dashed,
style=thick,
mark=triangle*,
mark repeat=20,
mark options=solid
]
table[x index=0,y index=2, header=true, col sep=comma]
{\experimentprefix/product-LS-opt.txt};
\ifnum \showlegend=1
{
\addlegendentry{BCFW-tavg}
}
\fi
}
\fi

\addplot[fill=cyan,draw=none,forget plot,opacity=0.3] table[x index=0,y
index=1, header=true, col sep=comma]
{\experimentprefix/product-LS-wavg_confidence.txt};

\addplot [
color=cyan,
densely dotted,
style=thick,
mark=square,
mark repeat=20,
mark options=solid
]
table[x index=0,y index=2, header=true, col sep=comma]
{\experimentprefix/product-LS-wavg.txt};
\ifnum \showlegend=1
{
\addlegendentry{BCFW-wavg}
}
\fi


\addplot[fill=green,draw=none,forget plot,opacity=0.3] table[x index=0,y
index=1, header=true, col sep=comma]
{\experimentprefix/pegasos_confidence.txt};

\addplot [
color=green,
solid,
style=thick,
mark=o,
mark repeat=20,
]
table[x index=0,y index=2, header=true, col sep=comma]
{\experimentprefix/pegasos.txt};
\ifnum \showlegend=1
{
\addlegendentry{SSG}
}
\fi

\ifdefined \showtavg
{

\addplot [
color=purple,
densely dashed,
style=thick,
mark=diamond,
mark repeat=20,
]
table[x index=0,y index=2, header=true, col sep=comma]
{\experimentprefix/optimalSG.txt};
\ifnum \showlegend=1
{
\addlegendentry{SSG-tavg}
}
\fi
}
\fi


\addplot[fill=purple,draw=none,forget plot,opacity=0.2] table[x index=0,y
index=1, header=true, col sep=comma]
{\experimentprefix/pegasos-wavg_confidence.txt};

\addplot [
color=purple,
densely dotted,
style=thick,
mark=*,
mark repeat=20,
]
table[x index=0,y index=2, header=true, col sep=comma]
{\experimentprefix/pegasos-wavg.txt};
\ifnum \showlegend=1
{
\addlegendentry{SSG-wavg}
}
\fi


\addplot[fill=black,draw=none,forget plot,opacity=0.3] table[x index=0,y
index=1, header=true, col sep=comma]
{\experimentprefix/pegasos-avg_confidence.txt};

\addplot [
color=black,
solid,
style=thick,
mark=pentagon,
mark repeat=20,
]
table[x index=0,y index=2, header=true, col sep=comma]
{\experimentprefix/pegasos-avg.txt};
\ifnum \showlegend=1
{
\addlegendentry{SSG-avg}
}
\fi

\ifdefined \showeg
{

\addplot[fill=gray,draw=none,forget plot,opacity=0.2] table[x index=0,y
index=1, header=true, col sep=comma]
{\experimentprefix/OEG_confidence.txt};

\addplot [
color=gray,
solid,
style=thick,
mark=pentagon*,
mark repeat=20,
]
table[x index=0,y index=2, header=true, col sep=comma]
{\experimentprefix/OEG.txt};
\ifnum \showlegend=1
{
\addlegendentry{online-EG}
}
\fi
}
\fi

\end{axis}

\end{tikzpicture}
\normalsize
        \def\xlabel{effective passes}
        \def\showlegend{0}
        \def\showtavg{1}
        \def\xmin{0.1}
        \def\xmax{30}
        \def\ymin{0.09}
        \def\ymax{0.15}
        \def\xmode{log}
        \def\experimentprefix{include/data/dataset=matching2_lambda=0.001000}
        \small
\begin{tikzpicture}[scale=0.63]

\pgfplotsset{y tick label style={ 
         scaled ticks=false, 
         /pgf/number format/fixed zerofill, 
         /pgf/number format/fixed, 
         /pgf/number format/precision=3, 
     }
}

\begin{axis}[
xlabel=\xlabel,
ylabel=test error,
xmin=\xmin,
xmax=\xmax,
ymin=\ymin,
ymax=\ymax,
xmode=\xmode,
enlargelimits=false, area style,
line legend,
]

\ifdefined \showsfwwithoutline
{

\addplot[fill=red,draw=none,forget plot,opacity=0.3] table[x index=0,y
index=2, header=true, col sep=comma]
{\experimentprefix/product_confidence.txt};

\addplot [
color=red,
solid,
style=thick,
mark=square,
mark repeat=20,
]
table[x index=0,y index=3, header=true, col sep=comma]
{\experimentprefix/product.txt};
\ifnum \showlegend=1
{
\addlegendentry{BCFW-fix}
}
\fi

}
\fi


\addplot[fill=blue,draw=none,forget plot,opacity=0.3] table[x index=0,y
index=2, header=true, col sep=comma]
{\experimentprefix/product-LS_confidence.txt};

\addplot [
color=blue,
solid,
style=thick,
mark=none,
mark=triangle,
mark repeat=20,
]
table[x index=0,y index=3, header=true, col sep=comma]
{\experimentprefix/product-LS.txt};
\ifnum \showlegend=1
{
\addlegendentry{BCFW}
}
\fi

\ifdefined \showtavg
{

\addplot [
color=cyan,
densely dashed,
style=thick,
mark=triangle*,
mark repeat=20,
mark options=solid
]
table[x index=0,y index=3, header=true, col sep=comma]
{\experimentprefix/product-LS-opt.txt};
\ifnum \showlegend=1
{
\addlegendentry{BCFW-tavg}
}
\fi
}
\fi

\addplot[fill=cyan,draw=none,forget plot,opacity=0.3] table[x index=0,y
index=2, header=true, col sep=comma]
{\experimentprefix/product-LS-wavg_confidence.txt};

\addplot [
color=cyan,
densely dotted,
style=thick,
mark=square,
mark repeat=20,
mark options=solid
]
table[x index=0,y index=3, header=true, col sep=comma]
{\experimentprefix/product-LS-wavg.txt};
\ifnum \showlegend=1
{
\addlegendentry{BCFW-wavg}
}
\fi


\addplot[fill=green,draw=none,forget plot,opacity=0.3] table[x index=0,y
index=2, header=true, col sep=comma]
{\experimentprefix/pegasos_confidence.txt};

\addplot [
color=green,
solid,
style=thick,
mark=o,
mark repeat=20,
]
table[x index=0,y index=3, header=true, col sep=comma]
{\experimentprefix/pegasos.txt};
\ifnum \showlegend=1
{
\addlegendentry{SSG}
}
\fi

\ifdefined \showtavg
{
\addplot [
color=purple,
densely dashed,
style=thick,
mark=diamond,
mark repeat=20,
]
table[x index=0,y index=3, header=true, col sep=comma]
{\experimentprefix/optimalSG.txt};
\ifnum \showlegend=1
{
\addlegendentry{SSG-tavg}
}
\fi
}
\fi


\addplot[fill=purple,draw=none,forget plot,opacity=0.2] table[x index=0,y
index=2, header=true, col sep=comma]
{\experimentprefix/pegasos-wavg_confidence.txt};

\addplot [
color=purple,
densely dotted,
style=thick,
mark=*,
mark repeat=20,
]
table[x index=0,y index=3, header=true, col sep=comma]
{\experimentprefix/pegasos-wavg.txt};
\ifnum \showlegend=1
{
\addlegendentry{SSG-wavg}
}
\fi


\addplot[fill=black,draw=none,forget plot,opacity=0.2] table[x index=0,y
index=2, header=true, col sep=comma]
{\experimentprefix/pegasos-avg_confidence.txt};

\addplot [
color=black,
solid,
style=thick,
mark=pentagon,
mark repeat=20,
]
table[x index=0,y index=3, header=true, col sep=comma]
{\experimentprefix/pegasos-avg.txt};
\ifnum \showlegend=1
{
\addlegendentry{SSG-avg}
}
\fi

\ifdefined \showeg
{

\addplot[fill=gray,draw=none,forget plot,opacity=0.2] table[x index=0,y
index=2, header=true, col sep=comma]
{\experimentprefix/OEG_confidence.txt};

\addplot [
color=gray,
solid,
style=thick,
mark=pentagon*,
mark repeat=20,
]
table[x index=0,y index=3, header=true, col sep=comma]
{\experimentprefix/OEG.txt};
\ifnum \showlegend=1
{
\addlegendentry{online-EG}
}
\fi
}
\fi

\end{axis}
\end{tikzpicture}
\normalsize
        \caption{$\lambda=0.001$.}
    \end{subfigure}
    \begin{subfigure}[t]{0.32\linewidth}
        \centering
        \def\xlabel{effective passes}
        \def\xmin{0.1}
        \def\xmax{30}
        \def\ymin{0.0001}
        \def\ymax{100}
        \def\xmode{log}
        \def\showlegend{0}
        \def\showtavg{1}
        \def\legendpos{north east}
        \def\experimentprefix{include/data/dataset=matching2_lambda=0.000200}
        \small
\begin{tikzpicture}[scale=0.63]

\begin{axis}[
xlabel=\xlabel,
ylabel=primal suboptimality for problem \eqref{eq:svmstruct_nslack_primal},
xmin=\xmin,
xmax=\xmax,
ymin=\ymin,
ymax=\ymax,
enlargelimits=false, area style,
ymode=log,
xmode=\xmode,
line legend,
legend pos=\legendpos,
]

\ifdefined \showsfwwithoutline
{

\addplot[fill=red,draw=none,forget plot,opacity=0.3] table[x index=0,y
index=1, header=true, col sep=comma]
{\experimentprefix/product_confidence.txt};

\addplot [
color=red,
solid,
style=thick,
mark=square,
mark repeat=20,
]
table[x index=0,y index=2, header=true, col sep=comma]
{\experimentprefix/product.txt};
\ifnum \showlegend=1
{
\addlegendentry{BCFW-fix}
}
\fi

}
\fi


\addplot[fill=blue,draw=none,forget plot,opacity=0.3] table[x index=0,y
index=1, header=true, col sep=comma]
{\experimentprefix/product-LS_confidence.txt};

\addplot [
color=blue,
solid,
style=thick,
mark=triangle,
mark repeat=20,
]
table[x index=0,y index=2, header=true, col sep=comma]
{\experimentprefix/product-LS.txt};
\ifnum \showlegend=1
{
\addlegendentry{BCFW}
}
\fi

\ifdefined \showtavg
{
\addplot [
color=cyan,
densely dashed,
style=thick,
mark=triangle*,
mark repeat=20,
mark options=solid
]
table[x index=0,y index=2, header=true, col sep=comma]
{\experimentprefix/product-LS-opt.txt};
\ifnum \showlegend=1
{
\addlegendentry{BCFW-tavg}
}
\fi
}
\fi

\addplot[fill=cyan,draw=none,forget plot,opacity=0.3] table[x index=0,y
index=1, header=true, col sep=comma]
{\experimentprefix/product-LS-wavg_confidence.txt};

\addplot [
color=cyan,
densely dotted,
style=thick,
mark=square,
mark repeat=20,
mark options=solid
]
table[x index=0,y index=2, header=true, col sep=comma]
{\experimentprefix/product-LS-wavg.txt};
\ifnum \showlegend=1
{
\addlegendentry{BCFW-wavg}
}
\fi


\addplot[fill=green,draw=none,forget plot,opacity=0.3] table[x index=0,y
index=1, header=true, col sep=comma]
{\experimentprefix/pegasos_confidence.txt};

\addplot [
color=green,
solid,
style=thick,
mark=o,
mark repeat=20,
]
table[x index=0,y index=2, header=true, col sep=comma]
{\experimentprefix/pegasos.txt};
\ifnum \showlegend=1
{
\addlegendentry{SSG}
}
\fi

\ifdefined \showtavg
{

\addplot [
color=purple,
densely dashed,
style=thick,
mark=diamond,
mark repeat=20,
]
table[x index=0,y index=2, header=true, col sep=comma]
{\experimentprefix/optimalSG.txt};
\ifnum \showlegend=1
{
\addlegendentry{SSG-tavg}
}
\fi
}
\fi


\addplot[fill=purple,draw=none,forget plot,opacity=0.2] table[x index=0,y
index=1, header=true, col sep=comma]
{\experimentprefix/pegasos-wavg_confidence.txt};

\addplot [
color=purple,
densely dotted,
style=thick,
mark=*,
mark repeat=20,
]
table[x index=0,y index=2, header=true, col sep=comma]
{\experimentprefix/pegasos-wavg.txt};
\ifnum \showlegend=1
{
\addlegendentry{SSG-wavg}
}
\fi


\addplot[fill=black,draw=none,forget plot,opacity=0.3] table[x index=0,y
index=1, header=true, col sep=comma]
{\experimentprefix/pegasos-avg_confidence.txt};

\addplot [
color=black,
solid,
style=thick,
mark=pentagon,
mark repeat=20,
]
table[x index=0,y index=2, header=true, col sep=comma]
{\experimentprefix/pegasos-avg.txt};
\ifnum \showlegend=1
{
\addlegendentry{SSG-avg}
}
\fi

\ifdefined \showeg
{

\addplot[fill=gray,draw=none,forget plot,opacity=0.2] table[x index=0,y
index=1, header=true, col sep=comma]
{\experimentprefix/OEG_confidence.txt};

\addplot [
color=gray,
solid,
style=thick,
mark=pentagon*,
mark repeat=20,
]
table[x index=0,y index=2, header=true, col sep=comma]
{\experimentprefix/OEG.txt};
\ifnum \showlegend=1
{
\addlegendentry{online-EG}
}
\fi
}
\fi

\end{axis}

\end{tikzpicture}
\normalsize
        \def\xlabel{effective passes}
        \def\showlegend{0}
        \def\showtavg{1}
        \def\xmin{0.1}
        \def\xmax{30}
        \def\ymin{0.09}
        \def\ymax{0.15}
        \def\xmode{log}
        \def\experimentprefix{include/data/dataset=matching2_lambda=0.000200}
        \small
\begin{tikzpicture}[scale=0.63]

\pgfplotsset{y tick label style={ 
         scaled ticks=false, 
         /pgf/number format/fixed zerofill, 
         /pgf/number format/fixed, 
         /pgf/number format/precision=3, 
     }
}

\begin{axis}[
xlabel=\xlabel,
ylabel=test error,
xmin=\xmin,
xmax=\xmax,
ymin=\ymin,
ymax=\ymax,
xmode=\xmode,
enlargelimits=false, area style,
line legend,
]

\ifdefined \showsfwwithoutline
{

\addplot[fill=red,draw=none,forget plot,opacity=0.3] table[x index=0,y
index=2, header=true, col sep=comma]
{\experimentprefix/product_confidence.txt};

\addplot [
color=red,
solid,
style=thick,
mark=square,
mark repeat=20,
]
table[x index=0,y index=3, header=true, col sep=comma]
{\experimentprefix/product.txt};
\ifnum \showlegend=1
{
\addlegendentry{BCFW-fix}
}
\fi

}
\fi


\addplot[fill=blue,draw=none,forget plot,opacity=0.3] table[x index=0,y
index=2, header=true, col sep=comma]
{\experimentprefix/product-LS_confidence.txt};

\addplot [
color=blue,
solid,
style=thick,
mark=none,
mark=triangle,
mark repeat=20,
]
table[x index=0,y index=3, header=true, col sep=comma]
{\experimentprefix/product-LS.txt};
\ifnum \showlegend=1
{
\addlegendentry{BCFW}
}
\fi

\ifdefined \showtavg
{

\addplot [
color=cyan,
densely dashed,
style=thick,
mark=triangle*,
mark repeat=20,
mark options=solid
]
table[x index=0,y index=3, header=true, col sep=comma]
{\experimentprefix/product-LS-opt.txt};
\ifnum \showlegend=1
{
\addlegendentry{BCFW-tavg}
}
\fi
}
\fi

\addplot[fill=cyan,draw=none,forget plot,opacity=0.3] table[x index=0,y
index=2, header=true, col sep=comma]
{\experimentprefix/product-LS-wavg_confidence.txt};

\addplot [
color=cyan,
densely dotted,
style=thick,
mark=square,
mark repeat=20,
mark options=solid
]
table[x index=0,y index=3, header=true, col sep=comma]
{\experimentprefix/product-LS-wavg.txt};
\ifnum \showlegend=1
{
\addlegendentry{BCFW-wavg}
}
\fi


\addplot[fill=green,draw=none,forget plot,opacity=0.3] table[x index=0,y
index=2, header=true, col sep=comma]
{\experimentprefix/pegasos_confidence.txt};

\addplot [
color=green,
solid,
style=thick,
mark=o,
mark repeat=20,
]
table[x index=0,y index=3, header=true, col sep=comma]
{\experimentprefix/pegasos.txt};
\ifnum \showlegend=1
{
\addlegendentry{SSG}
}
\fi

\ifdefined \showtavg
{
\addplot [
color=purple,
densely dashed,
style=thick,
mark=diamond,
mark repeat=20,
]
table[x index=0,y index=3, header=true, col sep=comma]
{\experimentprefix/optimalSG.txt};
\ifnum \showlegend=1
{
\addlegendentry{SSG-tavg}
}
\fi
}
\fi


\addplot[fill=purple,draw=none,forget plot,opacity=0.2] table[x index=0,y
index=2, header=true, col sep=comma]
{\experimentprefix/pegasos-wavg_confidence.txt};

\addplot [
color=purple,
densely dotted,
style=thick,
mark=*,
mark repeat=20,
]
table[x index=0,y index=3, header=true, col sep=comma]
{\experimentprefix/pegasos-wavg.txt};
\ifnum \showlegend=1
{
\addlegendentry{SSG-wavg}
}
\fi


\addplot[fill=black,draw=none,forget plot,opacity=0.2] table[x index=0,y
index=2, header=true, col sep=comma]
{\experimentprefix/pegasos-avg_confidence.txt};

\addplot [
color=black,
solid,
style=thick,
mark=pentagon,
mark repeat=20,
]
table[x index=0,y index=3, header=true, col sep=comma]
{\experimentprefix/pegasos-avg.txt};
\ifnum \showlegend=1
{
\addlegendentry{SSG-avg}
}
\fi

\ifdefined \showeg
{

\addplot[fill=gray,draw=none,forget plot,opacity=0.2] table[x index=0,y
index=2, header=true, col sep=comma]
{\experimentprefix/OEG_confidence.txt};

\addplot [
color=gray,
solid,
style=thick,
mark=pentagon*,
mark repeat=20,
]
table[x index=0,y index=3, header=true, col sep=comma]
{\experimentprefix/OEG.txt};
\ifnum \showlegend=1
{
\addlegendentry{online-EG}
}
\fi
}
\fi

\end{axis}
\end{tikzpicture}
\normalsize
        \caption{$\lambda=1/n=0.0002$.}
    \end{subfigure}
    \caption{Convergence (top) and test error (bottom) of the stochastic
    solvers on the Matching dataset. %
    }
    \label{fig:matching_results}
\end{figure}

\clearpage
\paragraph{More Information about Implementation.}
We note that since the value of the true optimum is unknown, the primal suboptimality for each experiment was measured as the difference to the highest dual objective seen for the corresponding regularization parameter (amongst all methods). Moreover, the \emph{lower envelope} of the obtained  primal objective values was drawn in Figure~\ref{fig:results} for the batch methods (cutting plane and Frank-Wolfe), given that these methods can efficiently keep track of the best parameter seen so far.

The \emph{online-EG} method used the same adaptive step-size scheme as described in~\citet{Collins2008} and with the parameters from their code \texttt{egstra-0.2} available online.\footnote{\href{http://groups.csail.mit.edu/nlp/egstra/}{http://groups.csail.mit.edu/nlp/egstra/}} Each datapoint has their own step-size, initialized at $0.5$. Backtracking line-search is used, where the step-size is halved until the objective is decreased (or a maximum number of halvings has been reached: 2 for the first pass through the data; 5 otherwise). After each line-search, the step-size is multiplied by $1.05$. We note that each evaluation of the objective requires a new call to the (expectation) oracle, and we count these extra calls in the computation of the effective number of passes appearing on the x-axis of the plots. Unlike all the other methods which initialize $\weightv^{(0)}=\0$, \emph{online-EG} 
initially sets the dual variables $\dualvarv_{(i)}^{(0)}$ to a uniform distribution, which yields a problem-dependent initialization $\weightv^{(0)}$.

For \emph{SSG}, we used the same step-size as in the `Pegasos' version of~\citet{ShalevShwartz:2010cg}: $\gamma_k := \frac{1}{\lambda (k+1)}$.

For the \emph{cutting plane} method, we use the version 1.1 of the \texttt{svm-struct-matlab} MATLAB wrapper code from~\citetsup{vedaldi:11svmstruct} with its default options.

The test error for the OCR and CoNLL tasks is the normalized Hamming distance on the sequences. 

For the matching prediction task, we use the same setting from~\citet{Taskar06extrag}, with $5,000$ training examples and $347$ Gold test examples. During training, an asymmetric Hamming loss is used where the precision error cost is $1$ while the recall error cost is $3$. For testing, error is the `alignment error rate', as defined in~\citet{Taskar06extrag}.

\bigskip
\bibliographysup{references}
\bibliographystylesup{icml2013}

\end{document}